%% file: ICML2024_paper.tex
\documentclass{article}

\usepackage{microtype}
\usepackage{graphicx}
\usepackage{subfigure}
\usepackage{booktabs} % for professional tables

\usepackage{multirow}
\usepackage[table,xcdraw]{xcolor}

\usepackage[breaklinks]{hyperref}

\usepackage[accepted]{icml2024}

\input{math_commands.tex}

\definecolor{darkblue}{rgb}{0.0,0.0,0.65}

\usepackage{url}
\usepackage{times}
\usepackage{latexsym}
\usepackage[inkscapelatex=false]{svg}
\usepackage{multicol}
\usepackage{graphicx} % DO NOT CHANGE THIS
\usepackage{natbib}  % DO NOT CHANGE THIS AND DO NOT ADD ANY OPTIONS TO IT
\usepackage{caption} % DO NOT CHANGE THIS AND DO NOT ADD ANY OPTIONS TO IT
\usepackage{amssymb}
\usepackage{subfigure}
\usepackage{amsmath}
\usepackage{bbm}
\usepackage{amsthm}
\usepackage{adjustbox}
\usepackage{multirow}
\usepackage{booktabs}
\usepackage{amsmath}
\usepackage{float}
\usepackage{arydshln}
\usepackage{cleveref}
\usepackage{wrapfig}
\usepackage{comment}
\usepackage{enumitem}
\crefformat{section}{\textcolor{blue}{\S#2#1#3}} % see manual of cleveref, section 8.2.1
\crefformat{subsection}{\textcolor{blue}{\S#2#1#3}}
\crefformat{subsubsection}{\textcolor{blue}{\S#2#1#3}}
\crefformat{equation}{\textcolor{orange}{Eq.~(#1)}}

\definecolor{mydarkblue}{rgb}{0,0.08,0.45}
\theoremstyle{plain}
\newtheorem{corollary}{Corollary}
\newtheorem{theorem}{Theorem}

\newtheorem{definition}{Definition}

\usepackage{pifont}% http://ctan.org/pkg/pifont
\definecolor{White}{rgb}{1, 1, 1}
\definecolor{Periwinkle}{rgb}{0, 0, 0}
\definecolor{myblue}{rgb}{0.82, 0.94, 0.75}
\definecolor{mygreen}{rgb}{0.64, 0.76, 0.68}
\definecolor{myyellow}{rgb}{0.88, 0.54, 0.35}
\definecolor{mygreen}{rgb}{0.68, 0.9, 0.8}
\definecolor{mypink}{rgb}{0.2, 0.87, 0.2}
\def\arrvline{\hfil\kern\arraycolsep\vline\kern-\arraycolsep\hfilneg}

\usepackage{amsmath}
\usepackage{amssymb}
\usepackage{amsthm}
\usepackage{thmtools}

\usepackage[most]{tcolorbox}

\colorlet{LightGray}{White!98!Periwinkle}

\declaretheoremstyle[
    name=Hypothesis,
]{thmsty}
\declaretheorem[style=thmsty]{hypothesis}

\tcolorboxenvironment{hypothesis}{enhanced jigsaw,colback=LightGray,drop shadow,boxrule=0.9pt, boxsep=0.1pt,left=4pt,right=4pt,top=4pt,bottom=4pt}

\usepackage{tikz}
\newcommand*\circled[1]{\tikz[baseline=(char.base)]{
            \node[shape=circle,draw,inner sep=0.6pt] (char) {#1};}}
\usepackage[textsize=tiny]{todonotes}

\begin{document}

\twocolumn[
% \icmltitle{Do Pretrained Transformers Really Learn In-context by Gradient Descent?}
\icmltitle{Do pretrained Transformers Learn In-Context by Gradient Descent?}
\icmltitlerunning{Do pretrained Transformers Learn In-Context by Gradient Descent?}



\icmlsetsymbol{equal}{*}

\begin{icmlauthorlist}
\icmlauthor{Lingfeng Shen}{equal,yyy}
\icmlauthor{Aayush Mishra}{equal,yyy}
\icmlauthor{Daniel Khashabi}{yyy}
\end{icmlauthorlist}

\icmlaffiliation{yyy}{Johns Hopkins University, Baltimore MD}

\icmlcorrespondingauthor{Aayush Mishra}{amishr24@jhu.edu}


\vskip 0.3in
]



\printAffiliationsAndNotice{\icmlEqualContribution} % otherwise use the standard text.

\begin{abstract}
% The emergence of In-Context Learning (ICL) in LLMs remains a remarkable phenomenon with little understanding. To explain ICL, 
% recent studies try to theoretically connect it to Gradient Descent (GD). 
The emergence of In-Context Learning (ICL) in LLMs remains a remarkable phenomenon that is partially understood. 
To explain ICL, recent studies have created theoretical connections to Gradient Descent (GD). 
We ask, do such connections hold up in actual pre-trained language models? 
We highlight the limiting assumptions in prior works that make their setup considerably different from the practical setup in which language models are trained. 
For example, their experimental verification uses \emph{ICL objective} (training models explicitly for ICL), which differs from the emergent ICL in the wild. Furthermore, the theoretical hand-constructed weights used in these studies have properties that don't match those of real LLMs. 
We also look for evidence in real models. We observe that ICL and GD have different sensitivity to the order in which they observe demonstrations. Finally, we probe and compare the ICL vs. GD hypothesis in a natural setting. We conduct comprehensive empirical analyses on language models pre-trained on natural data (LLaMa-7B). Our comparisons of three performance metrics highlight the inconsistent behavior of ICL and GD as a function of various factors such as datasets, models, and the number of demonstrations. 
We observe that ICL and GD modify the output distribution of language models differently. These results indicate that \emph{the equivalence between ICL and GD remains an open hypothesis} and calls for further studies. 
\end{abstract}

\section{Introduction}
    % \item ICL
    % \item Explaining ICL
    % \item Misleading theories
    % \item Our contribution

In-Context Learning (ICL) is an emergent behavior in Large Language Models (LLMs), which allows them to recognize patterns among demonstrations provided as prompts and extend these patterns to similar tasks~\citep{brown2020language}. 
This fascinating on-the-fly learning behavior has motivated ample studies to better of understand its dynamics.

\begin{comment}
\begin{wrapfigure}[11]{r}{0.65\textwidth}
\vspace{-19pt}
  \centering
  % \includegraphics[width=0.61\textwidth,trim=3cm 7.7cm 4.7cm 0cm]{figs/teaser1.pdf}
  \incl1udegraphics[width=0.65\textwidth,trim=3cm 7.7cm 5.8cm 0cm]{figs/teaser-2.pdf}
  \vspace{-10pt}
  % \caption{?}
  % \vspace{-10pt}
  \label{fig:teaser}
\end{wrapfigure}
\end{comment}
In particular, a notable line of work tries to explain ICL via Gradient Descent (GD)~\citep{garg2022can,zhang2023trained}.
This connection is interesting because GD has been around for decades and is well-understood, while ICL is a recent phenomenon that has emerged somewhat surprisingly~\citep{wei2022emergent}, and is not fully understood. Therefore, a solid formal bridge between the two approaches would be an exciting finding as it can open new doors for understanding ICL.

% that Transformers can implement learning algorithms for linear models based on gradient descent (or closed-form regression). \cite{} similarly show a dual between attention layers and linear layers optimized using gradient descent. \cite{li2023closeness} show such an equivalence on softmax regression tasks. \cite{garg2022can} then show that using an ICL objective, Transformers can learn not only linear function classes in-domain, but also out-of-domain and even sparse linear functions, decision trees and two layered ReLU neural networks. Finally \cite{} showed a similar construction with a simpler Linear Self-Attention based Transformer claiming that Transformers learn in-context using gradient descent on linear regression problems.

% \begin{wrapfigure}[4]{r}{0.99\linewidth}
% \vspace{-0.47cm}
% \begin{minipage}{0.99\linewidth}
\begin{hypothesis}
    \label{hypothesis1}
    % For \underline{any} Transformer weights resulting from self-supervised pretraining and \underline{any} well-defined task,
    % (with corresponding demos), 
    % ICL and GD lead to the same outcome.
    \underline{For any} Transformer weights resulting from self-supervised pretraining and \underline{for any} well-defined task, ICL is algorithmically equivalent to GD (whole model or sub-model).
\end{hypothesis}
% \end{minipage}
% \end{wrapfigure}
In this work, we revisit the hypothesis on the equivalence of ICL and GD, i.e., whether these two approaches to ``learning'' are functionally equivalent. 
Consider \cref{hypothesis1} that defines a \emph{universal} notion of equivalence between the ICL and GD. It defines equivalence as a property that must hold for \emph{any} Transformer model with parameters that \emph{emerge} naturally from {pretraining} on massive unlabeled data~\citep{brown2020language}, and is applicable for \emph{any} choice of well-defined tasks~\citep{srivastava2023beyond}. For example, \cite{dai2022can} claims that ICL is equivalent to implicit finetuning.
% , which are not wild expectations given the versatility of the two approaches in recent LLMs. 

% \begin{wrapfigure}[4]{r}{0.5\textwidth}
% \vspace{-0.4cm}
% \begin{minipage}{0.5\textwidth}
\begin{hypothesis}
    \label{hypothesis2}
    % For a \underline{given} well-defined task (with demonstrations), 
    % \underline{there exist} Transformer weights such that ICL and GD lead to similar outcomes.
    % \underline{There exist} Transformer weights resulting from \textit{ICL training} such that, for \underline{some} well-defined tasks, 
    % (with demos), 
    % ICL \textit{simulates} GD on an \textit{implicit} model.
    \underline{For a given} well-defined task, \underline{there exist} Transformer weights such that $\widehat{\text{ICL}}$
    % (trained for it or assumed weights) 
    is algorithmically equivalent to GD (whole model or sub-model).
\end{hypothesis}
% \end{minipage}
% % We use $\widehat{\text{ICL}}$ to denote in-context learning with assumed weights or
% \vspace{-0.1cm}
% \end{wrapfigure}
% However, other recent works have focused on a different goal outlined in \cref{hypothesis2}, which deviates from \cref{hypothesis1} in the family of models ($\text{ICL}$ vs $\widehat{\text{ICL}}$: differences in training setups) and family of tasks, as we will see in detail in \cref{subsec:existing}.
\begin{wrapfigure}[9]{R}{0.23\textwidth}
\vspace{-4mm}
    \centering
    \includegraphics[width=0.23\textwidth,trim=1.0cm 9.8cm 20cm 0.6cm,clip=true]{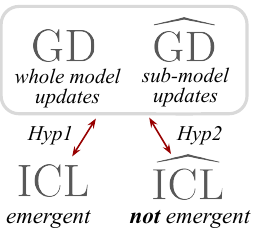}
    % \caption{
    % In search of the equivalence between standard ICL and GD, recent works have focused their efforts on other equivalences. 
    % We show how the standard ICL is not equivalent to GD or GDi.
    % ICL vs. GD, in real practice and the recent analysis. 
    % }
    % \vspace{4mm}
    % \label{fig:equivalence}
\end{wrapfigure}
However, other recent works have focused on a different claim outlined in \cref{hypothesis2}, which focuses on in-context learning behavior that is \textbf{not emergent} (denoted as $\widehat{\text{ICL}}$). 
This deviates from \cref{hypothesis1} in the family of models  (differences in training setups) and family of tasks, as we will see in detail in \cref{subsec:existing}.
This hypothesis articulates a tangential target: being able to \emph{simulate} GD on a given task with \emph{some} 
(trained or hand-constructed) Transformer weights.
This is mainly concerned with the \textbf{expressivity} of Transformer architecture~\citep{merrill2022saturated,chiang2023tighter}, ignoring how it may emerge from pre-training. 
% impose unrealistic assumptions on model parameters as we will see in \cref{subsec:existing}.    
% This is weaker argument that \cref{hypothesis1}, however, it is essentially what is proven in the recent literature. 
% ~\citep{akyurek2022learning,von2023Transformers,dai2022can}. 
% A few notable works aim to provide a theoretical argument for the ICL$\approx$GD claim using this hypothesis. 
A few notable works use this hypothesis to provide a theoretical argument for the ICL$\approx$GD claim. 
Specifically, \cite{akyurek2022learning,von2023Transformers} show (via a different set of arguments) that Transformer-based architectures~\citep{vaswani2017attention}, for appropriate choices of parameters, can process their in-context observations in a way that is equivalent to running gradient updates on an \textit{implicit sub-model}'s parameters using the same demonstrations.

These claims are made under strong assumptions, which raises the question of whether they hold in practice. 
Specifically, \textbf{do the recent results focusing on \cref{hypothesis2} provide any (even partial) evidence for \cref{hypothesis1}?} 
% As a result, there is much room for speculation regarding this equivalence and what these arguments entail in practice. 
Although these works highlight interesting abilities of the Transformer architecture, their claims about the equivalence between ICL and GD are \emph{too strong} for real-world models. 
% Our answer is \emph{unlikely}. 

We divide our study into three parts. 
In the first part (\cref{subsec:existing}), 
we show that previous works that study the ICL$\approx$GD hypothesis make assumptions that are hard to justify in the real world (\cref{hypothesis2}).
Then, we use order-sensitivity as an argument against the equivalence between ICL and GD (\cref{subsec:42}). 
Finally, we put these claimed equivalences to the test (\cref{sec:empirical}) by presenting a comprehensive empirical study. Our experiments reveal that ICL operates and performs differently from GD (fine-tuning the whole model or intuitive sub-models) on real-world language models across a variety of model sizes, datasets and the number of demonstrations.

In summary, 
\begin{enumerate}[leftmargin=6mm,topsep=0pt,itemsep=0mm]
    \item We provide arguments against existing theories of equivalence between ICL and GD, highlighting the gap between their experimental setup and real-world transformers.
    \item We empirically evaluate the equivalence between ICL and GD in the real-world setting using a variety of metrics and find that the two function quite differently.
    \item We call for more nuanced studies that maintain parallels with real-world LLMs so their inferences about ICL can be practically useful.
\end{enumerate}

\section{Background}
\label{sec:background}

We start with our problem setting (\cref{subsec:sampling:models:tasks}). We use ``sampling'' to emphasize \emph{a priori unknown problem parameters}. 
% Namely, sampling (choosing) a learning problem (task) and correspondingly sampling a pretrained model as the computational setup for our study.
Specifically, our computational setup consists of sampling (choosing) a learning problem (task) and correspondingly sampling (training) a pretrained model.
We then cover the two learning setups studied for equivalence (\cref{subsec:standard}), followed by the treatment of ICL$\approx$GD hypothesis in recent literature.

\subsection{Sampling tasks and models}
\label{subsec:sampling:models:tasks}
\paragraph{Sampling from the space of well-defined tasks.}
Consider a family of functions (tasks) $\mathcal{F}$ such that each $(f: \mathcal{X} \rightarrow \mathcal{Y}) \in \mathcal{F}$, maps inputs in the domain $\mathcal{X}$ to the domain $\mathcal{Y}$.
A particular function $f \in \mathcal{F}$ elicits a sampling process $x\overset{f}{\sim}\mathcal{X}$ which samples input from $\mathcal{X}$ such that they are compatible with $f$. For example, in natural language, $\mathcal{F}$ defines the space of all tasks that involve mapping from language input to language output, like sentence completion, summarization, QA, translation, etc. 
However, each task $f$ (e.g., translating English to French) would require specific inputs (English and not, say, German) pertinent to the task.  
The goal is to find models that learn (imitate) $f$ by conditioning on a set of examples $S^f = \left\{S^f_i = (x_i, f(x_i)) \Big| f \sim \mathcal{F},   x_i \overset{f}{\sim}\mathcal{X} \right\}$. The model's competence is then evaluated using a test set $S^f_\text{test} = \{(x^t_i, f(x^t_i))\}$, which is disjoint from $S^f$.
During the evaluation, only the inputs in $S^f_\text{test}$ (which we denote as $X^f_{\text{test}}$) are shown to the model.

% \vspace{-0.1cm}

\paragraph{Sampling from the space of pretrained models.}
LLMs like GPT and LLaMa~\citep{brown2020language,touvron2023llama} are pretrained using the Causal Language Modelling (CLM) objective~\citep{radford2019language} which is more commonly understood as \textit{next-word prediction} objective~\citep{liu2018generating}. This process of pretraining elicits a family of models $\mathcal{M}$ depending primarily on the \textit{data distribution and characteristics of sequences}, and additionally on the choice of architectures, initializations, etc.
% We assume a causal LLM that can map a language sequence to another language sequence, such as GPT and LLaMa~\cite{brown2020language,touvron2023llama}.
% Such LLMs are pretrained using the Causal Language Modelling (CLM) objective~\citep{radford2019language} which is more commonly understood as \textit{next-word prediction} objective~\citep{liu2018generating}. 
Formally, we denote this model $M_{\Theta_{0}}$ with pretrained weights $\Theta_0$, which is one model sampled from a much larger space of low perplexity pretrained models: 
$M_{\Theta_{0}} \sim \mathcal{M}$.

\subsection{Standard Learning Setups}
\label{subsec:standard}
We review the standard treatment of ICL and GD and introduce the relevant notation. 

% Inference-time learning: 
\paragraph{In-context learning (ICL).}
We follow the dominant definition of In-context Learning (ICL)~\citep{brown2020language}, which involves conditioning pretrained LLMs with a handful of examples 
of task $f$. 
Given these demonstrations, we want the LLM to perform $f$ on new inputs. 
% is an approach where a model, rather than adapting through conventional fine-tuning on annotated datasets, conditions itself on a set of limited examples, termed ICL demonstrations. 
% These demonstrations are typically provided in the form of input (context) and output (label) pairs. 
% Given these demonstrations, the model is then expected to produce a label for a \emph{new} input (context) such that it is consistent with the pattern in the demonstrations. 
% \daniel{
% Formally, given any model $M_{\Theta_{0}} \in \mathcal{M}$, 
% % with pretrained weights $\Theta_0$, 
% and any well-defined task $T \in \mathcal{T}$ with}
% % Formally, given a model $M_{\Theta_{0}}$ with pretrained weights $\Theta_0$, 
Formally, given demonstrations $S^f = \{S^f_i\}_{i=1}^{N}$ and a test input $x^t_i \in X_
\text{test}$, the model $M_{\Theta_{0}}$ generates a label $y_{t}$ when presented as $M_{\Theta_{0}}(S^f_{1} \circ S^f_{2} \circ ... S^f_{N} \circ x^t_i)$ or $M_{\Theta_{0}}(x_{1} \circ f(x_{1}) \circ x_{2} \circ f(x_{2}) ... x_{N} \circ f(x_{N}) \circ x^t_i)$, where $\circ$ is a delimiter like new-line which separates the instances. $M_{\Theta_{0}}$ produces a confidence distribution $\in \mathbb{R}^{|V|}$ over the vocabulary set $V$. 

% \begin{equation}
%     \operatorname{Argmax}(M_{\Theta_{0}}(S_{i} \circ x_{t}))
% \end{equation}

% Training-time learning: 
\paragraph{Gradient Descent (GD).}
Gradient Descent is an iterative numerical optimization algorithm used to minimize a given objective with respect to model parameters. Given a model with initial parameters $\Theta_0$ and a differentiable loss function $\mathcal{J} \in \mathcal{Y} \times \mathcal{Y} \rightarrow \mathbb{R}
$, the algorithm updates the parameters toward the negative gradient $\nabla_{\Theta_0} \mathcal{J}$. GD is a standard optimizer used to train neural networks including LLMs. Although there are variants, like SGD and Adam, that work well in practice, we focus our study on vanilla GD, which calculates the gradients and takes a step  (learning rate $\eta$) of fixed size. In the context of learning from a set of demonstrations, pretrained models $M_{\Theta_{0}} \sim \mathcal{M}$ are fine-tuned on a particular task $f$ using GD by updating model parameters. 
% Gradient descent is used to fine-tune the Large Language Model (LLM) with an objective function $\ell$ based on parameters $\Theta$ and dataset $S$. This is accomplished using an SGD optimizer with a learning rate $\alpha$. Specifically, when fine-tuning on $S_N$, we consider two objective functions: 
% (1) \textbf{Unified loss}: Treating $S_N$ as a singular unit and calculating the causal language modeling loss $L_{c l m}$ on $S_N$. This is defined as follows:
% \begin{gather}\label{objective1}
%     \Theta_{1} = \Theta_{0} - \alpha\nabla_{\Theta}L_{clm}(\Theta;S_{N})
% \end{gather}
% (2) \textbf{Separate loss}: 
Formally, parameter updates on the model $M_{\Theta_0}$ are performed for some epochs using the available demonstrations $S^f = \{S^f_i = (x_i, f(x_i))\}_{i=1}^{N}$ as follows:
\begin{gather}\label{eq:gd}
    \Theta_{1} = \Theta_{0} - \eta \nabla_{\Theta} \left(\frac{1}{N}\sum_{\substack{(x_i, f(x_i)) \\\in S^f}} \mathcal{J}\left(M_{\Theta_{0}}(x_i), f(x_i) \right) \right)
\end{gather}
% Then we fine-tune the model with objective defined in \autoref{objective1} and \autoref{objective2} . 
After this process, the model is expected to 
perform this task 
% understand this task better and perform better when 
given a new test sample \textit{directly as input}: $M_{\Theta_{1}}(x^t_i)$.

\begin{wrapfigure}[12]{R}{0.24\textwidth}
\vspace{-5mm}
    \centering
    \includegraphics[width=0.24\textwidth,trim=1.5cm 9.1cm 19.1cm 0.7cm,clip=true]{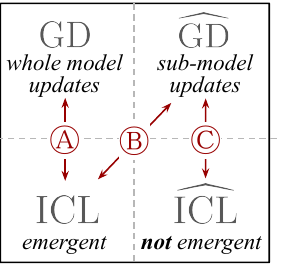}
    \vspace{-6mm}
    
    % In search of the equivalence between standard ICL and GD, recent works have focused their efforts on other equivalences. 
    % We show how the standard ICL is not equivalent to GD or GDi.
    % ICL vs. GD, in real practice and the recent analysis. 
    % ICL vs. GD, in real practice and the recent analysis. 
    % The discussed arguments denoted with numbers.
    % The confusing landscape of ICL vs GD. 
    \caption{\protect\circled{C} is discussed in 
    \cref{subsec:existing}. 
    \protect\circled{A}, \protect\circled{B} in 
    \cref{subsec:42}, \cref{sec:empirical};}
    \label{fig:equivalence}
\end{wrapfigure}
% \section{Problems with the equivalence between ICL and GD}
% \section{Problems with the existing arguments for equivalence between ICL and GD}
% \section{Issues with the existing arguments for ICL-GD equivalence}
% \section{Problems with the existing arguments for ICL-GD equivalence}
% \section{Problems with the current treatment of ICL$\approx$GD hypothesis}
% \section{\textcolor{purple}{Limits} of the current approaches to ICL$\approx$GD hypothesis}
% \section{\textcolor{purple}{The assumptions in the recent treatment} of ICL$\approx$GD hypothesis}
\section{The limiting assumptions in the study of ICL$\approx$GD hypothesis}
% \section{\textcolor{purple}{The limiting assumptions on the current treatment} of ICL$\approx$GD hypothesis}
% \section{Problems with the existing claims for equivalence of ICL and GD }
\label{subsec:existing}

% In this section, we first discuss how the recent investigation of ICL$\approx$GD hypothesis has deviated from the standard setting (\cref{subsec:current:treatment}).
% Subsequently in (\cref{subsec:41},\cref{subsec:42},\cref{subsec:43}) we discuss why these deviations are non-negligible and essentially the recent results offer little in support of the equivalence between ICL and GD in any practical settings. 

% \subsection{Current treatment of ICL$\approx$GD hypothesis}
% \label{subsec:current:treatment}

We highlight how recent studies drift from these conventional definitions of ICL and GD (\cref{subsec:standard}) to support another form of equivalence.
% Notably, this literature makes limiting assumptions on the space of base (pretrained) models $\mathcal{M}$, the space of tasks $\mathcal{F}$ and makes use of unrealistic assumptions on model weights to present this modified equivalence between ICL and GD. 
% We coin $\widehat{\text{ICL}}$ and $\widehat{\text{GD}}$ to refer to frameworks in the prior work under additional assumptions (see \autoref{fig:equivalence}). 
Specifically, they put restrictive assumptions on both the space of 
% \aayush{foundational: needed? introducing for the first time, may create confusion}
models $\mathcal{M}$ and the space of tasks $\mathcal{F}$ when training Transformers. Additionally, they impose impractical assumptions on model weights needed to prove their notion of equivalence between ICL and GD. 
% We coin the terms $\widehat{\text{ICL}}$ and $\widehat{\text{GD}}$ to denote the frameworks used in earlier research, and to distinguish them from real-world settings (see \autoref{fig:equivalence}).
We discuss why these deviations from real practice are non-trivial and offer little support for the equivalence between ICL and GD in practical settings. Fig.\ref{fig:equivalence} encapsulates the theme of our arguments discussed in detail next.
% Specifically, we show by theoretical and empirical arguments, why the equivalence between $\widehat{\text{ICL}}$ and $\widehat{\text{GD}}$, does not imply an equivalence between ICL and GD.
% We address why these deviations are far from minor, and essentially, the recent findings hardly back the similarity between ICL and GD in any real-world applications. 
% More precisely, using both theoretical explanations and empirical evidence, we demonstrate why the parallel between $\widehat{\text{ICL}}$ and $\widehat{\text{GD}}$ doesn't necessarily suggest a parallel between ICL and GD.

\subsection{Real LLMs are \underline{not} pretrained with ICL objective}
\label{subsec:41}
% In both \cite{akyurek2022learning, von2023Transformers}, the claimed equivalence between ICL and GD is based on Transformers trained with the \textit{ICL objective} which is similar to:
% % \begin{gather}
% %     \argmin_{\Theta} \mathop{\mathbb{E}}_{(x_i, f(x_i)) \sim (X_{\text{train}}, f)} \sum_{i = 1}^N \mathcal{L}\left(f\left(x_t\right), M_\Theta \left(x_{1} \circ f(x_{1}) \circ x_{2} \circ f(x_{2}) ... x_{N} \circ f(x_{N}) \circ x_t \right)\right),
% % \end{gather}
% \begin{gather}
%     \argmin_{\Theta} \mathop{\mathbb{E}}_{
%         \substack{
%         T \sim \mathcal{T} \\ 
%         (x_i, y_i)\sim T
%         }
%     } \left\Big[ \mathcal{L}\left\Big(f\left(x_t\right), M_\Theta \left(x_{1} \circ f(x_{1}) \circ x_{2} \circ f(x_{2}) ... x_{N} \circ f(x_{N}) \circ x_t \right)\right\Big) \right\Big],
% \end{gather}
% where $f$ is sampled from the family of functions we want the Transformer parameterized by $\Theta$ to in-context learn. 
% \textcolor{blue}{With this training, the Transformer is shown to achieve dynamics similar to a reference model trained with gradient descent that is initialized with the same weights as the \textit{implicit model} in the Transformer.}\aayush{

% \paragraph{Modifications to the space of pretrained models and tasks.} 
The widely-known ability of ICL \textit{emerges} in pre-trained models ($\mathcal{M}$) that are obtained by training on CLM objective with natural language text as described in \cref{subsec:sampling:models:tasks}. Sequences in the pretraining corpus of natural language have a complicated relationship with the family of tasks $\mathcal{F}$ that they can perform using ICL. Understanding this relationship is an active area of research (cf. \cref{sec:related}). However, we know that the pretraining corpus does not exclusively and explicitly contain sequences pertinent to $\mathcal{F}$. We refer to this training of Transformers with ``natural'' data (not necessarily natural language), which \textit{does not} explicitly train it to perform ICL, as training with the \textit{CLM objective}.

However, recent works use a different set of objectives. In \citet{akyurek2022learning, von2023Transformers, garg2022can}, the models are trained using the \textit{ICL objective}:
\begin{gather}
% \hspace*{-5cm}
\scriptstyle{
    \argmin_{\Theta} \mathop{\mathbb{E}}_{
        \substack{
        f \sim \widehat{\mathcal{F}} \\ 
        x_i \overset{f}{\sim} \mathcal{X}
        }
    } 
    \Big[ 
    \mathcal{L}
    \Big(
    f\left(x_i\right), 
    M_\Theta 
    \left(
    x_{1} \circ f(x_{1}) \circ x_{2} \circ f(x_{2}) \hdots \circ x_{i} 
    \right)
    \Big) 
    \Big].
}
\end{gather}
\vspace{-0.1cm}
This deviates from the real settings in at least two aspects: 

\vspace{-0.05cm}
\paragraph{Changing the space of tasks.} This objective trains the model on the same restricted task distribution that it is tested on via ICL. 
We call this $\widehat{\text{ICL}}$, or the ability to perform ICL by training on \textit{ICL objective} (cf. \autoref{fig:equivalence}) and the corresponding family of tasks $\widehat{\mathcal{F}}$.
% (hence, the name ``ICL objective''). 
For example, if the target task to learn is linear regression, 
% at the test time, 
the model is trained on the sequence of linear regression instances. Therefore, this setup does not necessarily capture the essence of how ICL \textit{emerges} in LLMs, which are not trained to perform ICL on a family of tasks.

\vspace{-0.05cm}
\paragraph{Changing the space of models.} Moreover, optimizing for this objective elicits a family of models $\widehat{\mathcal{M}}$ that is embedded with the inductive bias of expecting a constant structure in the sequence: a series of $(x, y)$ pairs followed with a query input. Combined with the training on sequences specifically related to a restricted family of tasks $\widehat{\mathcal{F}}$, this space of models has different characteristics from the space of models $\mathcal{M}$ defined in \cref{subsec:sampling:models:tasks}.
% The relationship of $\mathcal{M}$ and $\widehat{\mathcal{M}}$ is neither clear nor discussed in these works. 
% Training the Transformer with this objective arrives at weights that encode an \textit{implicit model}. During inference, the Transformer performs dynamics similar to gradient descent on this \textit{implicit model} using samples passed in context. 

The relationship between these sets of models is neither clear nor discussed in these recent works. Therefore, these works essentially equate $\widehat{\text{ICL}}$ with $\widehat{\text{GD}}$ (\circled{C} in \autoref{fig:equivalence}). Although restricted to a stricter family of tasks like Linear Regression is reasonable for analysis, it is important to discuss these distinctions between the setups. Using the term \textit{Transformers} to refer to both these spaces of models and using the term ICL for $\widehat{\text{ICL}}$ are both misleading.

\vspace{-0.1cm}
\subsection{\underline{Hand-constructed} weights and their limits}
% \subsection{Concerns about hand-constructed weights}
\label{subsec:43}

In this section, we analyze the weight matrices constructed by \citet{von2023Transformers} and \citet{akyurek2022learning}. 
As no method is provided to arrive at these weights by training, we place these hand-constructed weights under the umbrella of $\widehat{\text{ICL}}$. Next, we show how they are hard to justify for real-world language models (e.g., LLaMa-7B).
% First we rewrite the proposition from \cite{von2023Transformers}.

% \begin{proposition}
% Given a 1-head linear attention layer and the tokens $e_j = (x_j , y_j )$, for $j = 1, . . . , N$, one can construct key, query and value matrices $W_K, W_Q, W_V$ as well as the projection matrix $P$ such that a Transformer step on every token $e_j$ is identical to the gradient-induced dynamics $e_j \leftarrow (x_j , y_j ) + (0, -\Delta W x_j ) = (x_i, y_i) + P V K^T q_j$ such that $e_j = (x_j , y_j - \Delta y_j )$. For the test data token $(x_{N+1}, y_{N+1})$ the dynamics are identical.
% \end{proposition}

We first re-write the weight matrices of Transformers constructed by \citet{von2023Transformers}. Their proposition states that given a reference linear model $W$, there exist key, query, value, and projection matrices $(W_K, W_Q, W_V, P)$ of a Transformer such that a forward pass in that Transformer is identical to a gradient descent step on $W$, i.e., $e_j \leftarrow (x_j, y_j ) + (0, -\Delta W x_j ) = (x_i, y_i) + P V K^T q_j$.

The weight update $\Delta W$ is calculated by the mean squared error loss on the in-context samples as $\Delta W = -\eta\nabla_WL(W) = -\frac{\eta}{N}\sum_{i=1}^N(Wx_i - y_i)x^T_i$.

They construct $W_K = W_Q = \begin{pmatrix} I_x & 0\\ 0 & 0 \end{pmatrix}, W_V = \begin{pmatrix} 0 & 0\\ W_0 & -I_y \end{pmatrix}$ and $P = \frac{\eta}{N}I$, where $I_x, I_y$ and $I$ are identity matrices of size $N_x, N_y$ and $N_x + N_y$ respectively. Using these matrices, they achieve the dynamics of a gradient step in the forward pass of a Linear Self Attention Layer (without softmax). The construction by \citet{akyurek2022learning} is more complex and requires multiple steps to simulate one step of GD on one in-context sample. However, the construction is similar in that it is similarly sparse (see section C.4 in \citet{akyurek2022learning}'s appendix). These constructions raise multiple concerns about their scaling to real-world models.

% \begin{align}
%     \begin{pmatrix} x_j\\ y_j\end{pmatrix} &\leftarrow \begin{pmatrix} x_j\\ y_j\end{pmatrix} + \frac{\eta}{N}I\sum_{i = 1}^N \begin{pmatrix} \begin{pmatrix} 0 & 0\\ W_0 & -I_y\end{pmatrix} \begin{pmatrix} x_i\\ y_i\end{pmatrix}\end{pmatrix} \otimes \begin{pmatrix} \begin{pmatrix} I_x & 0\\ 0 & 0\end{pmatrix} \begin{pmatrix} x_i\\ y_i\end{pmatrix}\end{pmatrix}\begin{pmatrix} I_x & 0\\ 0 & 0\end{pmatrix}\begin{pmatrix} x_j\\ y_j\end{pmatrix}
%     \\
%     &= \begin{pmatrix} x_j\\ y_j\end{pmatrix} + \frac{\eta}{N}I\sum_{i = 1}^N \begin{pmatrix} 0 \\ W_0x_i - y_i\end{pmatrix} \otimes \begin{pmatrix} x_i\\ 0\end{pmatrix}\begin{pmatrix} x_j\\ 0\end{pmatrix}
%     \\
%     &= \begin{pmatrix} x_j\\ y_j\end{pmatrix} + \begin{pmatrix} 0\\ -\Delta Wx_j\end{pmatrix} \label{eq:vonweight}
% \end{align}

\vspace{-0.1cm}
\paragraph{How does the model arrive at the correct $P$?} In the construction by \citet{von2023Transformers}, $P$ is trivially assigned the value $\frac{\eta}{N}I$ which would change with the number of in-context samples. There is no insight into how a Transformer model would arrive at this information and how this formation behaves without any in-context samples. An edge case is  $N = 0$ (no demonstrations), which surprisingly makes terms in $P$ go to infinity.

% \begin{figure}[t]%% Put a picture that is too wide. It will continue in to the right column
%   \begin{minipage}[t]{0.52\textwidth}
%     \includegraphics[width=0.919\textwidth,trim=0cm 0.45cm 0cm 0.6cm, clip=false]{perturbation/steps_0.pdf}
%     \caption{GPT-J's ICL ability does not change much over time during training, while the parameters change steadily. `Parameter difference' refers to the average parameter changes across $W_K, W_Q$, and $W_V$ over all layers. More results in \autoref{setup:evolve}.}
%     \label{fig:perturb}
%   \end{minipage}
%   \rule{0.5em}{0pt}%
%   \begin{minipage}[t]{0.45\linewidth}
%     \centering
%     \includegraphics[width=0.93\textwidth,trim=0cm 0.48cm 0cm 1.6cm, clip=false]{sparse/ratio_llama-7b_Wv.pdf}
%     \caption{
%     We show that the sparsity ratio $W_{V}$ in LLaMA is much less than previous works required to implement GD. More results are deferred to \autoref{sparse_rate}.
%     }
%     \label{fig:sparsity}
%   \end{minipage}
% \end{figure}

% \vspace{-0.1cm}
\paragraph{Are LLM weights this sparse?} 
The weight construction by \citet{von2023Transformers} has a lot of extremely sparse weight matrices. To be precise, $W_K$ and $W_Q$ would be matrices with $N_x$ terms equal to $1$ in the top left of the diagonal with the rest of $(N_x + N_y)^2 - N_x$ terms equal to zero. For LLaMa, the embedding size of the token vector, $N_x = N_y = 4096$. This means that the sparsity ratio (SR) in the weight matrices should be $\frac{((N_x + N_y)^2 - N_x)}{(N_x + N_y)^2} > 99.99\%$. The sparsity ratio in $W_V$ should be close to $\approx 75\%$ if we assume each element in $W_0$ to be non-zero. 
% The weight construction by \cite{akyurek2022learning} is also similarly sparse. 
In practice, the sparsity ratio is much lower for real-world models like LLaMa and GPT-J. As precisely $0$ values for weights are unlikely, we measured the sparsity ratio in $W_K, W_Q$, and $W_V$ by measuring weights less than a threshold ($\delta$). \autoref{fig:enter-label:left} shows the average sparsity value across layers for LLaMa. Overall, real-world pretrained Transformers have a much lower sparsity ratio than the assumptions.
% \daniel{Need a concluding sentence here.}

% \paragraph{How does ICL evolve during training?} From the given constructions, models need to arrive at very specific weights to be able to perform gradient descent on in-context samples, but in practice, we observe models develop, retain, and improve this ability over time in training when the parameters change significantly (A detailed experimental setup is deferred to \autoref{setup:evolve}). In \autoref{fig:perturb}, we look at how the ability to perform ICL evolves compared with how the model parameters change over time (for each check-pointed GPT-J model). We measure the average parameter changes across all layers across $W_K, W_Q$, and $W_V$. 
% This reveals that real Transformers do not settle on one set of weights (as required by previous works for performing GD) but continue to evolve throughout training. This is in contradiction to the synthetic experiments made by prior works where training Transformers with ICL objective results in weights of submodels that are, not just a single choice of parameters, but a sub-space of parameters leads to in-context learning which should have been discussed by prior works that claim that weights found by optimization match their construction. \textbf{To prove the equivalence between GD and ICL, showing it for a single choice of parameters is not enough.}

\begin{figure}[ht]
    \centering
    \hspace{0.1cm}
    \includegraphics[width=0.98\linewidth]{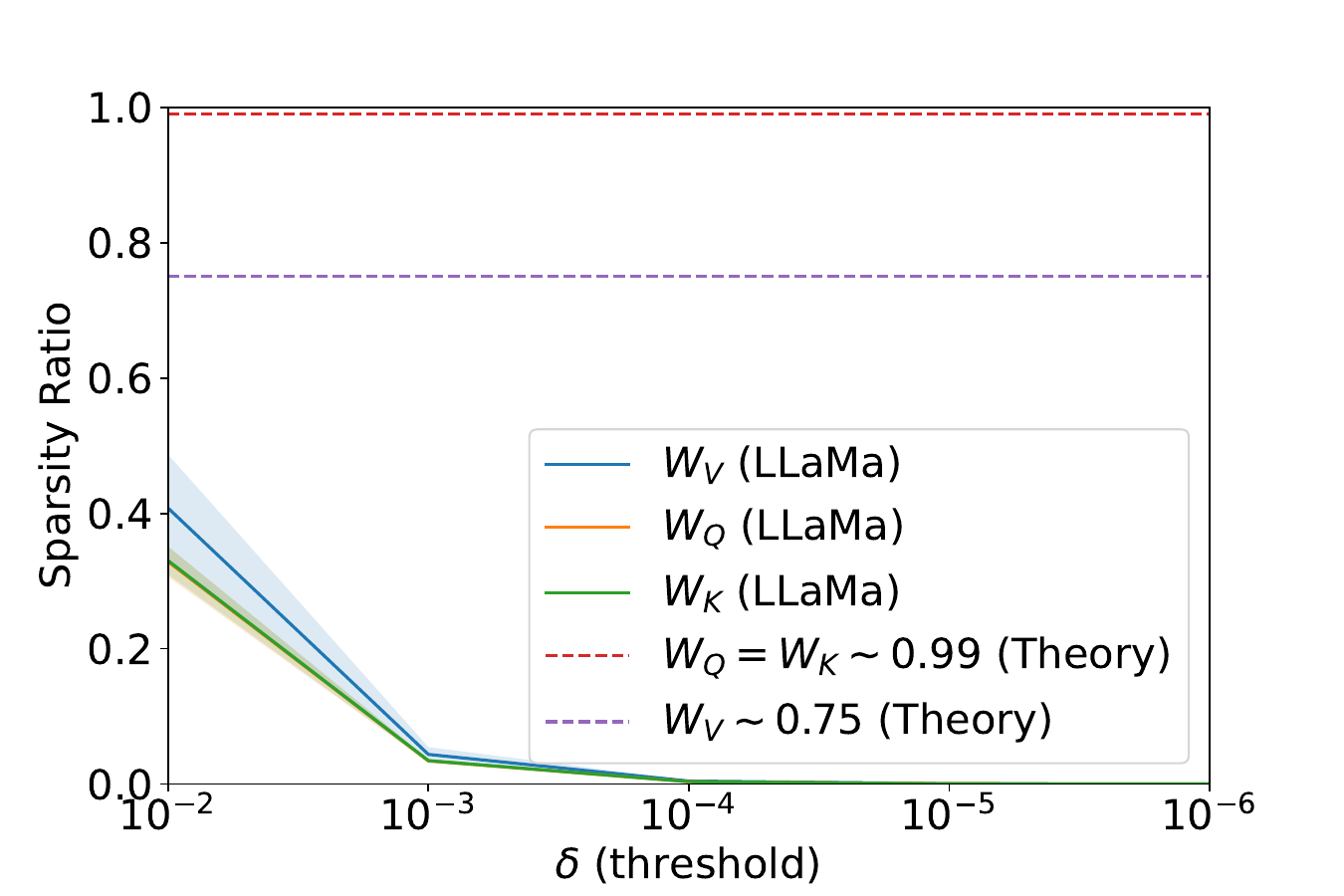} 
    \caption{We show that the sparsity ratio in LLaMA (averaged across layers with standard deviation shown with shade) is much less than required by previous works to implement GD. More plots in \autoref{sparse_rate}.}
    \label{fig:enter-label:left}
\end{figure}

\vspace{-3mm}
\paragraph{How does ICL evolve during training?} From the given constructions, models need to arrive at very specific weights to be able to perform gradient descent on in-context samples, but in practice, we observe models develop, retain, and improve this ability over time in training when the parameters change significantly (A detailed experimental setup is deferred to \autoref{setup:evolve}). In \autoref{fig:enter-label:right}, we look at how the ability to perform ICL evolves compared with how the model parameters change over time (for each check-pointed GPT-J model). We measure the average parameter changes across all layers across $W_K, W_Q$, and $W_V$. 
This reveals that real Transformers do not settle on one set of weights (as required by previous works for performing GD) but continue to evolve throughout training. Although this result is an average over all the weights, certain groups of parameters (as constructed in previous works) are unlikely to remain constant throughout training. Therefore, ICL emerges in real LLMs, not just for a single choice of parameters but a family of parameters. Hence, \textbf{to prove the equivalence between GD and ICL, showing it for a single choice of parameters is not enough.}

\begin{figure}[ht]
    \centering
    \includegraphics[width=0.90\linewidth]{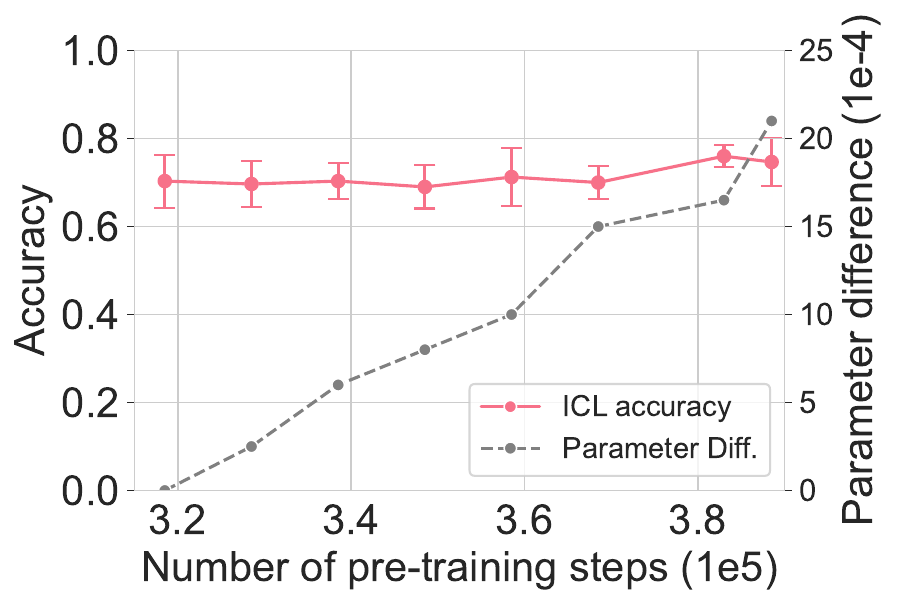}
    \caption{GPT-J's ability to do ICL (on AGNews) does not change much over a time cross-section of training while the parameters change steadily.}
    \label{fig:enter-label:right}
\end{figure}

\section{ICL is likely not equivalent to \underline{order-stable} algorithms}
\label{subsec:42}
% Now that we have established limiting assumptions in previous studies, we want to highlight an important functional e
While we established some limiting assumptions in previous studies,
% (claim the ICL$\approx$GD hypothesis), 
it remains unclear whether ICL$\approx$GD hypothesis is actually invalid for real LLMs (\circled{A} or \circled{B} in \autoref{fig:equivalence}).
% We then shift gears to properties of ICL in real world LLMs. 
For two algorithms to be equivalent, they must also have the \emph{same functional behavior}. Namely, they should respond identically to the changes in the ordering of the instances. In this section, we discuss the discrepant sensitivity of ICL and GD to the order in which they process training instances (demonstrations).

Let's begin with the definition of algorithmic equivalence.

\begin{definition}[Algorithmic equivalence to ICL]
\label{def:opt1}
Consider an optimization algorithm $\mathcal{A}$ that modifies a pretrained model $M_{\Theta_0} \in \mathcal{M}$, using demonstrations $S = \{(x_i, f(x_i)\}_{i=1}^N$ of a well defined task $f \sim \mathcal{F}$, i.e., $\Theta_{S} \leftarrow \mathcal{A}(S, M_{\Theta_{0}})$.  We call $\mathcal{A}$ ``equivalent'' to ICL if and only if the following holds:
\begin{equation}
    M_{\Theta_{0}}(S_{1} \circ S_{2} \circ ... S_{N} \circ x^t) = M_{\Theta_{S}}(x^t) 
    \quad \forall \hspace{2pt} x_i, x^t \overset{f}{\sim}\mathcal{X}.
\end{equation}

% such that $x_i$ and $x^t_i \overset{f}{\sim}\mathcal{X}$. 
% where the objective $\mathcal{J}$ is calculated using $S$ and $M_{\Theta_0}$. 
\end{definition}
The following theorem establishes the equivalence of order sensitivity between ICL and any algorithm $\mathcal{A}$ equivalent to it:
\begin{theorem}[Algorithmic equivalence implies the same order sensitivity]
\label{theorem:order}
    % Let $P, Q$ be two permutations of $\{1, 2, ..., N\}$. 
    Given a pretrained model $M_{\Theta_0} \in \mathcal{M}$, an algorithm $\mathcal{A}$ equivalent to ICL, and demonstrations $S = \{(x_i, f(x_i)\}_{i=1}^N$ of a well defined task $f \sim \mathcal{F}$, let $\sigma_A, \sigma_B$ denote two orders of elements in $S$, such that $\Theta_{\sigma_A} \leftarrow \mathcal{A}(\sigma_A, M_{\Theta_{0}})$ and $\Theta_{\sigma_B} \leftarrow \mathcal{A}(\sigma_B, M_{\Theta_{0}})$. Then, for $\forall \hspace{2pt} x^t \overset{f}{\sim}\mathcal{X}$, we have
    \begin{align}
        &\underbrace{M_{\Theta_{0}}(\sigma_A \circ x^t) - M_{\Theta_{0}}(\sigma_B \circ x^t)}_{\text{The order sensitivity of ICL}}  = \underbrace{M_{\Theta_{\sigma_A}}(x^t) - M_{\Theta_{\sigma_B}}(x^t)}_{\text{The order sensitivity of algorithm $\mathcal{A}$}}, \nonumber
        % , \forall M_{\Theta_0} \in \mathcal{M}
    \end{align}
\end{theorem}
% \vspace{-0.2cm}
\begin{proof}
The proof trivially follows from definition~\autoref{def:opt1}. We know that, $\forall \hspace{2pt} x^t \overset{f}{\sim}\mathcal{X}$ we have:
\begin{align*}
    &M_{\Theta_{0}}(\sigma_A \circ x^t) = M_{\Theta_{\sigma_A}}(x^t)\\
    &M_{\Theta_{0}}(\sigma_B \circ x^t) = M_{\Theta_{\sigma_B}}(x^t).
\end{align*}
Simply subtracting these two terms proves the theorem. 
\end{proof}

% \begin{wrapfigure}[15]{R}{0.45\textwidth}
% \vspace{-9mm}
%     \centering
%     \includegraphics[width=0.43\textwidth,trim=1cm 0cm 0cm 0cm]{Order/gd_agnews.pdf}
%     \vspace{-3mm}
%     \caption{Order Sensitivity (average standard deviation in output probabilities over the vocabulary set) of ICL (\textbf{high}) compared to GD (\textbf{zero}) as measured on the LLaMa-7B model on AGNews dataset. See \cref{subsec:42} for details. The standard deviation is taken across $10$ different orders of 8 ICL demos.}
%     \vspace{-4mm}
%     \label{fig:sensitivity}
% \end{wrapfigure}

% \begin{wrapfigure}[13]{R}{0.42\textwidth}
%     \vspace{-0.7cm}
%     \centering
%     \includegraphics[width=0.44\textwidth]{figs/sparsity.pdf}
%     \vspace{-5mm}
%     \caption{We show that the sparsity ratio in LLaMA (averaged across layers with standard deviation shown with shade) is much less than required by previous works to implement GD.}
%     \vspace{-8mm}
%     \label{fig:sparsity}
% \end{wrapfigure}

% \vspace{-3.0mm}
% \subsubsection{ICL is not equivalent to GD}
% \subsubsection{ICL$\neq$GD based on order inconsistency}
\subsection{ICL is likely not GD based on order inconsistency}
% \subsubsection{ICL is \textcolor{purple}{likely not} GD based on order inconsistency}
\label{subsubsec:iclgd}

Let's assume that GD is equivalent to ICL (arrow \circled{A} in \autoref{fig:equivalence}). 
We show that this assumption leads to a contradiction due to their inconsistent order sensitivity. 

\paragraph{GD is order-stable.} We know that GD is performed on a batch of samples from the training distribution, as seen in \autoref{eq:gd}. It does not matter which order the samples are presented. GD calculates the gradient using the average loss across all samples and is therefore agnostic of the order in which they are calculated. With respect to \cref{theorem:order}, if $\mathcal{A} =$ GD, $M_{\Theta_{\sigma_A}} = M_{\Theta_{\sigma_B}}$ or $M_{\Theta_{\sigma_A}}(x^t) - M_{\Theta_{\sigma_B}}(x^t) = 0$. 
% This means for any two orders of the same training set, $\Theta_1$ obtained after the update in \autoref{eq:gd} will be the same.

\begin{figure}%[15]{R}{0.42\textwidth}
    \centering
    \includegraphics[width=0.41\textwidth,trim=1cm 0cm 0cm 0cm]{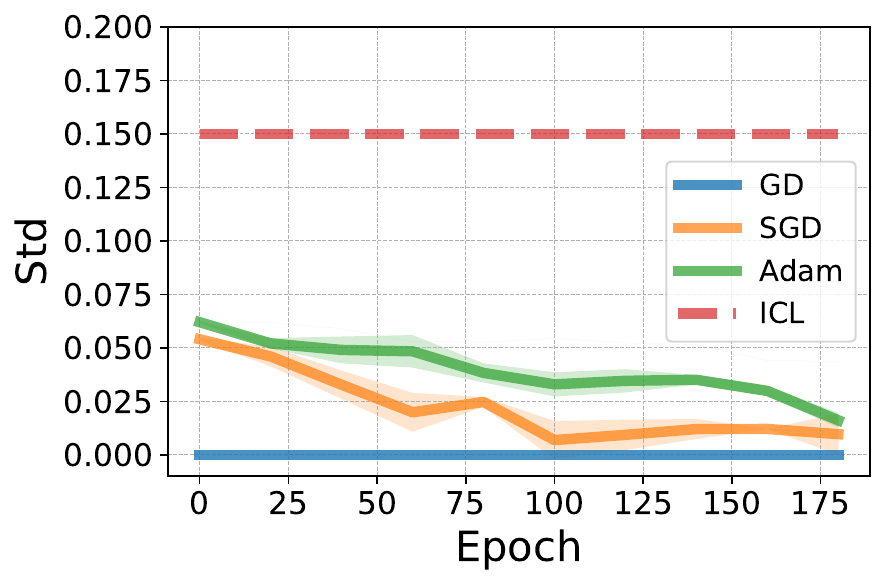}
    % {\includegraphics[width=0.32\textwidth]{Order/gd_agnews.pdf}}
    % {\includegraphics[width=0.32\textwidth]{Order/sgd_agnews.pdf}}
    % {\includegraphics[width=0.32\textwidth]{Order/adam_agnews.pdf}}
\caption{Order Sensitivity (standard deviation in output probabilities over the vocabulary) of ICL and GD (and its variants SGD and Adam) as measured on the LLaMa-7B on AGNews. The std is taken across $10$ different orders of 8 ICL demos. More results are deferred to \autoref{extra_order}.}  
\label{fig:sensitivity2}
\end{figure}
\paragraph{ICL and GD show different order-sensitivity.}
% \paragraph{ICL is not equivalent to GD.} 
For ICL to be equivalent to any order-stable algorithm like GD, it must also be order-stable. However, previous research \citep{Lu2022FantasticallyOP,hahn2023theory} has demonstrated that ICL is highly sensitive to the order of in-context samples. This is also easy to see because decoder-only Transformers exhibiting ICL only predict a token based on what they have seen before in the input. A different order of samples would change the behavior of the model. Therefore, ICL can not be equivalent to GD (arrow \circled{A} in \autoref{fig:equivalence}) as claimed by \cite{dai2022can}. 
These conclusions may change upon notable technological shifts (e.g., the architecture of LLMs).
We also empirically verify this phenomenon by comparing the output distributions produced by ICL and GD (\autoref{fig:sensitivity2}). Details 
% of the experimental setup 
% and comparisons with variants of GD are 
are deferred to \autoref{setup:order}.
% \vspace{-0.2cm}

% \begin{figure}[!ht]
% \centering
%     {}
%     {\includegraphics[width=0.32\textwidth]{Order/sgd_agnews.pdf}}
%     {\includegraphics[width=0.32\textwidth]{Order/adam_agnews.pdf}}
% \vspace{-3.0mm}
% \caption{Order Sensitivity (average standard deviation in output probabilities over the vocabulary set) of ICL compared to GD (+ variants) as measured on the LLaMa-7B model on AGNews dataset. See \autoref{subsec42} for details. The standard deviation is taken across $10$ different orders of 8 ICL demos. For SGD and Adam, batch size $= 2$.}  
% \label{}
% \end{figure}
% \vspace{-4mm}

% \subsubsection{ICL is not equivalent to $\widehat{\text{GD}}$}
% \subsubsection{ICL$\neq\textcolor{purple}{\widehat{\text{GD}}}$ based on order inconsistency}
% \subsubsection{ICL 
\subsection{ICL 
is likely not
$\widehat{\text{GD}}$ based on order inconsistency}
\label{subsubsec:iclgdhat}

\paragraph{Gradient Descent on \emph{implicit sub-model} ($\widehat{\text{GD}}$}).
\cite{akyurek2022learning, von2023Transformers} also hypothesize the existence of \textit{implicit sub-models} inside the weights of Transformer models. These sub-models (parameterized to perform linear regression) are constructed into the weights of the Transformer. When the Transformer is presented with in-context samples, it can simulate steps of gradient descent on the regression loss (using these samples) with respect to the sub-model parameters. Formally, for a sub-model with weights $W_0$, the Transformer model $M_{\Theta_0} = M_{\Theta_{0} \setminus W_0, W_0}$ with fixed parameters ($\Theta_{0} \setminus W_0$) would optimize the weights of the inbuilt implicit sub-model ($W_0$) when presented with in-context samples and make its final prediction using updated weights ($W_1$). We refer to this version of GD as $\widehat{\text{GD}}$.

% \daniel{Incorporate this:}
% We first define $\widehat{\text{GD}}$ and then present a corollary of the previous result.\daniel{TODO: revise}
% $\widehat{\text{GD}}$ refers to GD on an \textit{implicit model} as discussed in \cref{subsec:standard}. This 
% contradicting the equivalence between ICL and $\widehat{\text{GD}}$.
% from \cite{von2023Transformers}. To do so, we briefly discuss \cite{von2023Transformers}'s approach to implementing gradient-based optimization. 
% This work shows that Transformers can optimize an \textit{implicit model} using in-context demonstrations in dynamics equivalent to gradient descent. Let the weights of this \textit{implicit model} be denoted by $W$. Let $M_{\Theta_{0}, W}$ denote the model with fixed parameters $\Theta_{0}$ that optimizes this \textit{implicit model} with weights $W$.
% \vspace{-3mm}

Now we define the equivalence of ICL to an algorithm that updates the implicit model only. 

% \daniel{Revise def2 so that it looks similar to the statemnt of def1.}
\begin{definition}\label{def:implicit}
Consider an optimization algorithm $\mathcal{A}$ that modifies the implicit sub-model weights $W_0$ of a pretrained model $M_{\Theta_0} \in \mathcal{M}$, using demonstrations $S = \{(x_i, f(x_i)\}_{i=1}^N$ of a well defined task $f \sim \mathcal{F}$, i.e., $W_{S} \leftarrow \mathcal{A}(S, W_0)$.  We call $\mathcal{A}$ ``equivalent'' to ICL if and only if the following holds, given $\forall \hspace{2pt} x_i, x^t \overset{f}{\sim}\mathcal{X}$:
\begin{equation}
    M_{\Theta_{0} \setminus W_0, W_0}(S_{1} \circ S_{2} \circ ... S_{N} \circ x^t) = M_{\Theta_{S} \setminus W_S, W_S}(x^t)
\end{equation}
and $\Theta_{0} \setminus W_0 = \Theta_{S} \setminus W_S$, i.e., the pretrained model only updates by the sub-models weights.
% Consider an optimization algorithm $\mathcal{A}$ that modifies a reference model parameters $W_0$ using demonstrations $S$: 
% {$W_{S} \leftarrow \mathcal{A}(S, W_0)$}.  We call $\mathcal{A}$ ``equivalent'' to ICL if and only if the following equation holds {for any well-defined task $T \in \mathcal{T}$ and
% any pretrained parameters $M_{\Theta_0} \in \mathcal{M}$}:
% \begin{equation}
%     M_{\Theta_{0} \setminus W_0, W_0}(S_{1} \circ S_{2} \circ ... S_{N} \circ x_t) = M_{\Theta_{S} \setminus W_S, W_S}(x_t). 
%     % ; \quad W_S = \mathcal{A}(\mathcal{J}(S, W_0))
% \end{equation}
% where objective $\mathcal{J}$ is calculated using $S$ and initial weights $W_0$
\end{definition}

% \daniel{Where did we define the notation $M_{\Theta_{0} \setminus W_0, W_0}$?}

When the model with implicit sub-model weights $W_0$ is provided with in-context examples, it arrives at updated weights $W_S$ using $\mathcal{A}$ without changing any other weights. This is equivalent to when the model starts with sub-model weights $W_S$ and is provided no in-context examples, so no update happens on the weights via $\mathcal{A}$. Now, based on Definition \autoref{def:implicit} and \autoref{theorem:order}, the following corollary about the equivalence of order sensitivity between ICL and an equivalent algorithm $\mathcal{A}$ also holds:

\begin{corollary}
\label{corr:gdi}
{
For a pretrained model $M_{\Theta_0} \in \mathcal{M}$, an algorithm $\mathcal{A}$ equivalent to ICL (according to \cref{def:implicit}) and two orders $\sigma_A, \sigma_B$ of elements in the demonstration set $S$, $\forall \hspace{2pt} x^t \overset{f}{\sim}\mathcal{X}$},
\begin{align}
    &M_{\Theta_{0} \setminus W_0, W_0}(\sigma_A \circ x^t) - M_{\Theta_{0} \setminus W_0, W_0}(\sigma_B \circ x^t) \nonumber \\&= M_{\Theta_{\sigma_A} \setminus W_{\sigma_A}, W_{\sigma_A}}(x^t) - M_{\Theta_{\sigma_B} \setminus W_{\sigma_B}, W_{\sigma_B}}(x^t)
\end{align}
\end{corollary}

% \daniel{Doesn't it look a bit inconsistent that Corollary1 and Theorem1 show similar result but one of them is "corrolary" and another is "theorem"? }
% \lingfeng{I think it looks good, corollary is an easy or evident consequence of another theorem. So it's OK they look similar}

\paragraph{ICL and $\widehat{\text{GD}}$ show different order-sensitivity.}
% \paragraph{ICL is not equivalent to $\widehat{\text{GD}}$.} 
Let's assume that $\widehat{\text{GD}}$ is equivalent to ICL (arrow \circled{B} in \autoref{fig:equivalence}) according to \cref{def:implicit}. According to the same argument as in \cref{subsubsec:iclgd}, $W_{\sigma_A} = W_{\sigma_B}$ or $\Theta_{\sigma_A} \setminus W_{\sigma_A}, W_{\sigma_A} = \Theta_{\sigma_B} \setminus W_{\sigma_B}, W_{\sigma_B}$ or $M_{\Theta_{\sigma_A} \setminus W_{\sigma_A}, W_{\sigma_A}}(x^t) - M_{\Theta_{\sigma_B} \setminus W_{\sigma_B}, W_{\sigma_B}}(x^t) = 0$. This again implies that for ICL to be equivalent to $\widehat{\text{GD}}$, it must be order-stable. Again, empirical evidence in today's LLMs shows that ICL is not order-stable and hence not equivalent to $\widehat{\text{GD}}$ (arrow \circled{B} in \autoref{fig:equivalence}). These conclusions may change in future.

% \vspace{-0.1cm}
\paragraph{What about variants of GD?} We note that the construction of \citet{akyurek2022learning} allows for order sensitivity in GD as the update is performed on samples one by one instead of the batch update performed by \citet{von2023Transformers}. Although it is unclear which order is used to perform this update, we compared the order-sensitivity of ICL with SGD and Adam (\autoref{fig:sensitivity2}) and found that ICL is still significantly more sensitive to order than SGD/Adam. Therefore, it is unlikely that ICL is equivalent to even variants of GD. We provide more order-sensitivity results in \autoref{setup:order}.

\section{Empirical evalutation of ICL vs. GD/$\widehat{\text{GD}}$ in large pre-trained language models}\label{experiments}

This section provides an empirical evaluation of ICL$\approx$GD equivalence in realistic settings. 
Specifically, we take a language model pretrained on natural data and use it with ICL demos to get ICL outputs. Then, we use the same demos to fine-tune the model using GD and $\widehat{\text{GD}}$, and get their respective output (without ICL demos). Next, we compare these outputs on various metrics to see how well ICL and GD/$\widehat{\text{GD}}$ align in practice. 
% If GD and ICL are equivalent, at least for certain choice of hyperparameters, we would expect to see a close alignment between the outputs of the two models. 
% , at least for certain cases that leads to their alignment. 

\label{sec:empirical}
\subsection{Experimental settings}\label{eval}
\paragraph{Model and benchmarks.}
We choose LLaMa (7B) \citep{touvron2023llama} as our primary model for evaluation. Our model-size comparative studies use the GPT family of models (as discussed later \cref{subsec:52}). For benchmarking, we select the following datasets: AGNews \citep{zhang2015character}, CB \citep{de2019commitmentbank}, SST-2 \citep{socher2013recursive}, and RTE \citep{dagan2005pascal}. 

\paragraph{Experimental setup.}
We evaluate ICL with varying demonstration sizes $N \in \{1, 2, 4, 8\}$ and for GD, we fine-tune the models with the same corresponding ICL demonstrations, experimenting with a variety of learning rates $\{\text{1e-4},\text{5e-4}, \text{1e-5}, \text{5e-5}\}$ over 200 epochs, which ensures the convergence of model. Specifically, the objective function of GD is $\mathcal{J} = \sum_{(x,y) \in S}\mathcal{L}_{\text{clm}}(y;x)$\label{objective_new}, where $\mathcal{L}_{\text{clm}}(y;x)$ is the CLM loss of $y$, given $x$ as the prefix. It is noteworthy that we only use gradients of the label and not the whole prefix to update the model. This is done to keep settings similar to the existing formalisms around ICL$\approx$GD equivalence, where only output loss is calculated.

For $\widehat{\text{GD}}$, it is not trivial to identify the \textit{implicit} sub-model as described in \cref{subsubsec:iclgdhat}. Moreover, it is computationally infeasible to experiment on all possible subsets of parameters to identify the sub-model. Therefore, we use the hypotheses in \citet{akyurek2022learning, von2023Transformers}, to experiment with intuitive subsets. According to \citet{von2023Transformers} the \textit{implicit} model lies in $W_V$ of the Transformer while the probing experiments in \citet{akyurek2022learning} suggest that this iterative optimization happens in top layers of the Transformers. This guides us to choose two intuitive subsets to simulate $\widehat{\text{GD}}$: 
\begin{enumerate}[leftmargin=5mm,topsep=0pt,itemsep=0pt]
    \item $W_V$ of a single deep layer.
    \item $W_V$ of a single middle layer (for comparison).
\end{enumerate}
Overall, we compare ICL to GD, $\widehat{\text{GD}}$ (mid), and $\widehat{\text{GD}}$ (deep). Exact details about this setup are deferred to \autoref{gd_sub}.

\begin{figure*}[ht]
\centering

    \subfigure[\textit{Accuracy} comparison of ICL and GD variants.  
    % GD, $\widehat{\text{GD}}$ (8 layers), $\widehat{\text{GD}}$ (1 middle layer) and $\widehat{\text{GD}}$ (1 deep layer)
    ]{
    \includegraphics[width=0.28\textwidth]{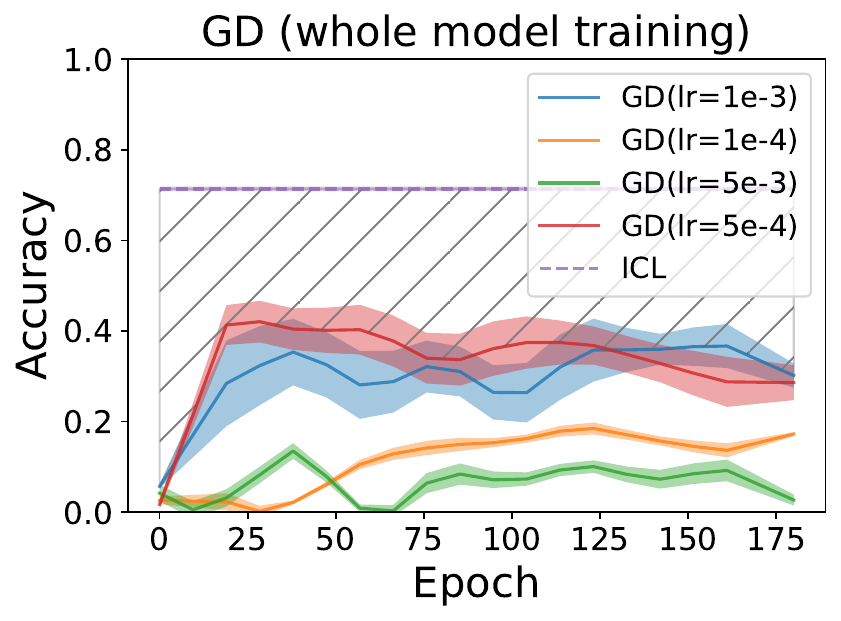}
    \includegraphics[width=0.243\textwidth,trim=1.1cm 0cm 0cm 0cm,clip=true]{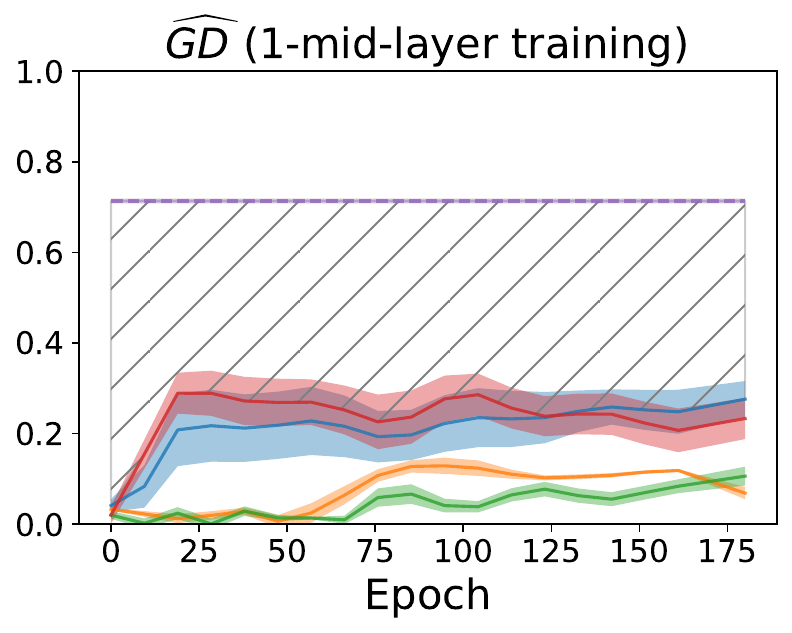}
    \includegraphics[width=0.243\textwidth,trim=1.1cm 0cm 0cm 0cm,clip=true]{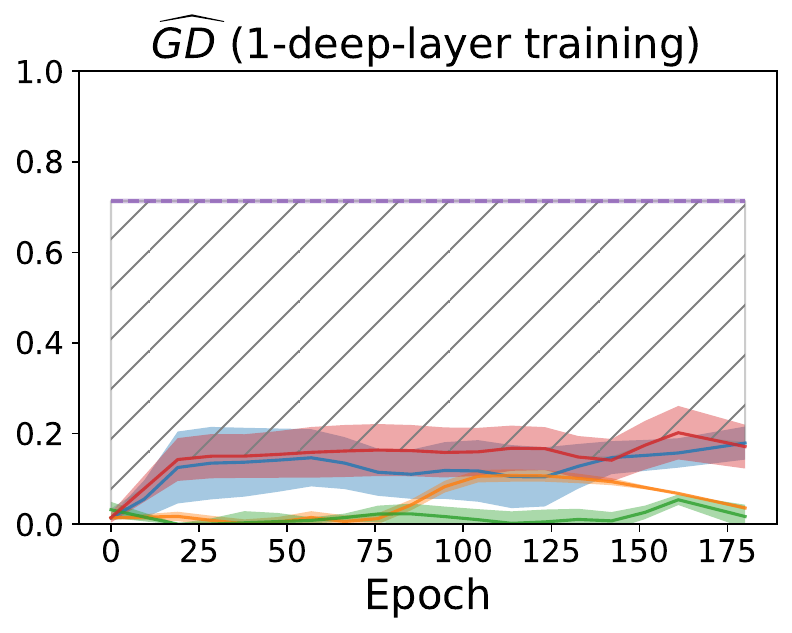}
    }

    \subfigure[\textit{Token Overlap} of ICL with GD variants.]{
    \includegraphics[width=0.28\textwidth]{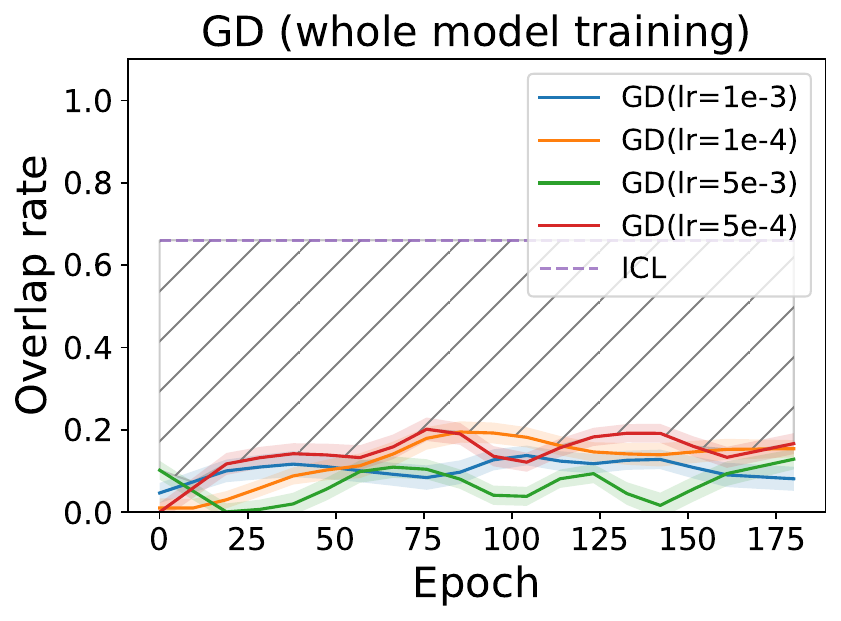}
    \includegraphics[width=0.243\textwidth,trim=1.1cm 0cm 0cm 0cm,clip=true]{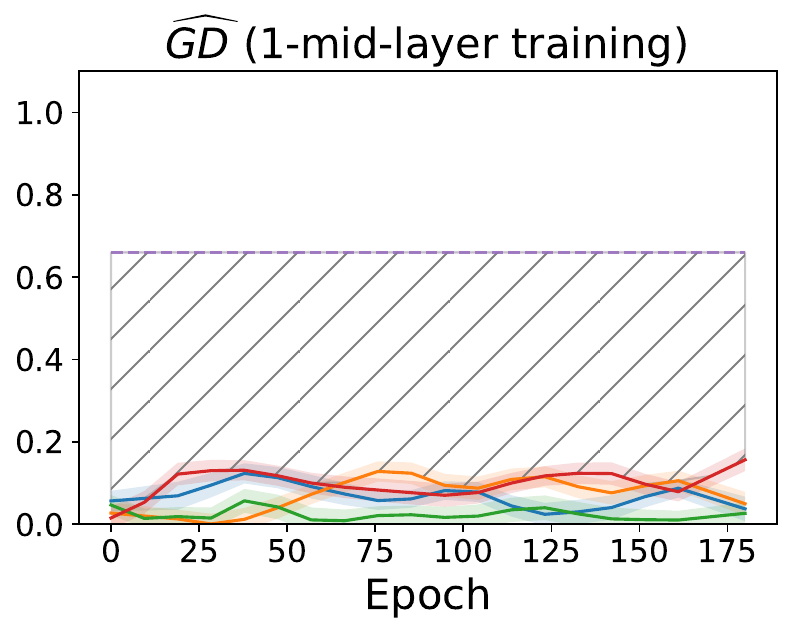}
    \includegraphics[width=0.243\textwidth,trim=1.1cm 0cm 0cm 0cm,clip=true]{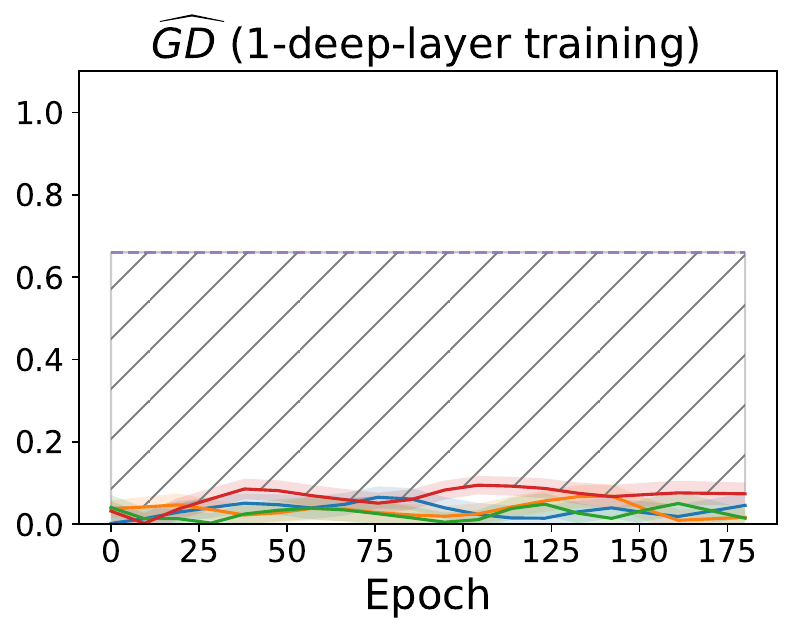}
    } 
    
    \subfigure[\textit{Overlap Cosine Similarity} of ICL with GD variants.]{
    \includegraphics[width=0.28\textwidth]{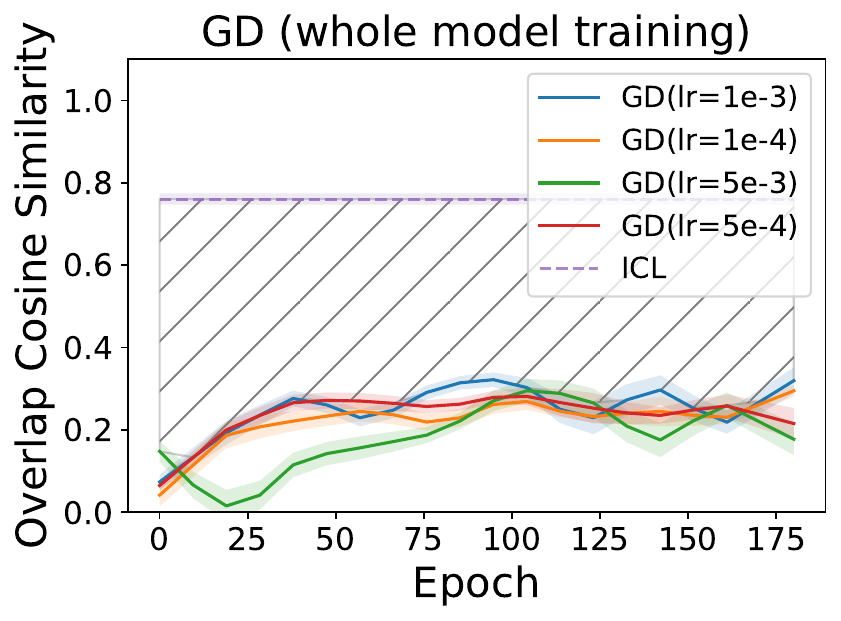}
    \includegraphics[width=0.243\textwidth,trim=1.1cm 0cm 0cm 0cm,clip=true]{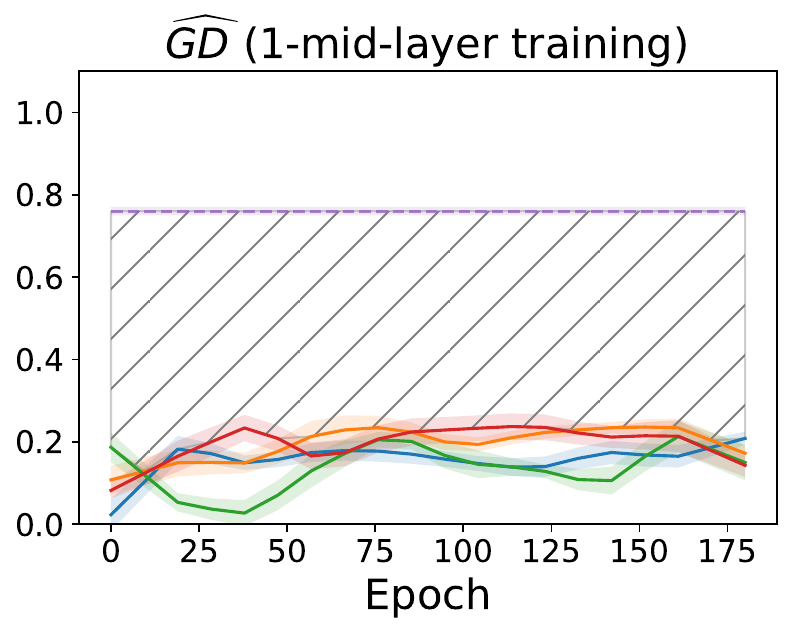}
    \includegraphics[width=0.243\textwidth,trim=1.1cm 0cm 0cm 0cm,clip=true]{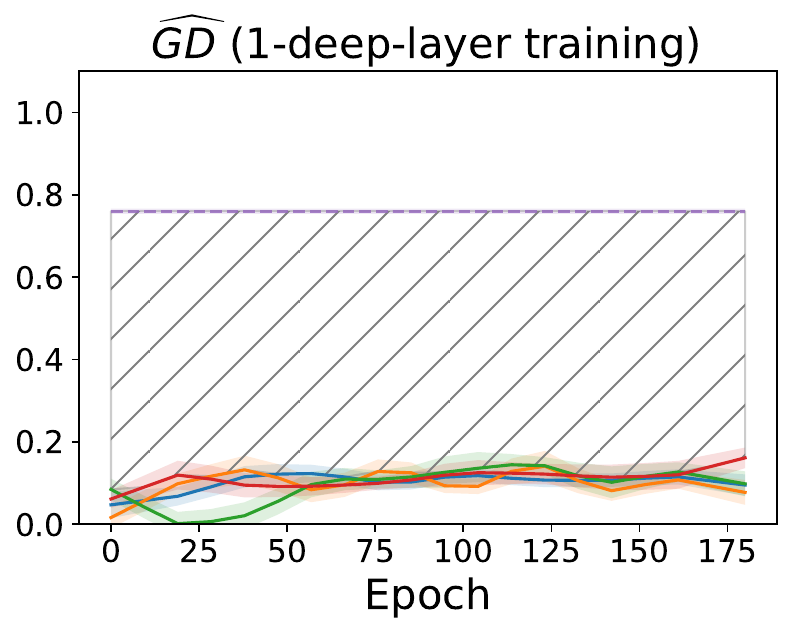}
    }
\caption{Comparison of ICL and GD/$\widehat{\text{GD}}$ on our three metrics for the AGNews dataset (with 4 ICL demos). ICL lines in \textit{Token Overlap} and \textit{Overlap Cosine Similarity} are calculated between two different ICL output distributions (with different order of demonstrations in the prompt). A substantial gap between ICL and GD is highlighted by the gray diagonal lines.}  
\label{unified_gd}
\end{figure*}

\paragraph{Evaluation metrics.}
Previous works often use standard performance metrics (accuracy and loss) based on the token with the maximum probability from \emph{label set} $\mathcal{Y}$ \citep{srivastava2023beyond,wei2021finetuned}. We argue that these metrics do not paint the whole picture. Even if two sorting algorithms reach the same result, their dynamics may differ. For this purpose, we propose to look at \textit{relative uplifting} of tokens in the output distribution. This nuanced analysis presents finer-grained information. A match/mismatch at the distributional level sheds more light on the dynamics of the algorithm. Therefore, we use the following metrics for analysis. 

\textit{Accuracy}: It is calculated using the target labels and predicted tokens with highest probability mass from the whole vocabulary $V$ (rather than just the label set $\mathcal{Y}$) as it better evaluates the model's understanding of the task. It is defined as $\frac{1}{|S_{\text{test}}|} \sum_{(x^t_i, y^t_i) \in S_{\text{test}}} \1 \{ y^t_i =\argmax M(C \circ x^t_i) \} $, where $M$ is the model, $C$ is the context and $S_{\text{test}}$ is the test set. 
    
\textit{Token Overlap}: This is a relative metric which compares two output distributions over the vocabulary $V$. These distributions could be either produced by the same model on different inputs (in case of ICL: different number of demos, order of demos, etc.) or different models on the same inputs (ICL (with context) vs GD (fine-tuned, without context)). We sort the tokens based on their probability mass for each token and select the top-$K$ tokens (denoted by $T^1_K$ and $T^2_K$). 
The token overlap is calculated as $\frac{1}{K} |T_K^1 \cap T_K^2|$.
We use $K = 10$ in our experiments (most of the probability mass typically lies in top-10 tokens).

\emph{Overlap Cosine Similarity (OCS)}: 
Token overlap evaluates each of the top-$K$ tokens with the same weight. With OCS, we measure how well the tokens agree individually. This metric is computed on the confidence distribution of top-$K$ tokens to avoid trivial values (most vocabulary tokens have low probabilities, making OCS $\approx1$). We denote the intersection of the two sets $T_K^1, T_K^2$ by $O = T_K^1 \cap T_K^2$ and use the following formula:
\begin{equation}
    \text{OCS} = \frac{\sum_{t \in O} p^1(t) \cdot p^2(t)}{\sqrt{(\sum_{t \in O}{p^1(t)^2})\cdot(\sum_{t \in O}{p^2(t)^2})\cdot(K - |O|)}}
\end{equation}
% $\frac{\sum_{t_i \in O} p^1(t_i) \cdot p^2(t_i)}{\sqrt{(\sum_{t_i \in O}{p^1(t_i)^2})\cdot(\sum_{t_i \in O}{p^2(t_i)^2})\cdot(K - |O|)}}$. 
Intuitively, this quantifies the cosine distance between the overlapping tokens and assumes all the other tokens have zero overlap, therefore normalizing by $\sqrt{(K - |O|)}$ (when $K = |O|$, we divide by $\sqrt{1}$).
% It measures the cosine distance between the overlapping tokens and assumes all the other tokens have zero overlap, therefore normalizing by $\sqrt{(K - |O|)}$ (when $K = |O|$, we divide by $\sqrt{1}$).
% Similar to \emph{Token Overlap}, this is a relative metric between two distributions $p^1$, $p^2$ over vocabulary $V$. 
% It is different because it accounts for confidences assigned to each token. Specifcially, it is defined as follows: 
% Token overlap evaluates each of the top-$K$ tokens with the same weight. 
% With confidence overlap, we measure the agreement at a finer level, looking at how well the tokens agree at the individual level. 
% This metric is computed on the confidence distribution on top-$K$ tokens. 
% This is done because the vocabulary set is large and most tokens have low probabilities so the overall cosine similarity is always close to $1$. 
% As the top-$K$ tokens may not be the same in the two distributions being compared, we denote the intersection of the two sets $T_K^1, T_K^2$ by $O = T_K^1 \cap T_K^2$ and used the following formula: 
    % Co (T_K^1, T_K^2) = 
% c
% where $O = T_K^1 \cap T_K^2$.
% Intuitively, this quantifies the cosine distance between the overlapping tokens and assumes all the other tokens have zero overlap, therefore normalizing by $\sqrt{(K - |O|)}$ (when $K = |O|$, we divide by $\sqrt{1}$).

%  (details in \autoref{appendix:metrics})

We evaluate every metric across three random seeds and compute the mean and std. Each random seed is used to sample demos for use in ICL experiments. The same demos are used to fine-tune models for GD/$\widehat{\text{GD}}$. Note that for \textit{Token Overlap} and \textit{Overlap Cosine Similarity}, the values for ICL are calculated between predictions made for the same set of demos but presented in a different order in the prompt.

% We evaluate each metric across three random seeds and compute the average and variance. Each random seed is used to sample demos and their order for use in ICL experiments. Each seed uses the same demos in ICL and GD/$\widehat{\text{GD}}$ for consistency.
% For the relative metrics (\textit{Token Overlap} and \textit{Overlap Cosine Similarity}),  the values for ICL are calculated between predictions made for the same set of demos but in a different order. 
% , also underlining the high order sensitivity of ICL.

% \begin{figure*}[ht]
%     \centering
%     \includegraphics[scale=0.3]{Separate_GD_results/Seperate-GPT-J-acc-demo4.pdf}
%     \caption{
%         Performance-level comparison between ICL and GD (Separate loss).
%         }
%     \label{fig1}
% \end{figure*}

% \begin{wrapfigure}{L}{0.4\textwidth}\centering
%     \includegraphics[width=0.4\textwidth]{Unified_GD_results/Unified-GPT-J-acc-demo4.pdf}
% \caption{Performance-level comparison between ICL and GD (Unified loss).}
% \label{unified_gd_performance}\vspace{-7.5mm}
% \end{wrapfigure}

\subsection{Results}
\label{subsec:52}
% With the experimental setup defined in \cref{objective_new}, we investigate the gap between ICL and GD/$\widehat{\text{GD}}$ with our three evaluation metrics. 

\paragraph{Gap between ICL and GD.} \autoref{unified_gd} shows our findings via plots of the three metrics, comparing ICL to various types of GD and /$\widehat{\text{GD}}$). We only show results for one dataset and one demonstration set size here. Other corresponding results are deferred to \autoref{additional} and \ref{gd_sub} due to space constraints. We see a clear gap between them ICL and all variants of GD and $\widehat{\text{GD}}$, across all three metrics, suggesting that these learning mechanisms likely work differently.

\paragraph{Comparing ICL vs. GD, ICL vs. ICL \& GD vs. GD.} In \autoref{unified_gd}, we see that the \textit{Token Overlap} as well as \textit{OCS} are consistently smaller between ICL and GD variants compared to ICL and ICL (with different demonstration order). For completeness, we conducted another experiment on AGNews where we calculated these relative metrics for different GD model checkpoints (say lr=1e-4 at epoch 20 and 1e-5 at epoch 200). Apart from the early epoch checkpoints (when most models have not changed much), most pairs had small \textit{Token Overlap} and \textit{OCS}. This shows how drastically GD based learning changes the model's behavior. With ICL--ICL comparisons, we see significantly higher values which point to a different functional behavior.

\paragraph{Why does GD perform poorly?} As a trend in most datasets and setup variations, ICL outperforms GD and improves faster with increasing size of demonstration set (please see accuracy plots in \autoref{additional} and \ref{gd_sub}). This underlines our understanding about GD which tends to overfit when trained with only few samples. For illustration, we fine-tuned the model with GD using 512 demos and saw a boost in the performance (\autoref{tab:gd_more_samples}). Note that we can not compare this setting (with many demonstrations) with ICL because of the limited context window of LLaMa. Similar to our previous arguments, this also highlights that when a model performs ICL, it does not simply utilize demos like GD, but possibly recognizes the task from the demos and uses its prior knowledge about it to make predictions \citep{pan2023incontext}.
% is We did another experimentFor some set of experimental parameters, the performance metrics might be similar (as also shown in \cite{dai2022can}), but the other nuanced metrics (Overlap rate and OCS) between the two are significantly different, which must result from a different functional dynamic. This is true for both GD and $\widehat{\text{GD}}$, implying that the dynamics of fine-tuning any or all sets of weights on in-context samples results differs greatly from the mysterious ICL. 
% (2) Both token and confidence overlap between ICL and GD\textcolor{purple}{/$\widehat{\text{GD}}$} saturate at a very low level, underlining the functional differences between the two. 

Additional results on other datasets, with different numbers of ICL demos are deferred to \autoref{additional} (GD) and \autoref{gd_sub} ($\widehat{\text{GD}}$). We also present other results about the impact of model size in \autoref{app:model_size}.

% \section{Empirical validation for theoretical arguments}

% \subsection{How fine-tuning affect the ICL ability?}
% \paragraph{Experimental setup.}
% We chose intermediate checkpoints from GPT-J, ranging from 310k to 380k pretraining steps. Using these varied pretraining steps, our approach simulates the fine-tuning process. Specifically, we focus on two metrics to quantify the magnitude of fine-tuning: 
% (1)Step Gap: This represents the difference in pretraining steps between selected checkpoints.
% (2)Parameter Gap: In line with the assumptions made by Oswald et al. \citep{von2023Transformers}, we compute the average differences for each parameter within the $W_K$, $W_Q$, and $W_V$ matrices across different checkpoints.
% To evaluate the ICL capacity of the models, we conducted tests on AGNews, SST-2, CB, and RTE using 8 demonstrations.

% \paragraph{Results.}
% The results are shown in \autoref{fig:perturb}, from where we can observe that there is no significant gap between ICL capacity of different checkpoints, indicating that continue fine-tuning (pretraining) will not substantially hurt the ICL performance.

\begin{table}[]
\caption{Performance of GD (accuracy) increases with more samples, as expected. 
GD with many more demos obtains comparable performance to ICL with fewer demos, highlighting yet another empirical discrepancy.  
}
\label{tab:gd_more_samples}
\small
\centering
\resizebox{0.25\textwidth}{!}{%
\begin{tabular}{
>{\columncolor[HTML]{EFEFEF}}c 
>{\columncolor[HTML]{FFFDFA}}c 
>{\columncolor[HTML]{FFFDFA}}c }
\toprule
\cellcolor[HTML]{EFEFEF}               & \multicolumn{2}{c}{\cellcolor[HTML]{EFEFEF}\textbf{Demos}}                                        \\ \cline{2-3} 
\multirow{-2}{*}{\cellcolor[HTML]{EFEFEF}\textbf{Dataset}} & \multicolumn{1}{c}{\cellcolor[HTML]{EFEFEF}\textbf{8}} & \cellcolor[HTML]{EFEFEF}\textbf{512} \\ 
\midrule
{\color[HTML]{333333} \textbf{AGNews}} & \multicolumn{1}{c}{\cellcolor[HTML]{FFFDFA}{\color[HTML]{333333} 0.42}} & {\color[HTML]{333333} 0.69} \\ 
% \hline
{\color[HTML]{333333} \textbf{CB}}     & \multicolumn{1}{c}{\cellcolor[HTML]{FFFDFA}{\color[HTML]{333333} 0.39}} & {\color[HTML]{333333} 0.72} \\ 
% \hline
{\color[HTML]{333333} \textbf{SST-2}}  & \multicolumn{1}{c}{\cellcolor[HTML]{FFFDFA}{\color[HTML]{333333} 0.49}} & {\color[HTML]{333333} 0.75} \\ 
% \hline
{\color[HTML]{333333} \textbf{RTE}}    & \multicolumn{1}{c}{\cellcolor[HTML]{FFFDFA}{\color[HTML]{333333} 0.36}} & {\color[HTML]{333333} 0.65} \\ 
\bottomrule
\end{tabular}%
}
\end{table}
\section{Related Work}
\label{sec:related}
% We review the relevant literature on the functional interpretation of in-context learning via GD. 
% We delegate other explanations of ICL to \autoref{appendix:related:work} as they are tangential to our focus.  

\paragraph{Functional explanations.} 
Many works offer functional explanations of ICL~\citep{liu2022Transformers,olsson2022context,schlag2021linear}. 
% For example, ~\cite{olsson2022context} hypothesizes the existence of special circuits inside Transformer models, which is responsible for in-context learning. 
% \cite{schlag2021linear} showed an equivalence between linearized self-attention Transformers and fast weight programmers~\citep{schmidhuber1992learning}.
Among these, explanations via GD~\cite{garg2022can,zhang2023trained,ahn2023Transformers} are most pertinent to our work. Notably, 
\citet{akyurek2022learning} showed that Transformers can implement learning algorithms (gradient descent or closed-form OLS) for linear regression problems and empirically showed that the optimality of algorithms implemented experience a \textit{phase shift} with increasing model size. \citet{raventos2023pretraining} discovered similar results about algorithm discovery and phase shifts with increasing task diversity. \citet{dai2022can} similarly showed a dual between attention layers and linear layers optimized using gradient descent. \citet{li2023closeness} showed such an equivalence on softmax regression tasks. 
% \cite{garg2022can,zhang2023trained} then show that using an ICL objective, Transformers can learn not only linear function classes in-domain, but also out-of-domain and even sparse linear functions, decision trees and two layered ReLU neural networks.
% \cite{ahn2023Transformers} show that Transformers trained on random instances of linear regression also develop the ability to perform GD on linear regression tasks. 
Finally, \citet{von2023Transformers} showed a similar construction with a simpler Linear Self-Attention Transformer, claiming that Transformers learn in-context using gradient descent on linear regression problems. Notably, \citet{akyurek2022learning} found this GD behavior applicable only in small models, with bigger models exhibiting Bayes optimal learning behavior (like Ordinary Least Squares for linear regression). In contrast, \citet{von2023Transformers} claimed that bigger Transformers also implement GD with added data transformations. 

Most of this line of work shows how Transformers have the ability to implement such algorithms resulting from training on ICL objectives (\textbf{\cref{hypothesis2}}) and not that real-world models pretrained on natural data develop this ability (\textbf{\cref{hypothesis1}}).
% In this work, we argue how the equivalence between ICL and GD is does not extend to real-world models.
% Furthermore,~\cite{akyurek2022learning,garg2022can} show that in-context trained Transformers can mimic implementing linear regression algorithms (e.g., gradient descent) with detailed derivation. 
% Our work contributes to and complements this line of work by (1) finding that ICL is not always close to CPT and explaining when ICL is very similar to CPT. (2) we try to explain ICL based on the real LLMs instead of Transformers pretrained on linear functions~\citep{akyurek2022learning,garg2022can,giannou2023looped}.

\paragraph{Distributional explanations.} This body of work explains ICL via distributional frameworks and the relevant properties of LLMs~\citep{xie2021explanation,wies2023learnability}. 
\citet{xie2021explanation} explained ICL as implicit Bayesian inference, which implicitly maps a given set of demonstrations to an appropriate latent concept (task) learned via pretraining on a massive unsupervised corpus.  
% which can emerge by pretraining on documents having long-range coherence. 
Similarly, \citet{hahn2023theory} theorized that natural language pretraining data consists of compositional structure, which leads to the emergent ability of in-context learning, while \citet{chan2022data} showed that this might be because of distributional properties of the training distribution (like burstiness). 
These are all reasonable explanations of how ICL works, although they are somewhat tangential to the focus of this study.

\paragraph{Empirical studies.}
Various empirical works study ICL under various settings~\citep{brown2020language,zhao2021calibrate,min2022rethinking,mishra2022reframing,han2023understanding,wang2023selfinstruct}. 
To note a few, \citet{srivastava2023beyond} famously benchmarked ICL for many tasks and models. 
\citet{perez2021true,Lu2022FantasticallyOP} showed the sensitivity of ICL to the choice of demonstrations and their orderings. 
\citet{shin2022effect,razeghi2022impact} showed the sensitivity of ICL performance to the frequency and size of the relevant pretraining corpus. 
% \citet{wang2023selfinstruct} show the feasiblity of LLM alignment with its own feedback, enabled by the guidance provided by ICL. 
\citet{shen2023flatnessaware} treat the ICL prompt selection as an optimization problem. \citet{pan2023incontext} disentangle task recognition and task learning in ICL, which is analyzed in theory recently by \citet{lin2024dual}. These works highlight numerous ways the ability of models to perform ICL changes under different conditions but do not attempt to explain how it functions.

% \vspace{-0.1cm}
\section{Discussion and Conclusion}

% When ICL was first discovered in LLMs~\citep{radford2019language,brown2020language}, the performance of solely pretrained models with a few examples in context could match or even surpass task-specific fine-tuned models. This led many to believe in the idea that the ability to perform ICL is achieved in Transformer models by implementing implicit fine-tuning. 

This work intends to clarify the distinction between naturally emergent ICL (commonly seen in LLMs pretrained on natural text data); \cref{hypothesis1}) vs. 
% and their ability to simulate gradient descent steps 
task-specific ICL as a result of training Transformers for ICL (\cref{hypothesis2}). 
While recent work has shown that Transformers have the \textit{\textbf{expressive capacity}} to simulate gradient-descent in their forward pass, 
% Some recent works attributed this ability to gradient descent by constructing weights that simulate GD on parameters of an \textit{implicit model} in their forward pass. 
% This work presents arguments that casts doubt on these claims. 
% pointing to the contrary of implicit fine-tuning behavior.
% With this work, we present arguments and evidence pointing to the contrary of implicit fine-tuning behavior. 
% While the prior works discover interesting abilities of the Transformer architecture and show that there exist weights of Transformers which \textit{\textbf{can}} simulate fine-tuning on in-context training samples in simpler laboratory settings. 
% In other words, just because Transformer architectures have the capacity to simulate another algorithm, 
this does \textit{\textbf{not}} immediately imply that real-world models \textit{\textbf{actually do}} simulate it. 
% We acknowledge the statement of \cref{hypothesis1}, which defines one universal notion of equivalence between ICL and GD, may be too strong. A more realistic hypothesis may require slightly more restricted conditions, such as equivalence under some distributional properties of the target task or the number of demonstrations. Since the true nature of such conditions are unclear to us, we have chosen the most general statement.
% While \cref{hypothesis2} results in insights about the expressivity of the Transformer, it is unclear whether it reveals much information about \cref{hypothesis1}. 
We hope this work motivates alternative approaches that reveal the true nature of in-context learning in pretrained LLMs.

We recognize that \cref{hypothesis1} establishing a universal equivalence between ICL and GD may be too strong. A more reasonable hypothesis might involve certain restrictions, such as the target task's distributional properties or the number of demonstrations. However, the specifics of such conditions are unclear, so we have opted for a general statement.

% We first presented both theoretical and empirical arguments against construction-based equivalence between ICL and GD that hypothesize the existence of \textit{implicit models} that are updated using GD when in-context demos are provided to a Transformer. We then presented empirical evidence against the equivalence of ICL and GD in the realistic setting where GD refers to models fine-tuned on in-context demos. 
% Moreover, it is impossible for a generic implicit model to fine-tune on complex tasks like Machine Translation from scratch with just a handful of examples because of lack of distribution coverage. \daniel{the message of the previous sentence is unclear.}
Besides using in-context demonstrations, recent work has also discovered other ways in which in-context prompts 
enhance the performance of LLMs. 
For example, appending prompts like ``\textit{Think step by step}'' \citep{kojima2022large} or ``\textit{Take a deep breath and think}'' \citep{yang2023large} before asking a task-specific question has been shown
to improve zero-shot performance of LLMs.  
Such evidence may suggest that an optimization algorithm like GD cannot fully describe the ability of ICL. 
% solely on the basis of theoretical constructions and experiments in oversimplified settings. 
Understanding ICL dynamics requires a more holistic theory which considers the various nuances of this remarkable learning paradigm.

% With this work, we want to emphasize the distinction between In-Context Learning that emerges in LLMs via pretraining on natural text data and other properties of the Transformer architecture like simulation of steps of gradient descent on implicit models. 
% While \cref{hypothesis1} defines one universal notion of equivalence between ICL and other optimization algorithms like GD, it maybe hard to investigate; and although \cref{hypothesis2} results in insights about the expressivity of the Transformer, it does not reveal much information about \cref{hypothesis1}. For understanding the functional form of ICL, we must investigate hypotheses that meaningfully capture all aspects of it and reveal equivalences that tend towards the universal notion captured in \cref{hypothesis1}.

\section{Limitations and Future Opportunities}

% Although we criticize recent works that suggest a broader equivalence of ICL and GD, we still credit them for discovering interesting abilities in Transformers. 
Because of its computationally infeasible nature, we were not able to do an exhaustive search over all sub-models and pinpoint which subset of parameters could correspond to sub-models that could get updated in $\widehat{\text{GD}}$. This could be an interesting avenue of research. Moreover, we do not provide alternate explanations of how ICL works functionally. As ICL is hard to study directly in LLMs, it is natural to turn to simpler settings. But it is imperative that we keep the setups analogous so that inferences from one can be extended to the other.

\section*{Impact Statement}

It is evident that LLMs and their remarkable ability to learn in context have far-reaching impacts in various applications. Understanding the nuances of ICL and its exact functional behavior will uncover the true strengths and limits of LLMs, which is essential to use them reliably. 
A growing line of research shows theoretical expressivity of transformers to simulate gradient descent by training them on ICL objectives. But it is important to differentiate this from the natural ICL that emerges in language models, so that progress towards understanding its true nature is made in the right direction.

\section*{Acknowledgements}

This work is supported in part by ONR grant   N00014-24-1-2089, and generous gifts from Amazon and the Allen Institute for AI. 
We are grateful to the anonymous reviewers for constructive feedback for improving this work. 
We also thank Anqi Liu, Jason Eisner, Holden Lee, Tianjian Li and the anonymous reviewers for their insightful discussions. 
GPU machines for conducting experiments were provided by ARCH Rockfish cluster at Johns Hopkins University (\url{https://www.arch.jhu.edu}).

\bibliography{ref}
\bibliographystyle{icml2024}

%%%%%%%%%%%%%%%%%%%%%%%%%%%%%%%%%%%%%%%%%%%%%%%%%%%%%%%%%%%%%%%%%%%%%%%%%%%%%%%
%%%%%%%%%%%%%%%%%%%%%%%%%%%%%%%%%%%%%%%%%%%%%%%%%%%%%%%%%%%%%%%%%%%%%%%%%%%%%%%
% APPENDIX
%%%%%%%%%%%%%%%%%%%%%%%%%%%%%%%%%%%%%%%%%%%%%%%%%%%%%%%%%%%%%%%%%%%%%%%%%%%%%%%
%%%%%%%%%%%%%%%%%%%%%%%%%%%%%%%%%%%%%%%%%%%%%%%%%%%%%%%%%%%%%%%%%%%%%%%%%%%%%%%
\newpage
\appendix
\onecolumn
\section*{\LARGE{Supplementary Material}}

\section{Order sensitivity of ICL and GD-based algorithms}\label{setup:order}
We present empirical evidence highlighting the distinct sensitivities of GD-based algorithms and ICL with respect to data order. Specifically, we assess the variation in confidence assigned to vocabulary $V$ by the model across different data orderings. 

\paragraph{Experimental setup} We evaluate the order sensitivity of GD-based algorithms using the GD, SGD, and Adam optimizers. The chosen learning rates are 1e-4, 1e-5, 5e-4, and 5e-5. Our experiments are conducted on the AGNews dataset using the LLaMa-7B model. We set the number of demonstrations to 8. GD training continues for 200 epochs to avoid issues of non-convergence,  but is evaluated at every 20 epochs. The number $N$ of random orders $\left\{\sigma_i\right\}_{i=1}^N$ is set as 10 (as the total number of orders are combinatorial).

\paragraph{Evaluation metric ($\text{Sen}$)}\label{sen}
As for the evaluation metric of sensitivity ($\text{Sen}$), it is defined as follows: Given a set of confidence vectors $\left\{p_i\right\}_{i=1}^N$ resulting from distinct data orders $\left\{\sigma_i\right\}_{i=1}^N$, we calculate the standard deviation for each dimensionality within $V$ using the samples $\left\{p_i\right\}_{i=1}^N$. Subsequently, the variances for individual tokens are aggregated.

\paragraph{Results}
In \autoref{fig:sensitivity2}, we presented a high level overview of our findings. In \autoref{fig:order1}, we present it in detail. First, ICL exhibits a much more pronounced data order sensitivity than the three GD-based algorithms. Second, as GD training progresses, its sensitivity diminishes. And third, this happens with both GD and $\widehat{\text{GD}}$. Overall, these findings underscore distinct behaviors of ICL and GD-based algorithms with respect to data order. This suggests a disparity between ICL and GD, as shown in \autoref{theorem:order}.

\textbf{Ablation results}\label{extra_order}

\textit{Batch size}: In \autoref{fig:order2}, we show that a similar trend is seen when we ablate the batch size. 

\textit{Model}: This difference in order sensitivity is not restricted to the LLaMa model. In \autoref{fig:order3}, we show an experiment with the AGNews dataset, where the order sensitivity of ICL is similarly higher than GD variants for other LLMs (like Qwen-7B and GPT-J).

\begin{figure}[H]
\centering
    \subfigure[Order sensitivity of ICL and \text{GD} when batchsize = 1]{
    \includegraphics[width=0.32\textwidth]{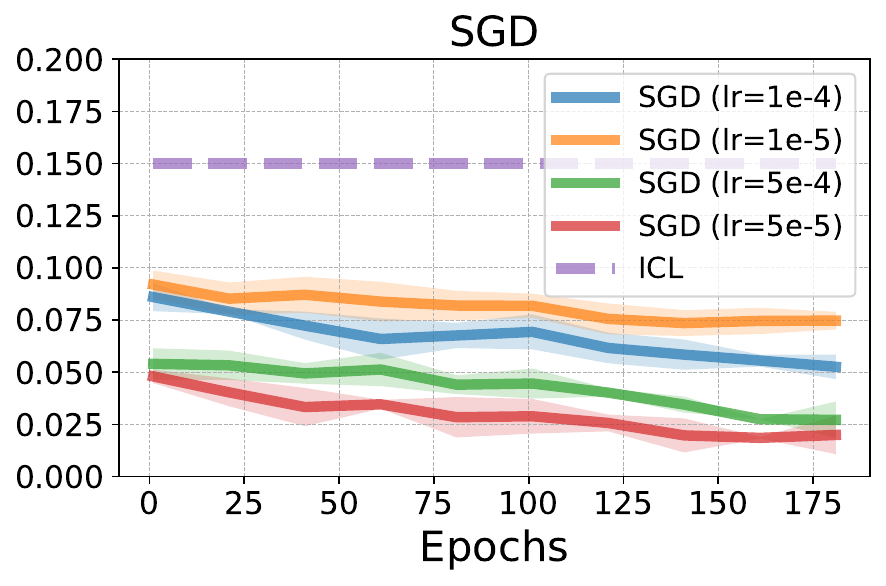}
    \includegraphics[width=0.32\textwidth]{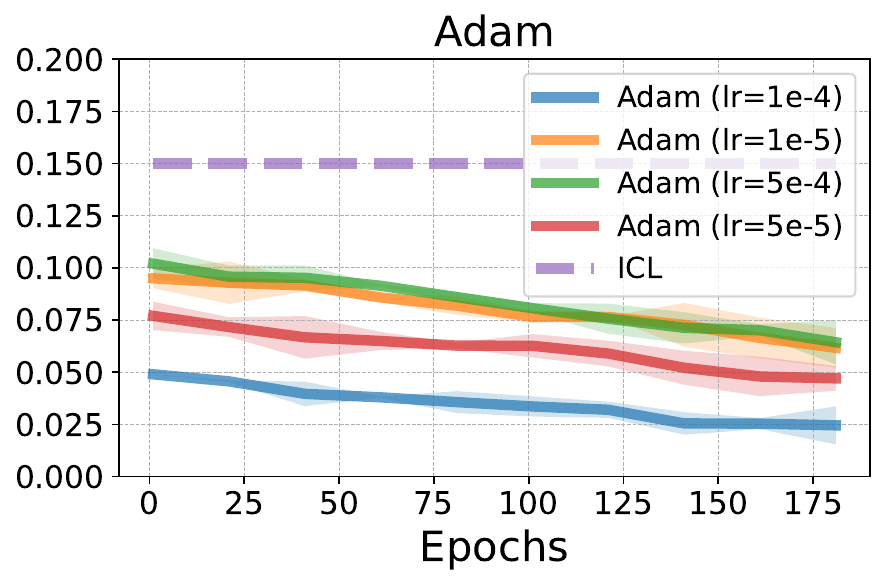}
    }
    \subfigure[Order sensitivity of ICL and $\widehat{\text{GD}}$ when batchsize = 4]{
    \includegraphics[width=0.32\textwidth]{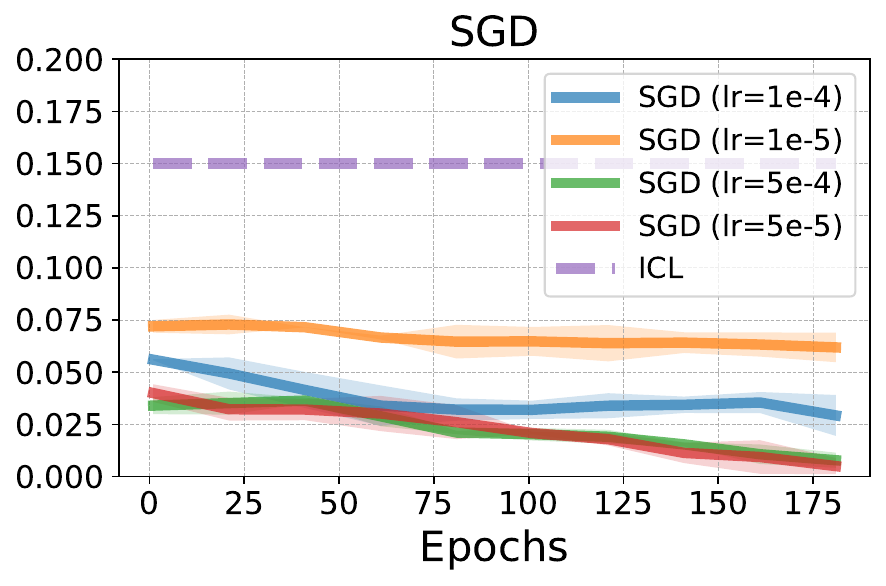}
    \includegraphics[width=0.32\textwidth]{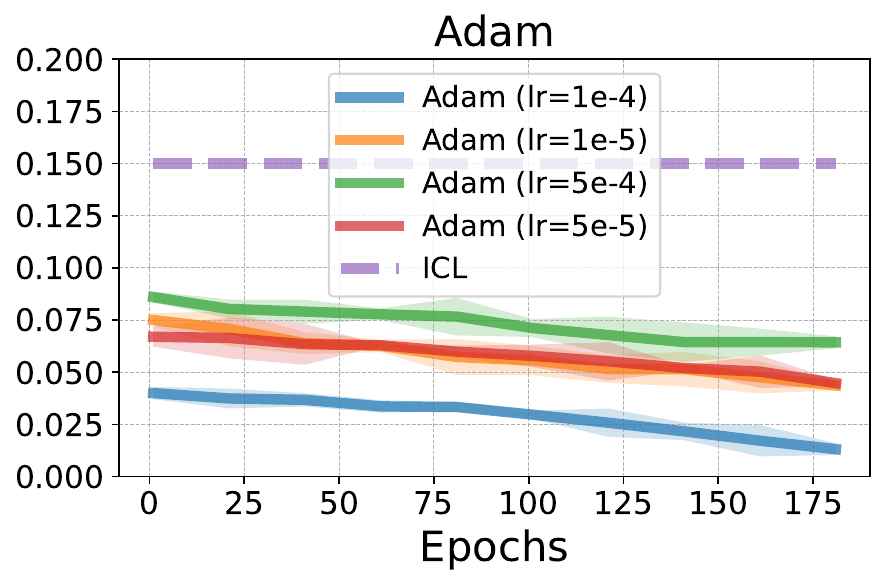}
    }
\vspace{-3.0mm}
\caption{The order sensitivity (y-axis represents $\text{Sen}$ (\cref{sen})) of ICL and GD (SGD and Adam) as the batchsize changes.}  
\label{fig:order1}
\end{figure}

\begin{figure}[H]
\centering
    \subfigure[Order sensitivity of ICL and $\widehat{\text{GD}}$ (SGD)]{
    \includegraphics[width=0.32\textwidth]{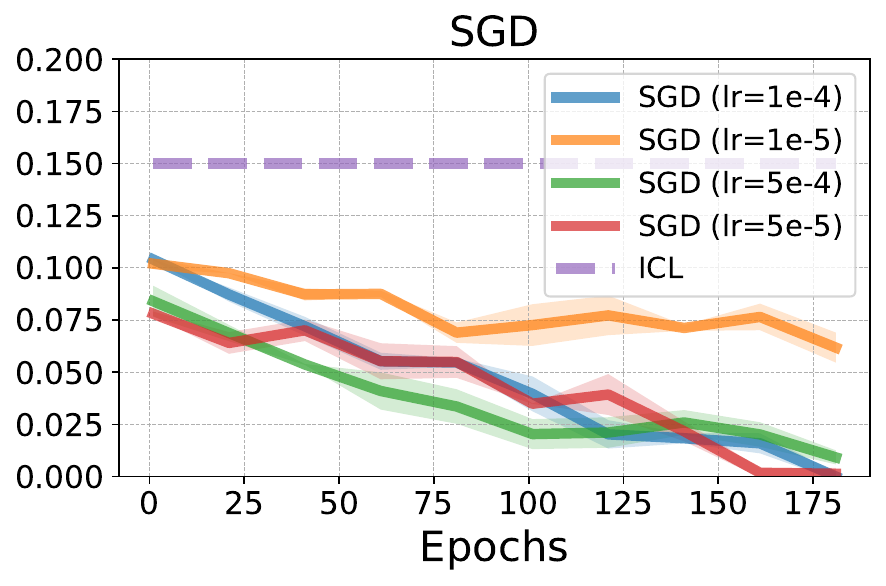}
    \includegraphics[width=0.32\textwidth]{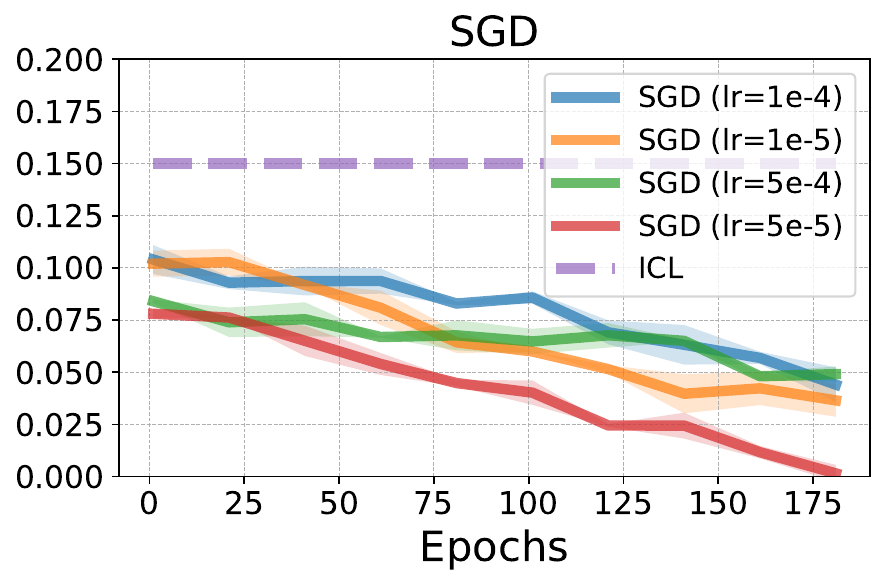}
    \includegraphics[width=0.32\textwidth]{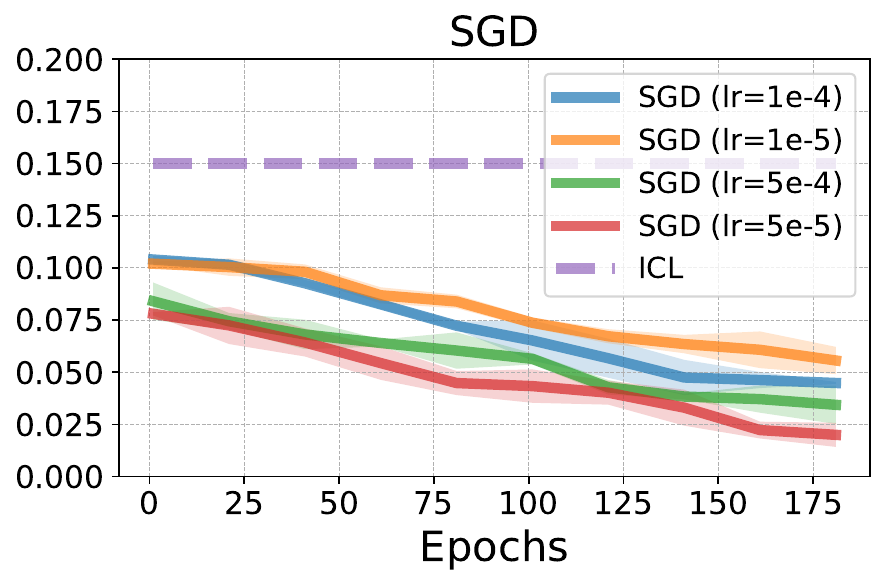}
    }
    \subfigure[Order sensitivity of ICL and $\widehat{\text{GD}}$ (Adam)]{
    \includegraphics[width=0.32\textwidth]{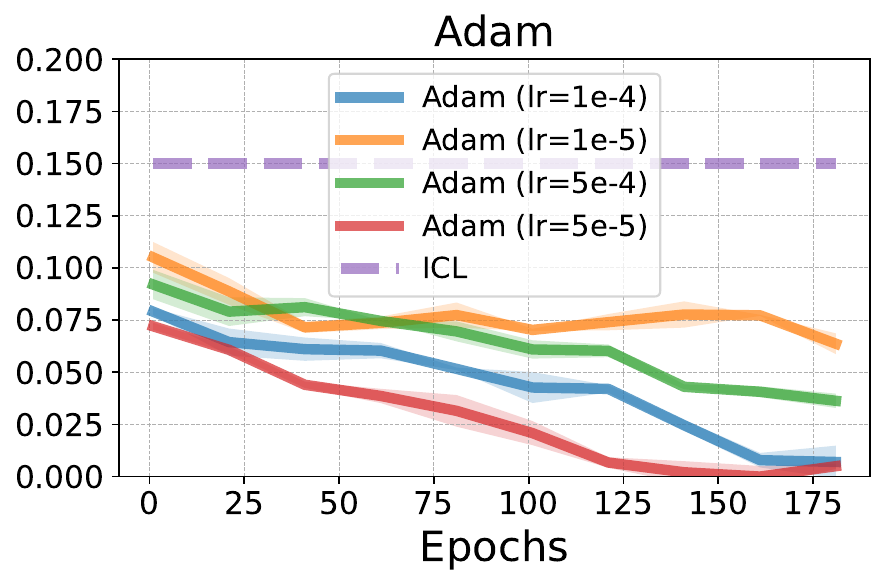}
    \includegraphics[width=0.32\textwidth]{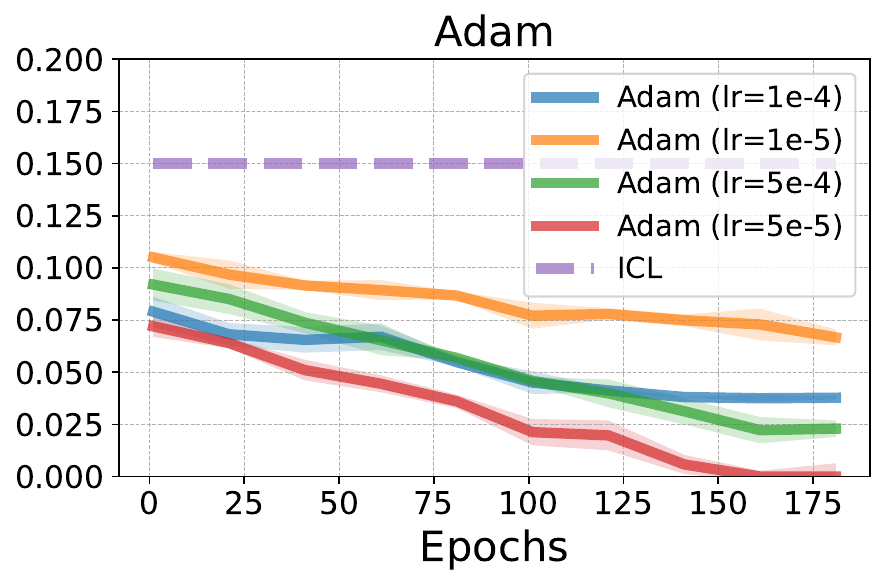}
    \includegraphics[width=0.32\textwidth]{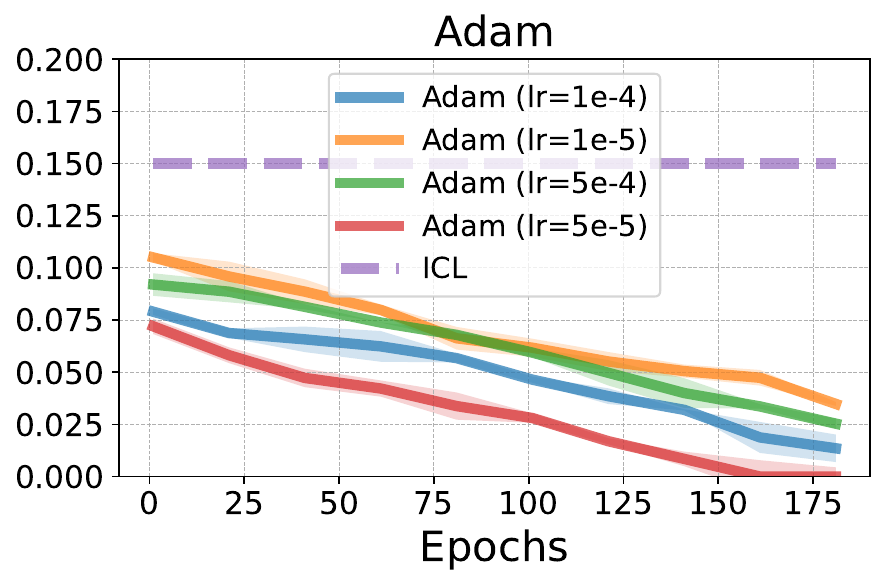}
    }
\vspace{-3.0mm}
\caption{The order sensitivity (y-axis represents $\text{Sen}$ (\cref{sen})) of ICL and $\widehat{\text{GD}}$ (SGD and Adam) as the batchsize changes. From left to right, three figures refer to cases bs=1, 2, 4.}  
\label{fig:order2}
\end{figure}

\begin{figure}[H]
\centering
    \subfigure[Order sensitivity of ICL and ${\text{GD}}$]{
    \includegraphics[width=0.32\textwidth]{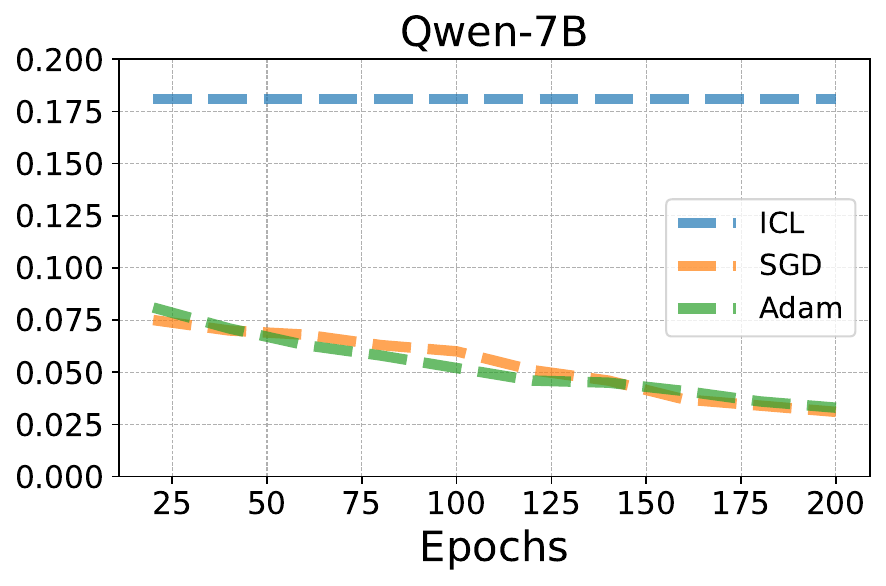}
    \includegraphics[width=0.32\textwidth]{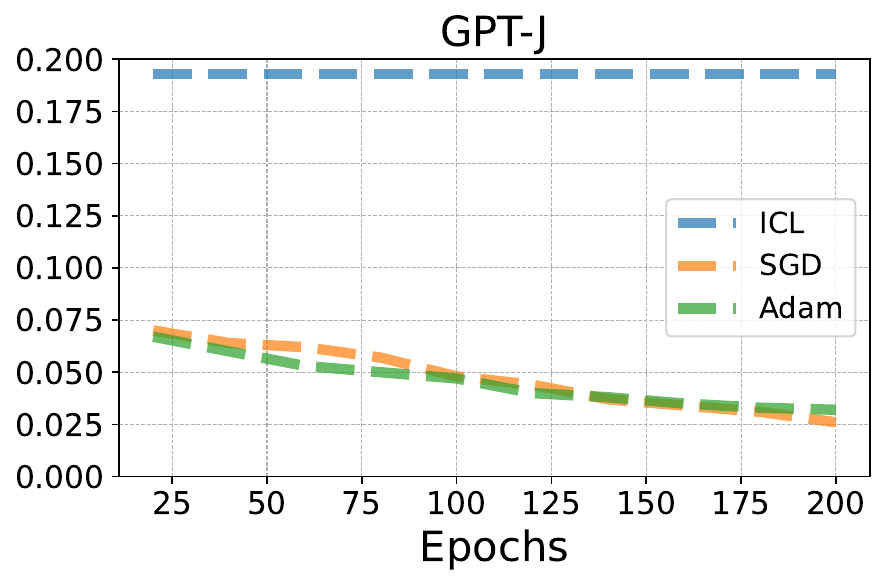}}
\vspace{-3.0mm}
\caption{The order sensitivity (y-axis represents $\text{Sen}$ (\cref{sen})) of ICL and ${\text{GD}}$ (SGD and Adam) on Qwen-7B and GPT-J. The dataset is AGNews, and the batchsize is set as 4.}  
\label{fig:order3}
\end{figure}

\newpage

\section{How does ICL evolve during training?}\label{setup:evolve}
\paragraph{Experimental setup.}
We chose intermediate checkpoints from GPT-J, ranging from 310k to 380k pretraining steps. Using these varied pretraining steps, our approach simulates the fine-tuning process. Specifically, we focus on two metrics to quantify the magnitude of fine-tuning: 
(1) Step Gap: This represents the difference in pretraining steps between selected checkpoints.
(2) Parameter Gap: In line with the assumptions made by Oswald et al. \citep{von2023Transformers}, we compute the average differences for each parameter within the $W_K$, $W_Q$, and $W_V$ matrices across different checkpoints.
To evaluate the ICL capacity of the models, we conducted tests on AGNews, SST-2, CB, and RTE using eight demonstrations.

\paragraph{Results.}
The results are shown in \autoref{fig:perturb2}, from where we can observe that there is no significant gap between ICL capacity of different checkpoints, indicating that continued fine-tuning (pretraining) will not substantially hurt the ICL performance.
\begin{figure}[H]
\centering
    \subfigure[AGNews]{\includegraphics[width=0.48\textwidth]{perturbation/steps_0.pdf}}
    \subfigure[SST-2]{\includegraphics[width=0.48\textwidth]{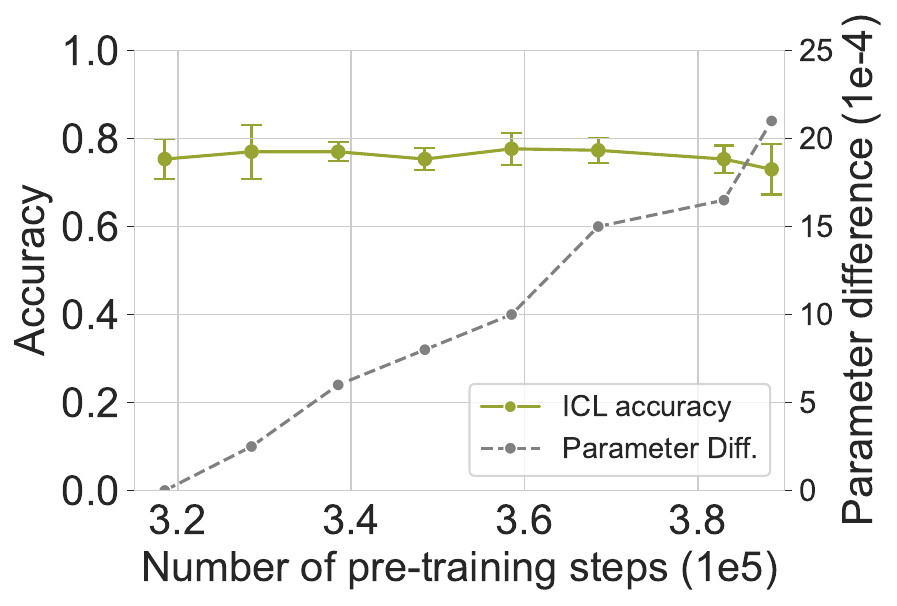}}
    \subfigure[CB]{\includegraphics[width=0.48\textwidth]{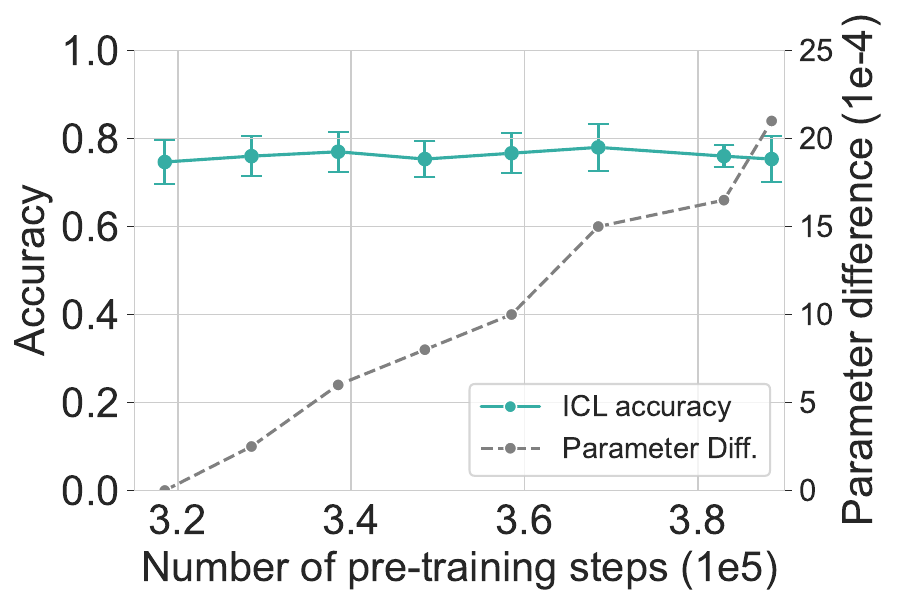}}
    \subfigure[RTE]{\includegraphics[width=0.48\textwidth]{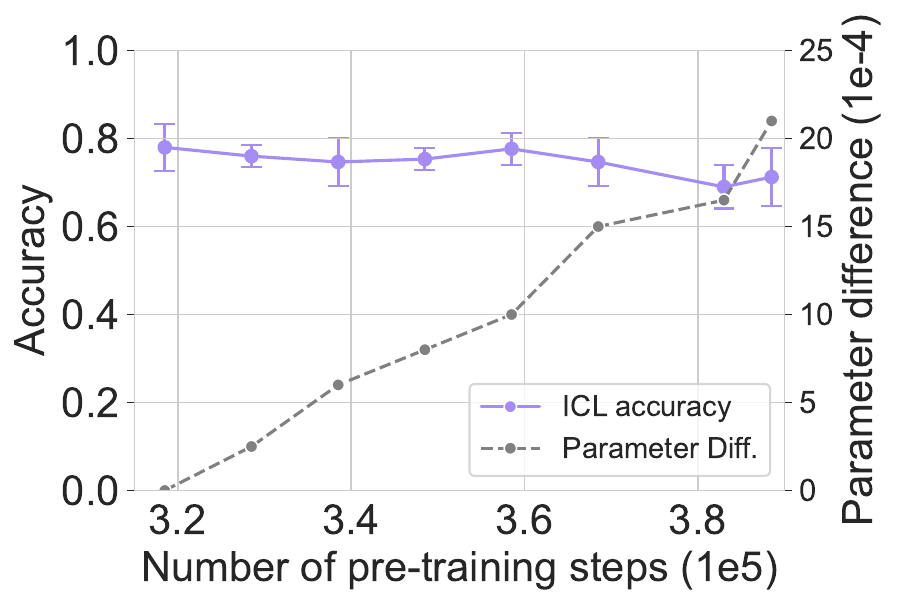}}
\vspace{-3mm}
\caption{The ability of GPT-J to perform ICL does not change much over a time cross-section of training while the parameters change steadily.}  
\label{fig:perturb2}
\end{figure}

\clearpage

\section{Layer-wise sparsity rate of LLMs}\label{sparse_rate}
We show the sparsity ratio of each layer of LLMs. Specifically, in our paper, we have used LLaMa-7B and GPT-J are main experiments, so we show their sparsity rate of $W_{K}$, $W_{Q}$, and $W_{V}$ in each layer. The results are shown in \autoref{fig:sparse}. It is interesting that although $W_{K}$ and $W_{Q}$ have almost constant sparsity in all layers, $W_{V}$ has slightly decaying sparsity.

\begin{figure}[H]
\centering
    \subfigure[Sparsity ratio of LLaMa-7B]{
    \includegraphics[width=0.32\textwidth]{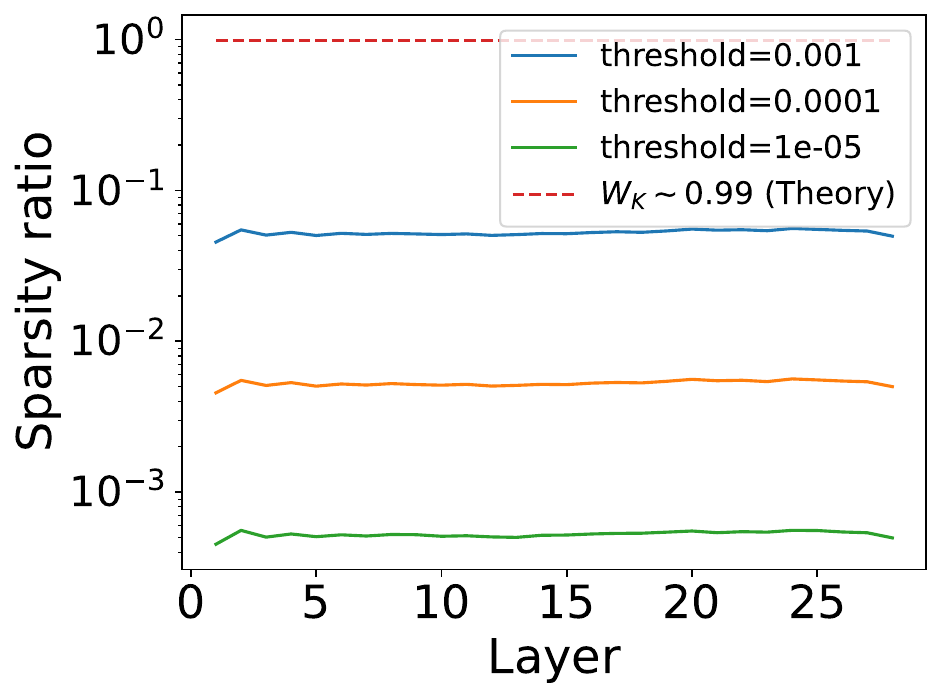}
    \includegraphics[width=0.32\textwidth]{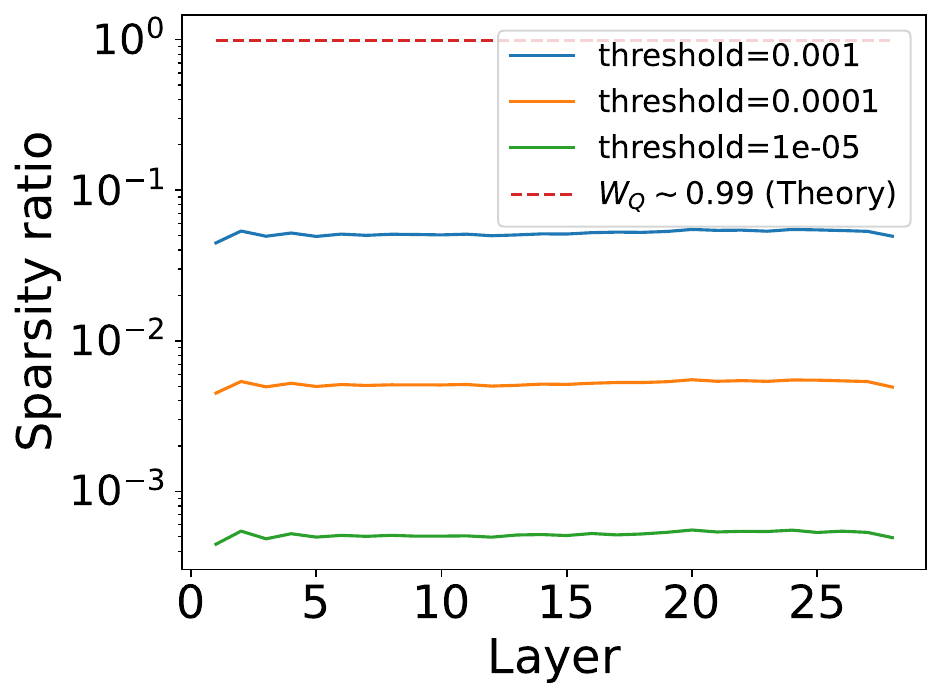}
    \includegraphics[width=0.32\textwidth]{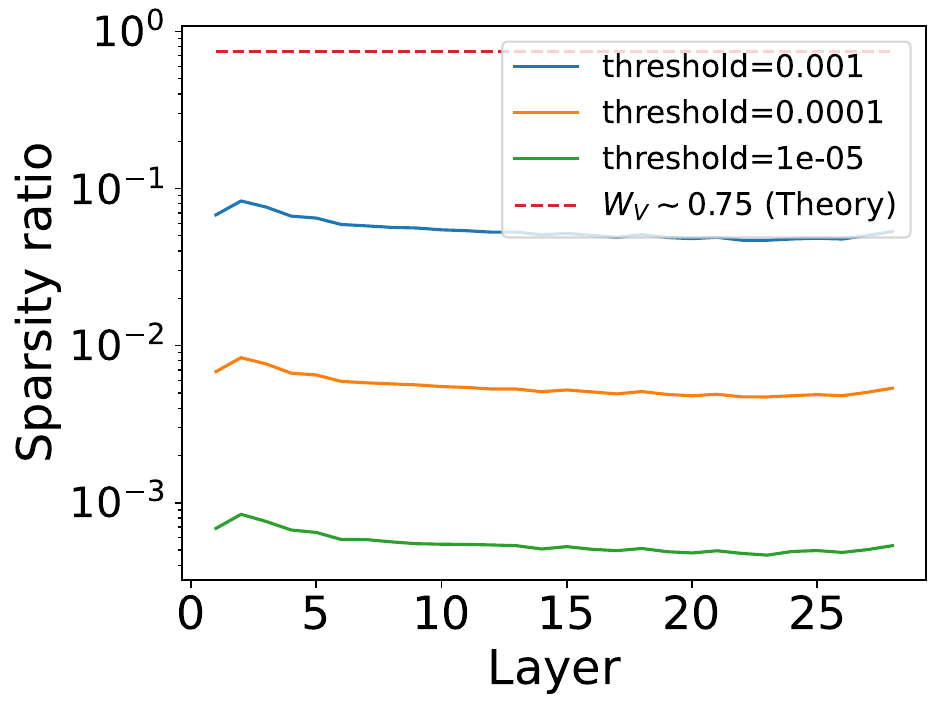}
    }
    \subfigure[Sparsity ratio of GPT-J]{
    \includegraphics[width=0.32\textwidth]{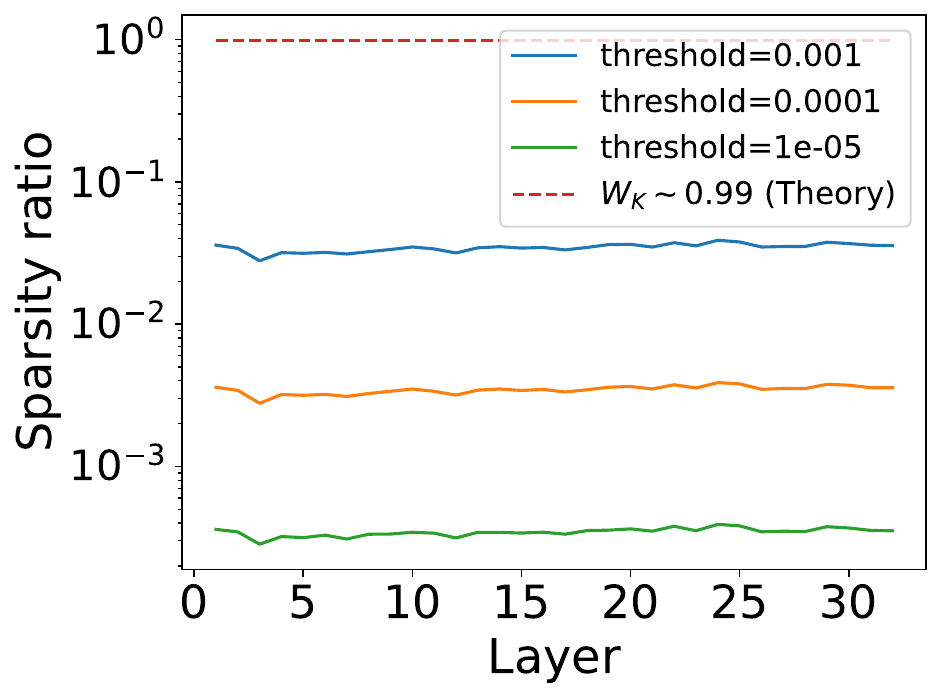}
    \includegraphics[width=0.32\textwidth]{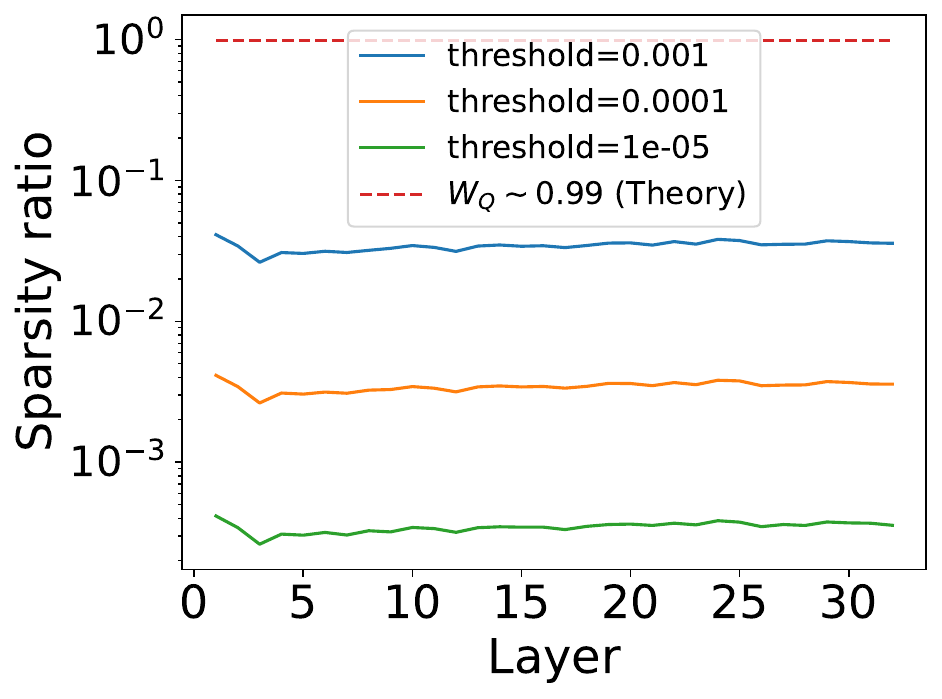}
    \includegraphics[width=0.32\textwidth]{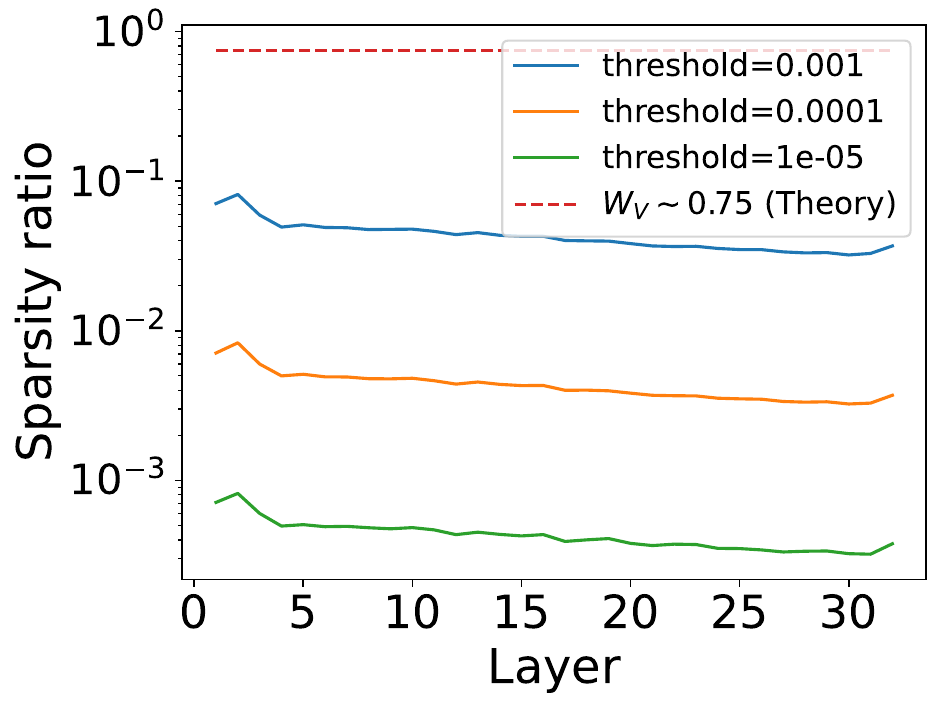}
    }
\vspace{-3.0mm}
\caption{The sparse ratio of LLaMa-7B and GPT-J in each layer. From left to right, three figures represent the cases of $W_{K}$, $W_{Q}$, and $W_{V}$.}  
\label{fig:sparse}
\end{figure}

\clearpage

\section{Additional results on ICL vs GD comparisons}\label{additional}
Here we present all comprehensive plots on ICL vs GD on AGNews and other datasets.

\paragraph{The case of $N=1$.} We see an almost similar accuracy between ICL and one GD variant in all datasets, which is an interesting finding. There are several reasons why this does not directly imply ICL$\approx$GD:
\begin{enumerate} [leftmargin=5mm,topsep=0pt,itemsep=0pt]
    \item There are different GD variants that correspond to the ICL performance in each dataset. This implies the absence of a standard GD-like algorithm that would work on all problems.
    \item Other nuanced metrics show that there is a stark difference in the output distributions of ICL and all GD variants.
    \item The jump in performance from $N=1$ to $N=2$ is typically much more pronounced for ICL than GD. This hints at differences in their functional behavior. 
\end{enumerate}

\begin{figure}[H]
\centering
    \subfigure[\textit{Accuracy} comparison]{
    \includegraphics[width=0.25\textwidth]{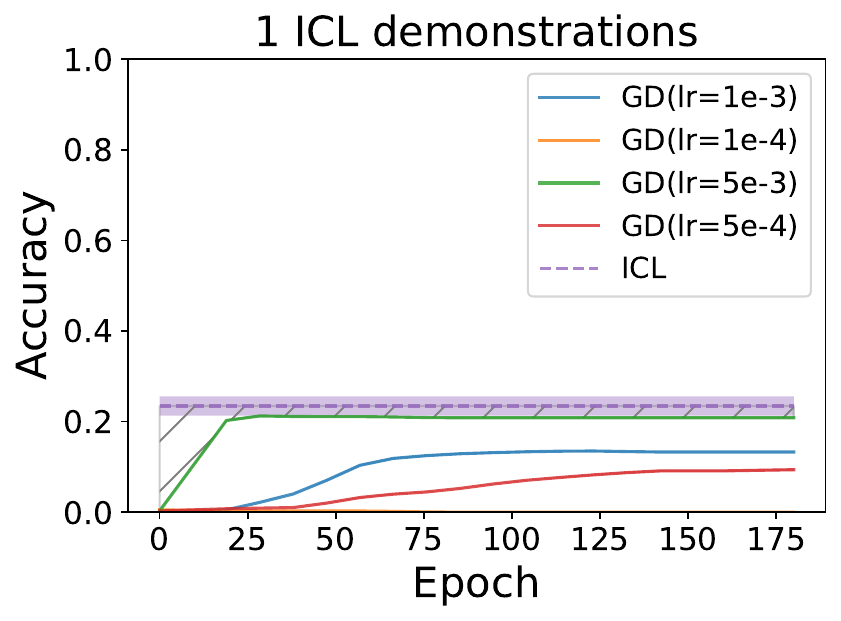}
    \includegraphics[width=0.24\textwidth]{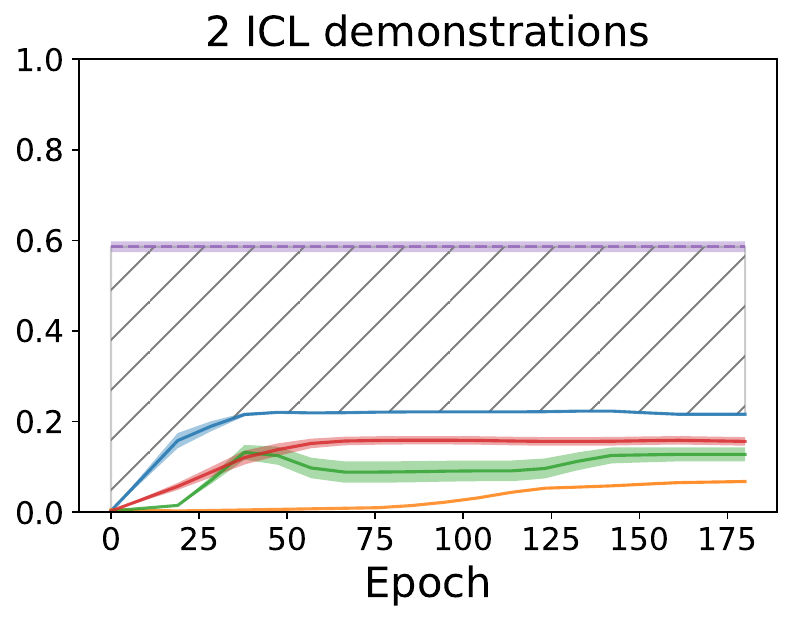}
    \includegraphics[width=0.24\textwidth]{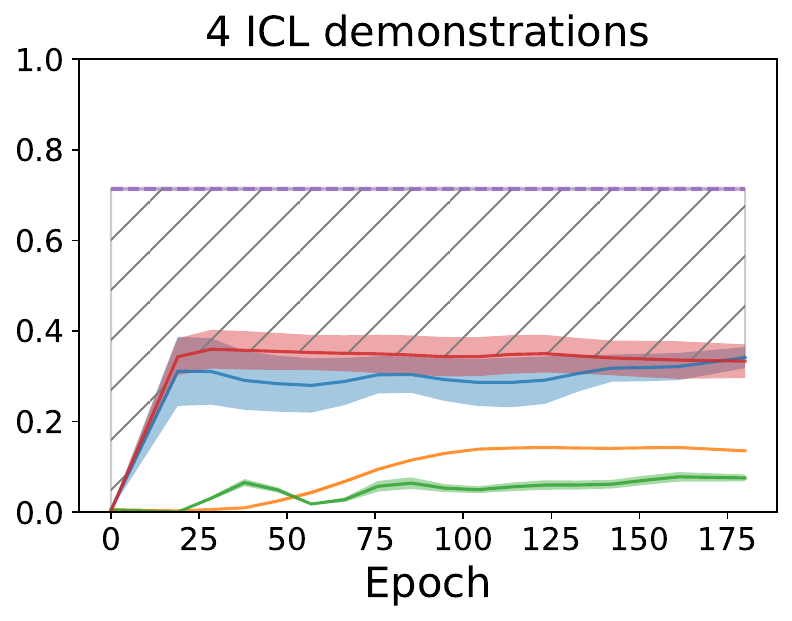}
    \includegraphics[width=0.24\textwidth]{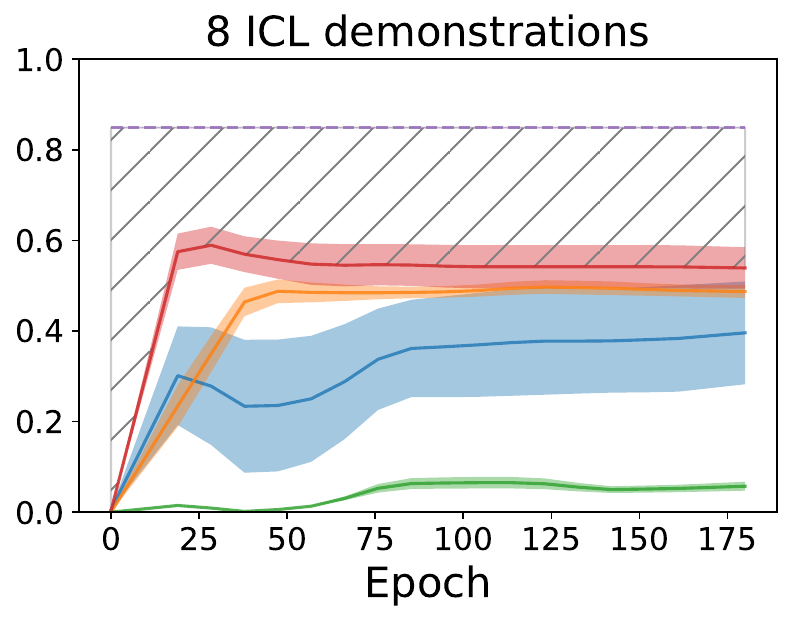}
    }
    \subfigure[\textit{Token overlap} comparison]{
    \includegraphics[width=0.25\textwidth]{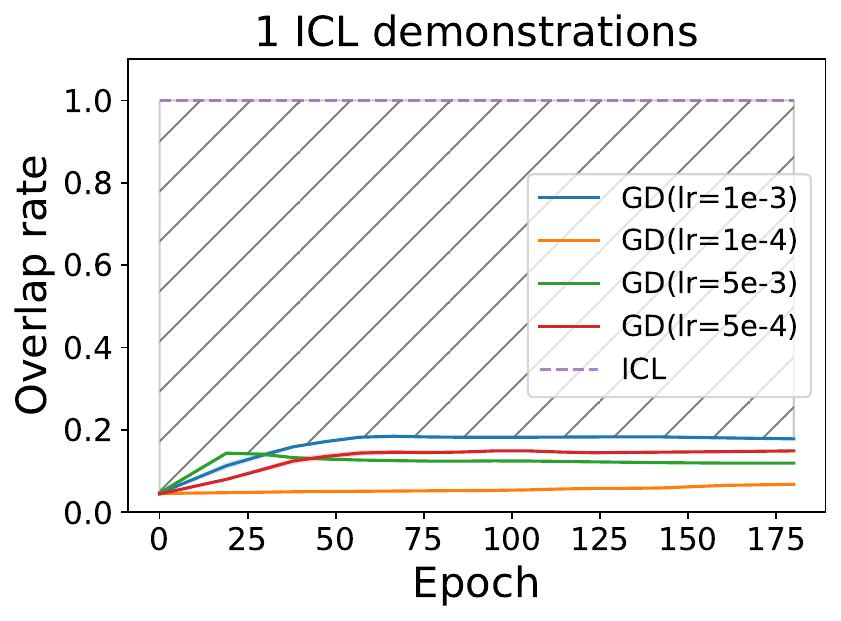}
    \includegraphics[width=0.24\textwidth]{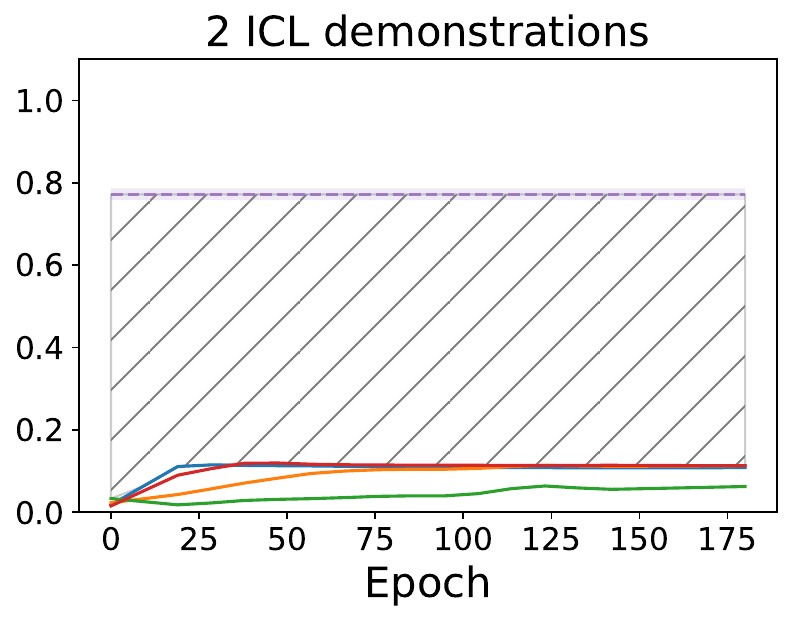}
    \includegraphics[width=0.24\textwidth]{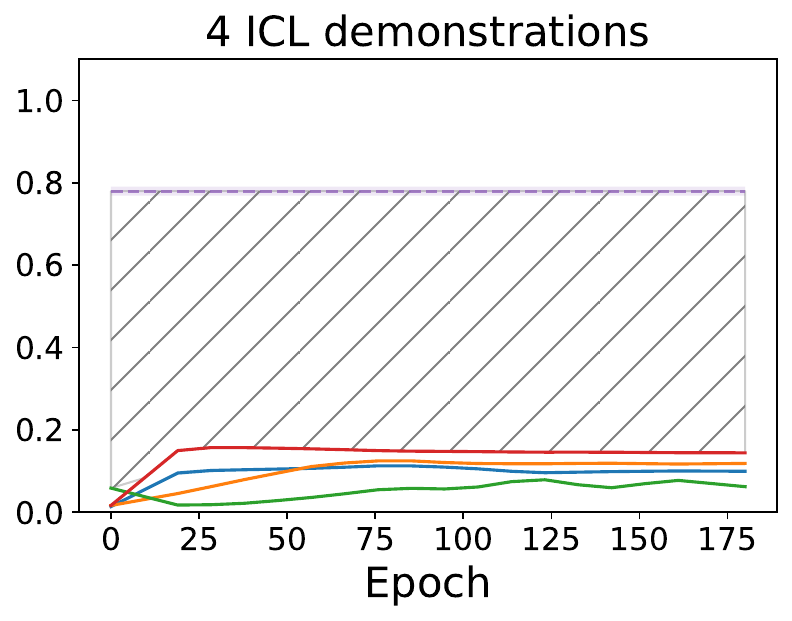}
    \includegraphics[width=0.24\textwidth]{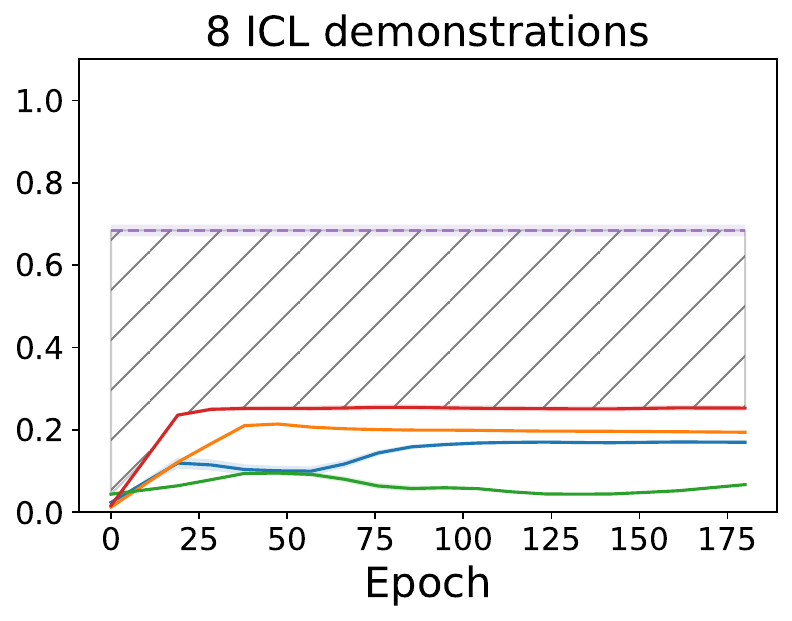}
    } 
    \subfigure[\textit{Overlap Cosine Similarity} comparison]{
    \includegraphics[width=0.25\textwidth]{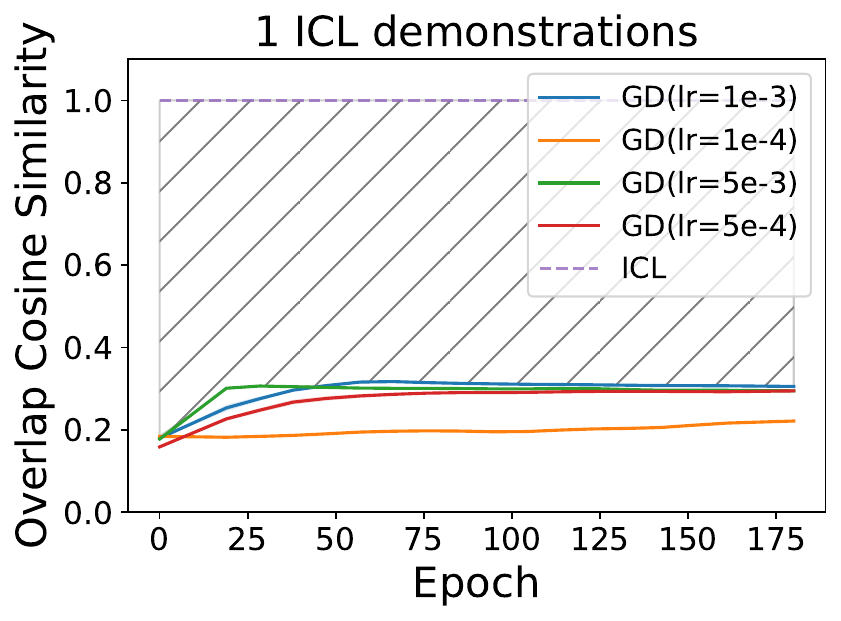}
    \includegraphics[width=0.24\textwidth]{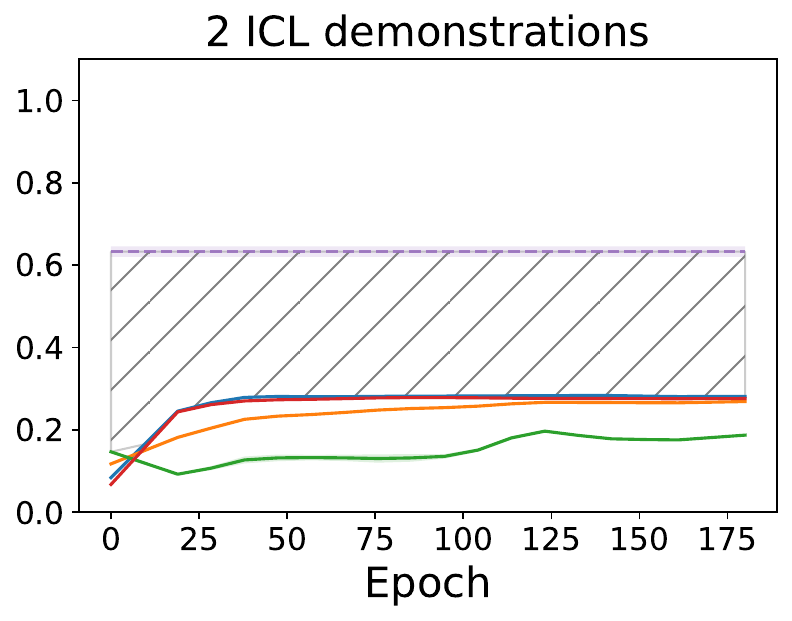}
    \includegraphics[width=0.24\textwidth]{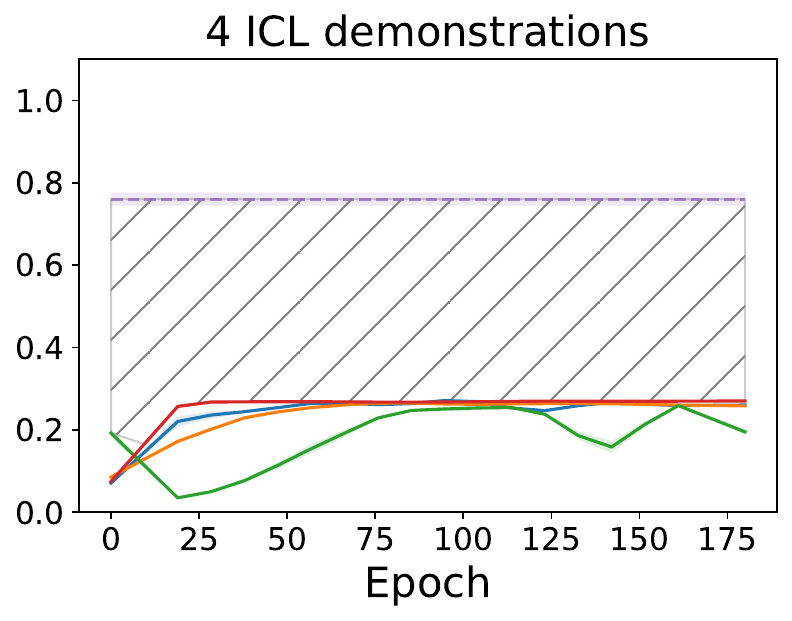}
    \includegraphics[width=0.24\textwidth]{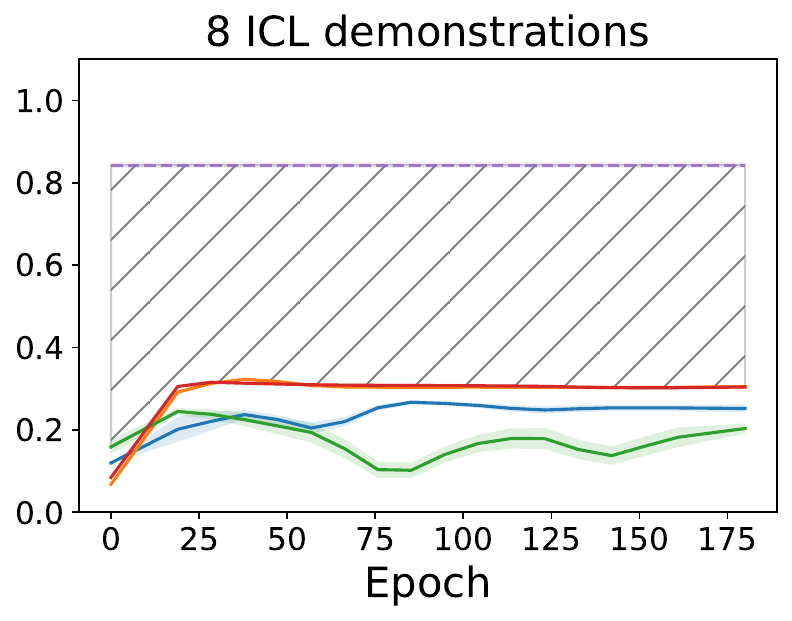}
    }\vspace{-3.5mm}
\caption{Comparison of ICL and GD for the AGNews dataset, with increasing number of demonstrations.}  
\label{agnews}
\end{figure}
\vspace{-2mm}

\begin{figure}[H]
\centering
    \subfigure[\textit{Accuracy} comparison]{
    \includegraphics[width=0.25\textwidth]{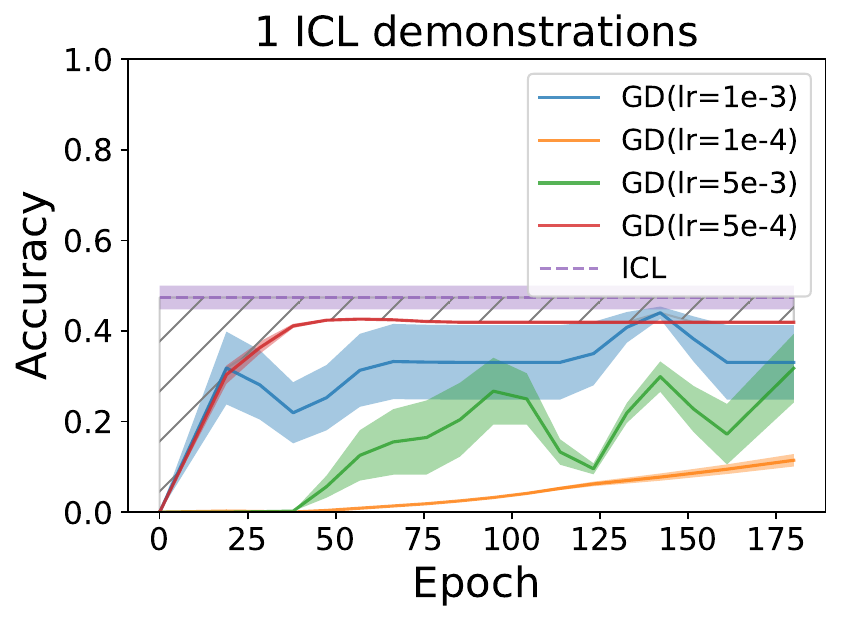}
    \includegraphics[width=0.24\textwidth]{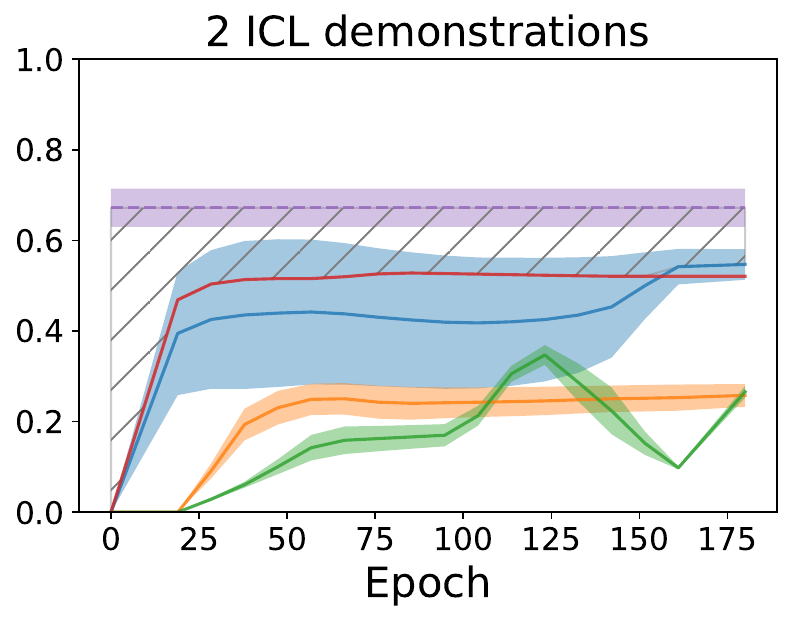}
    \includegraphics[width=0.24\textwidth]{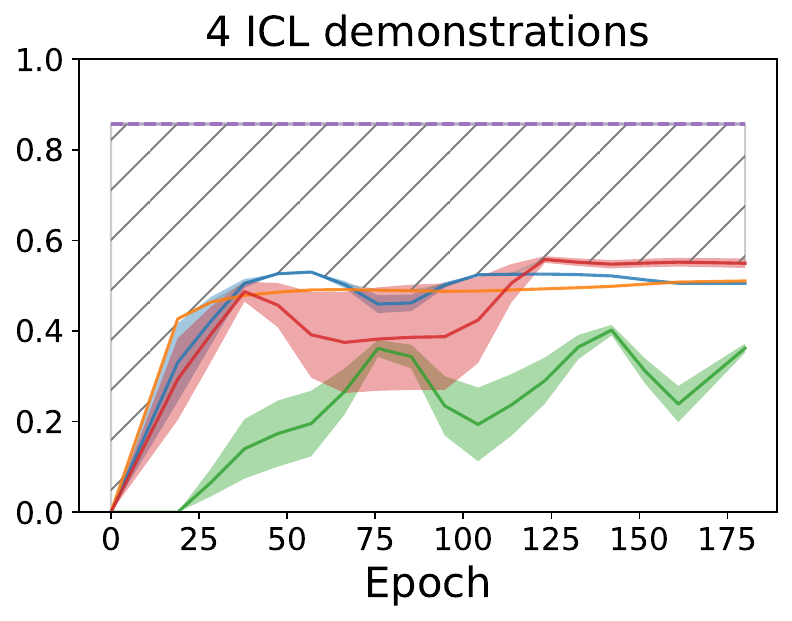}
    \includegraphics[width=0.24\textwidth]{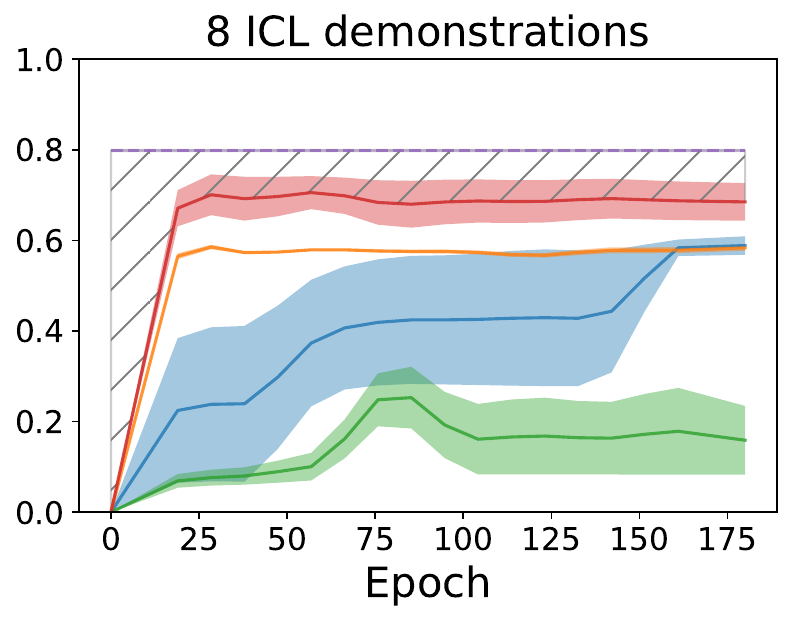}
    }
    \subfigure[\textit{Token overlap} comparison]{
    \includegraphics[width=0.25\textwidth]{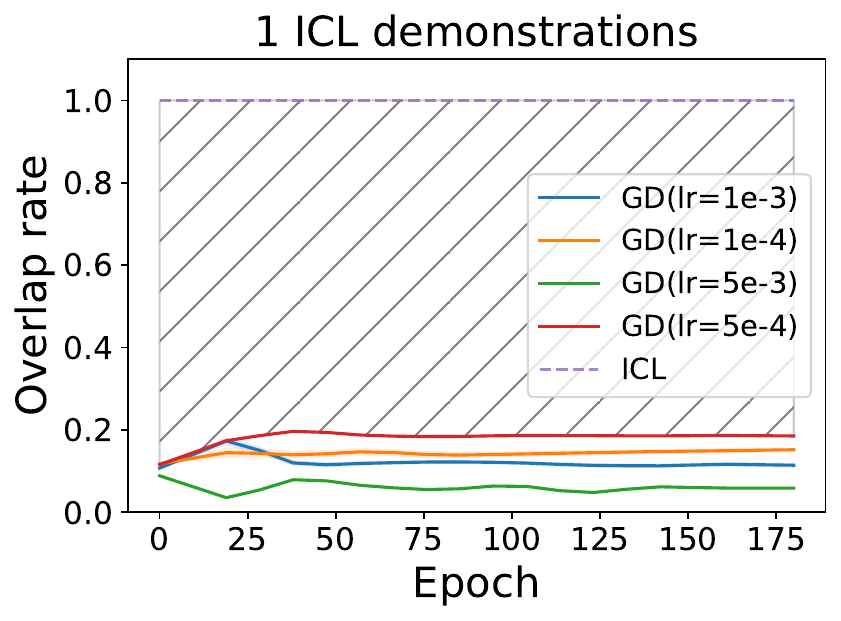}
    \includegraphics[width=0.24\textwidth]{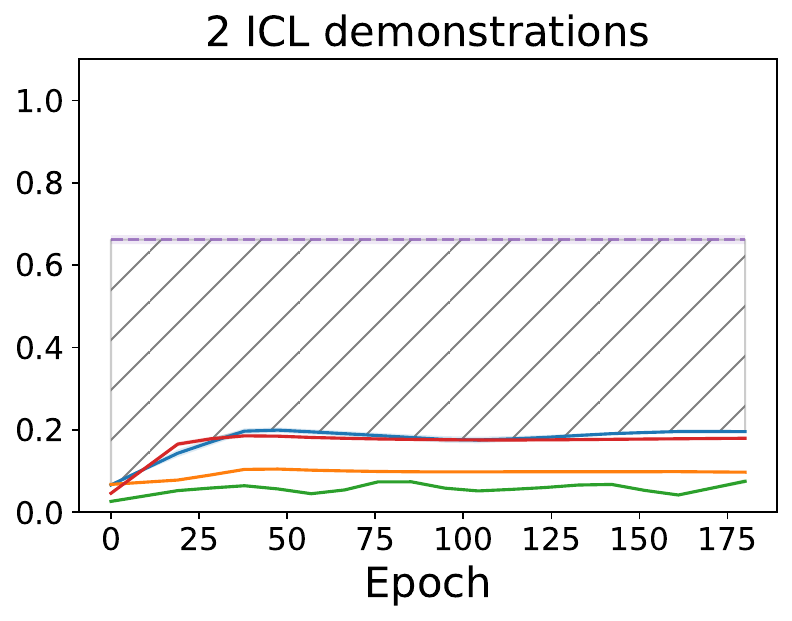}
    \includegraphics[width=0.24\textwidth]{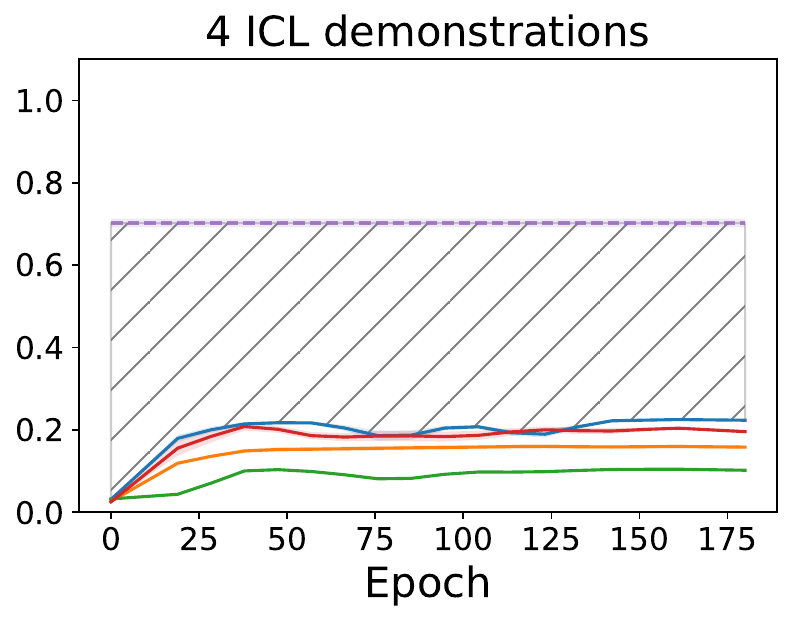}
    \includegraphics[width=0.24\textwidth]{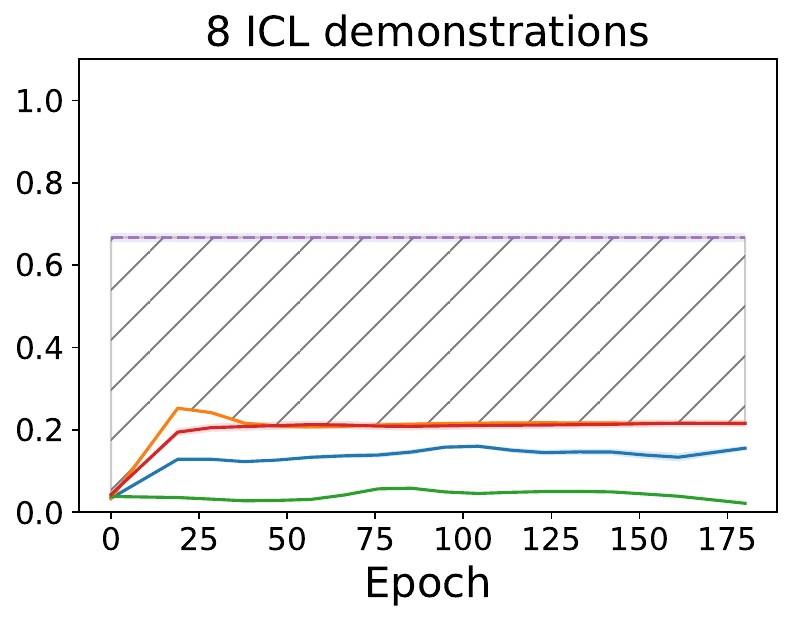}
    } 
    \subfigure[\textit{Overlap Cosine Similarity} comparison]{
    \includegraphics[width=0.25\textwidth]{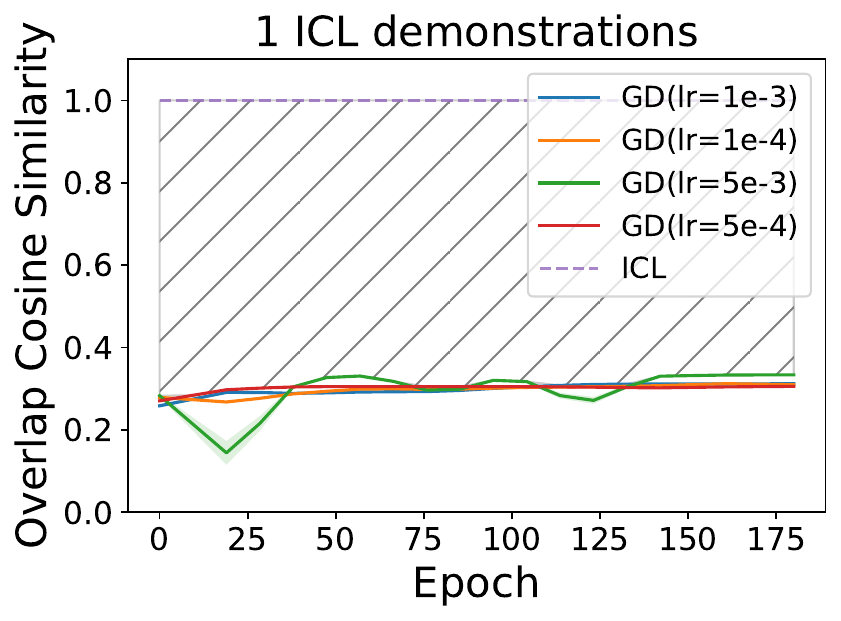}
    \includegraphics[width=0.24\textwidth]{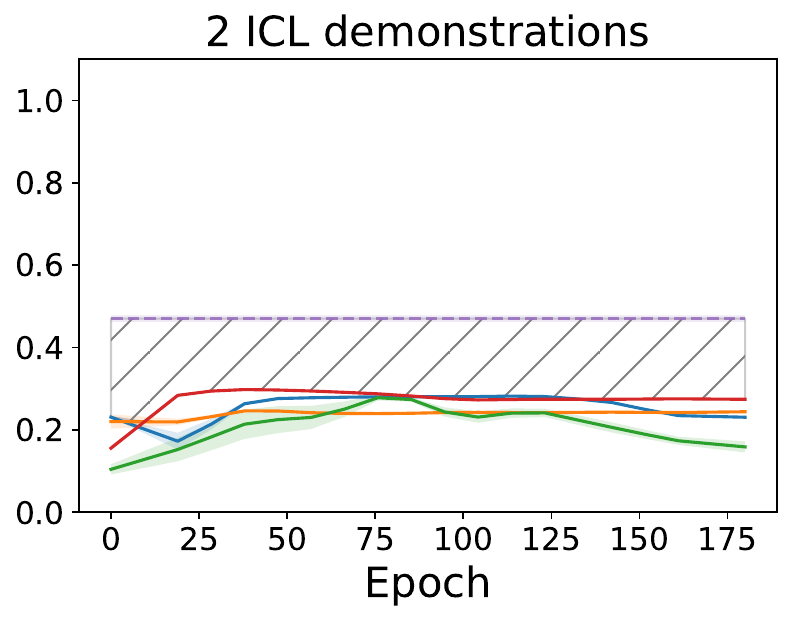}
    \includegraphics[width=0.24\textwidth]{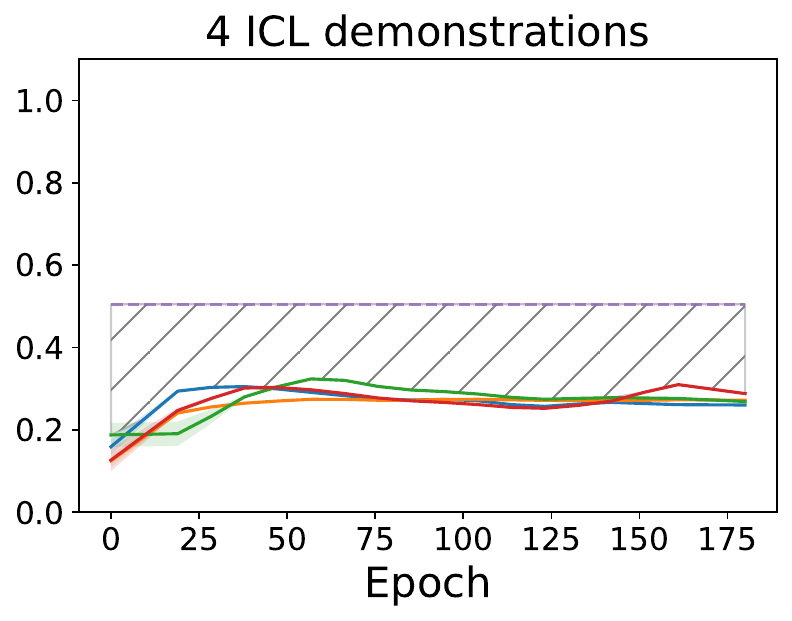}
    \includegraphics[width=0.24\textwidth]{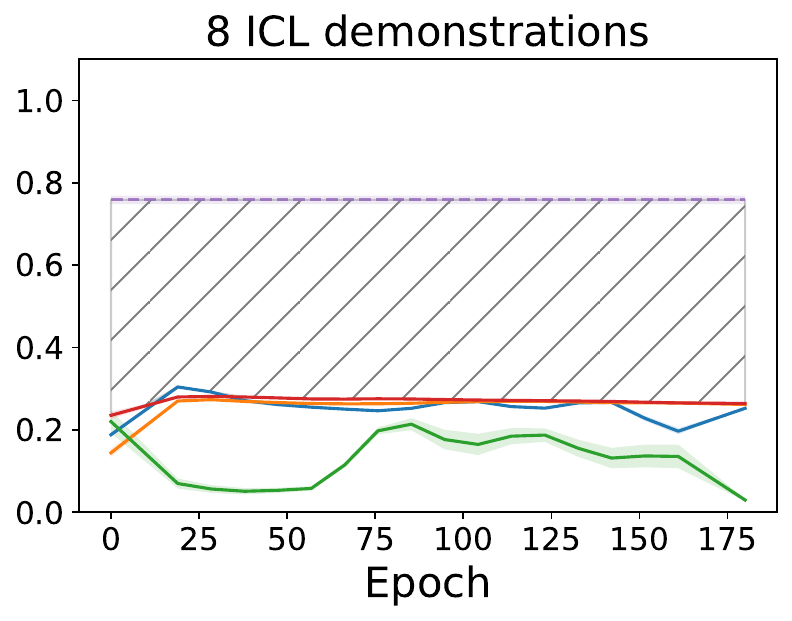}
    }\vspace{-3.5mm}
\caption{Comparison of ICL and GD for the SST dataset, with increasing number of demonstrations.}  
\label{sst2}
\end{figure}
\vspace{-2mm}

\begin{figure}[H]
\centering
    \subfigure[\textit{Accuracy} comparison]{
    \includegraphics[width=0.25\textwidth]{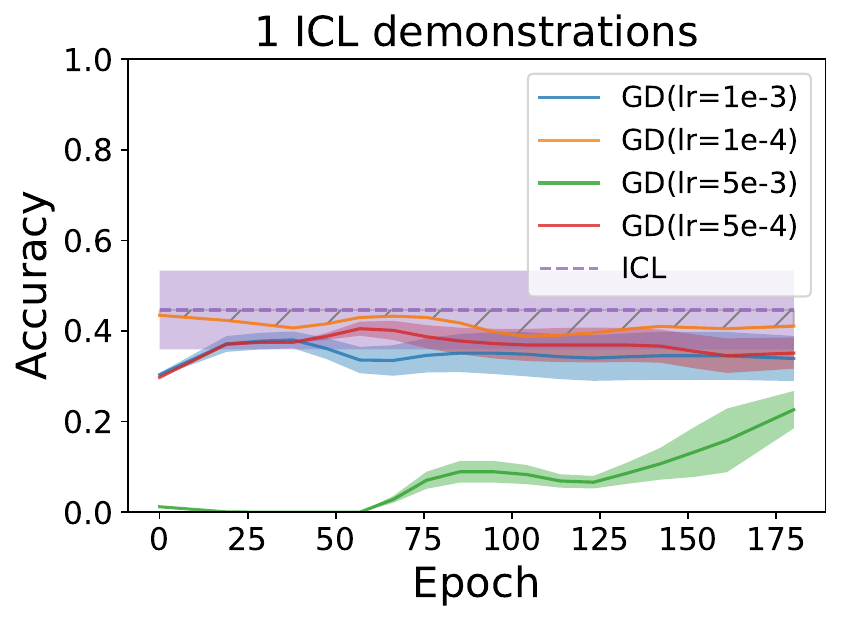}
    \includegraphics[width=0.24\textwidth]{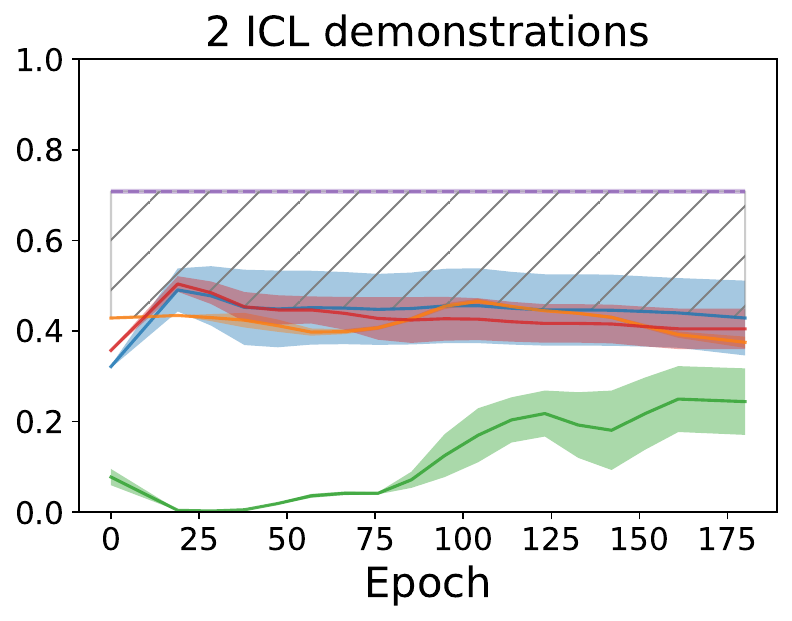}
    \includegraphics[width=0.24\textwidth]{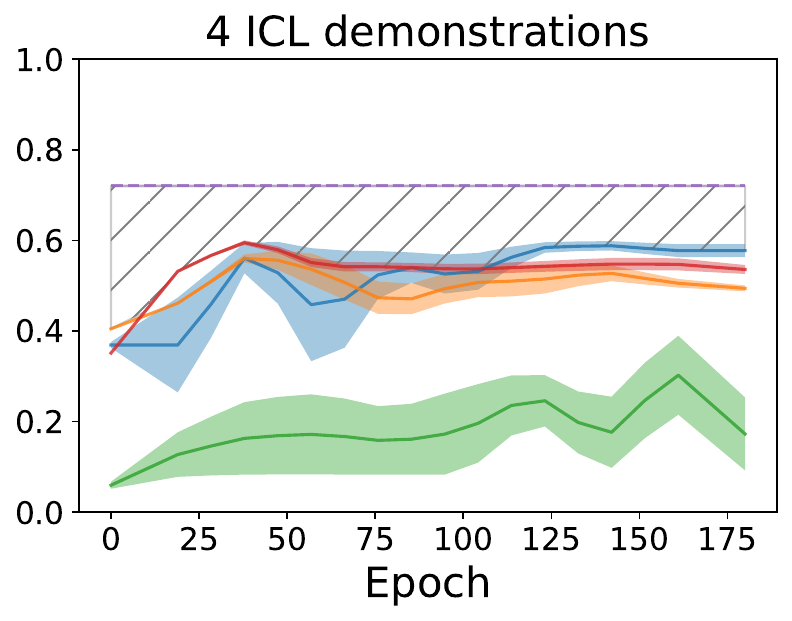}
    \includegraphics[width=0.24\textwidth]{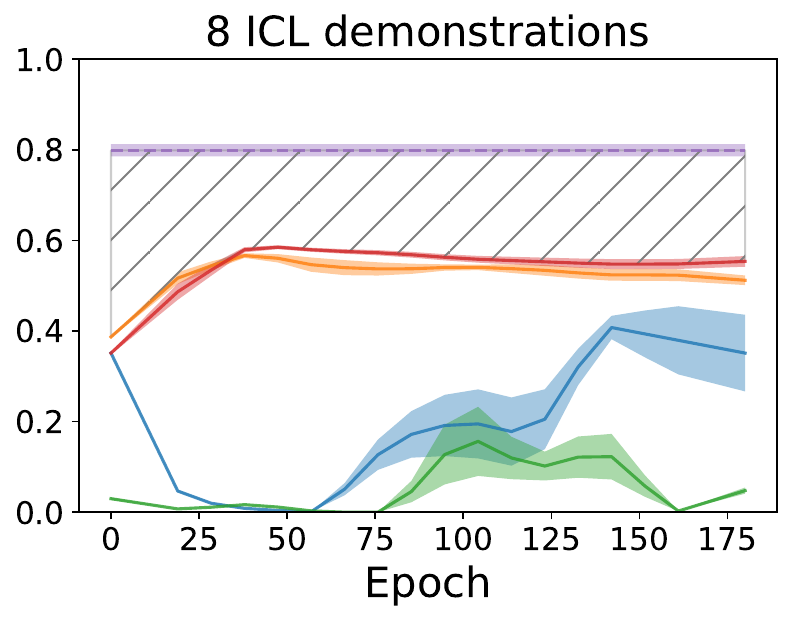}
    }
    \subfigure[\textit{Token overlap} comparison]{
    \includegraphics[width=0.25\textwidth]{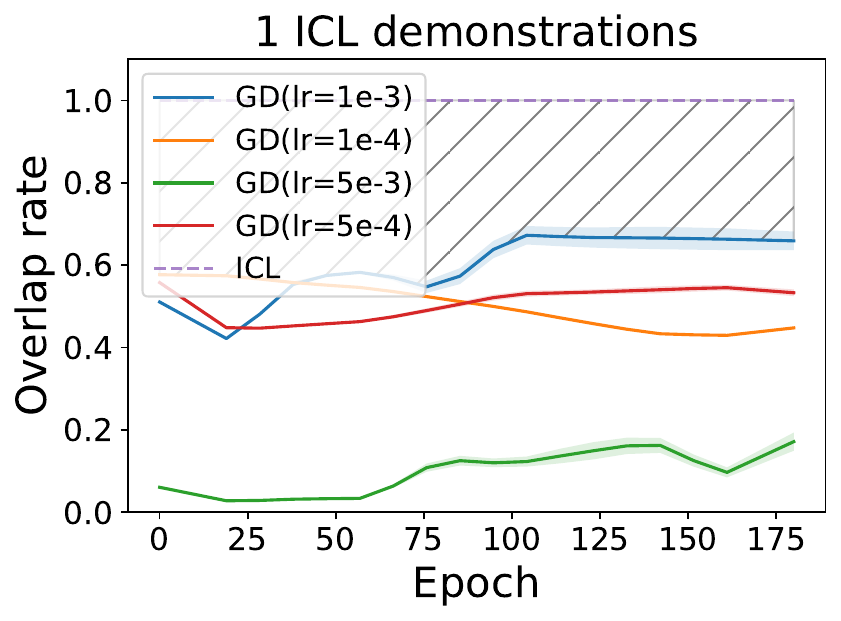}
    \includegraphics[width=0.24\textwidth]{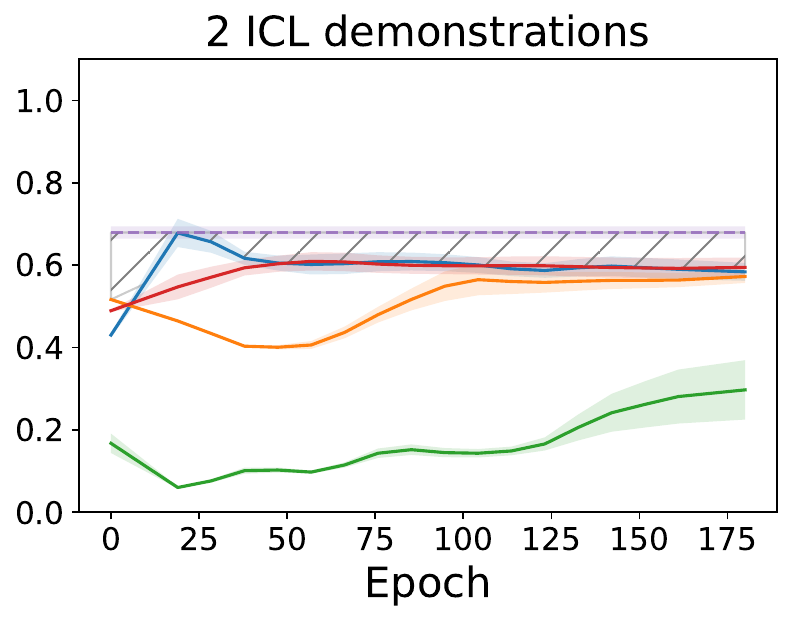}
    \includegraphics[width=0.24\textwidth]{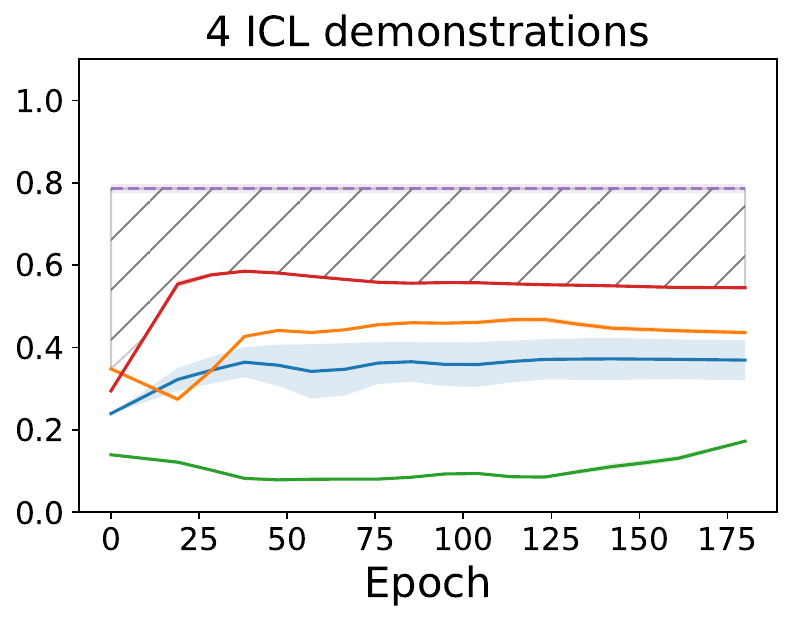}
    \includegraphics[width=0.24\textwidth]{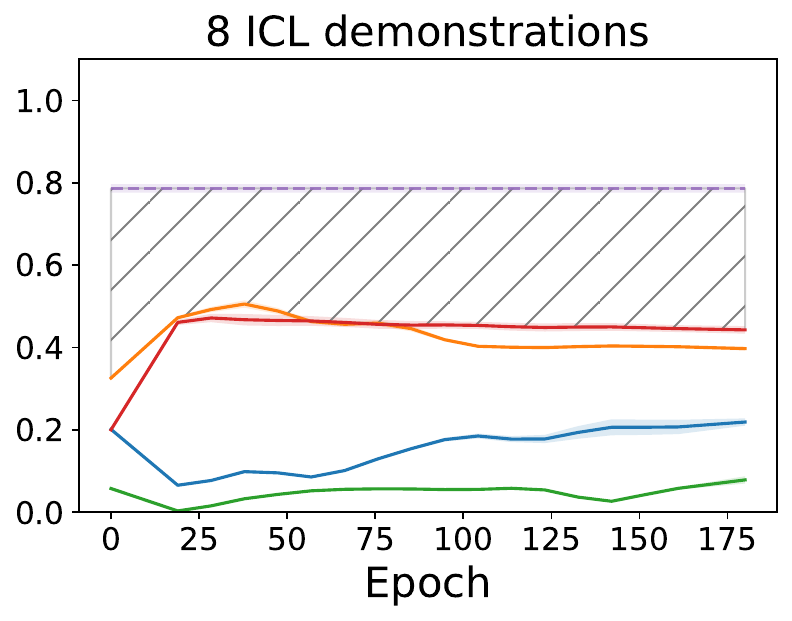}
    } 
    \subfigure[\textit{Overlap Cosine Similarity} comparison]{
    \includegraphics[width=0.25\textwidth]{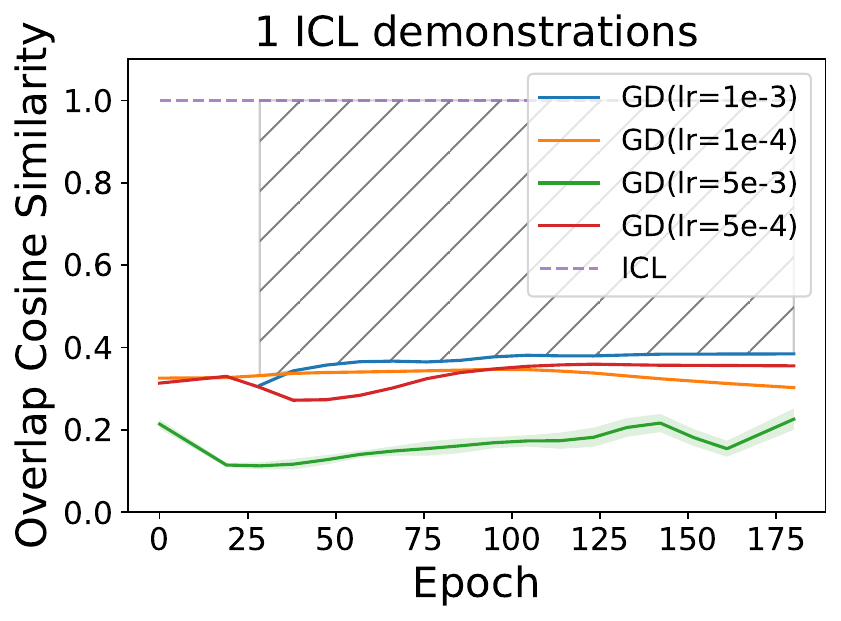}
    \includegraphics[width=0.24\textwidth]{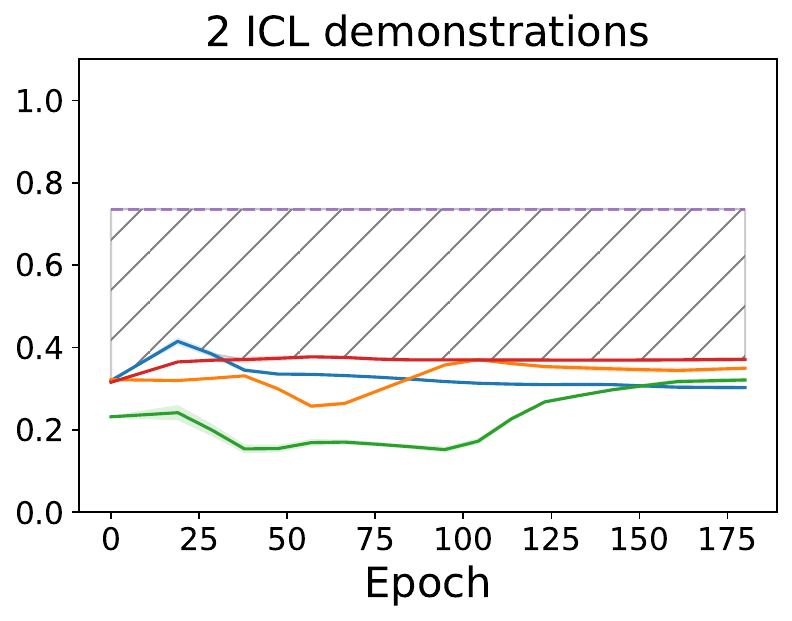}
    \includegraphics[width=0.24\textwidth]{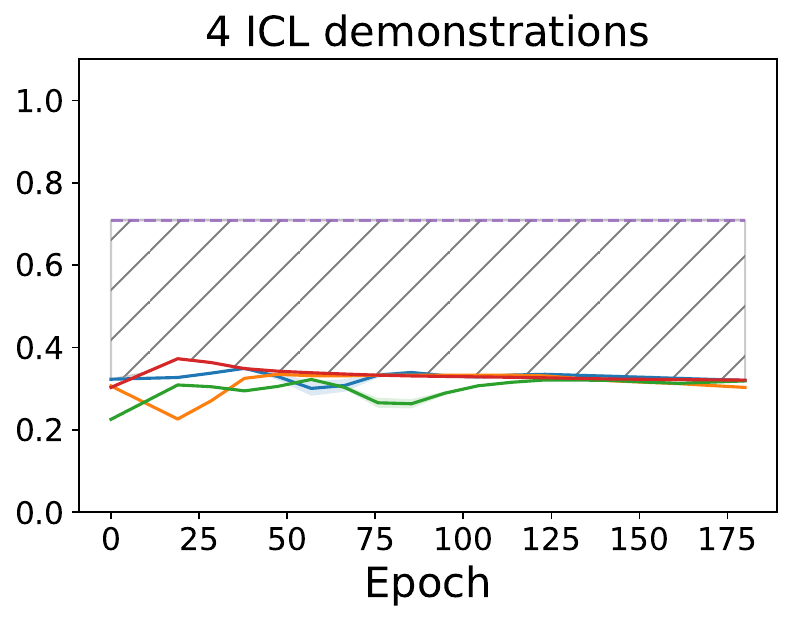}
    \includegraphics[width=0.24\textwidth]{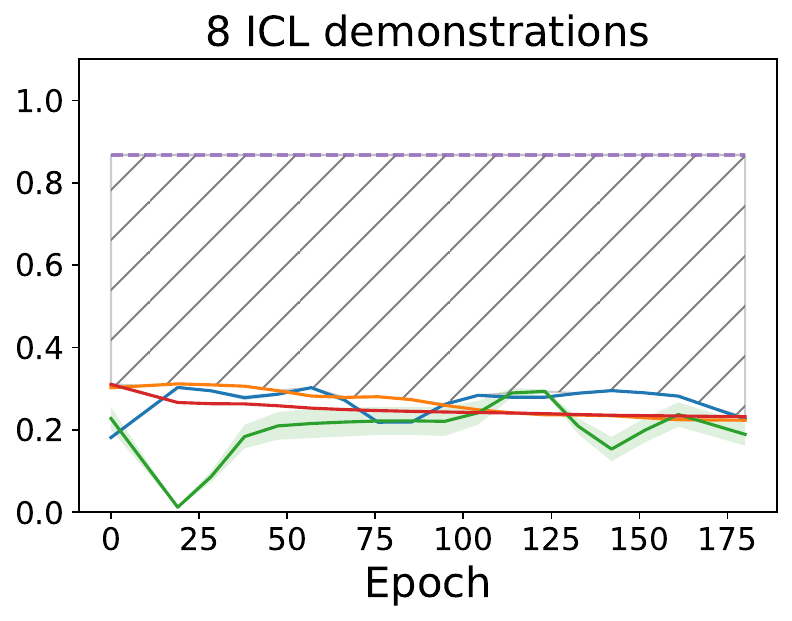}
    }\vspace{-3.5mm}
\caption{Comparison of ICL and GD for the CB dataset, with increasing number of demonstrations.}  
\label{cb}
\end{figure}
\vspace{-2mm}

\begin{figure}[H]
\centering
    \subfigure[\textit{Accuracy} comparison]{
    \includegraphics[width=0.25\textwidth]{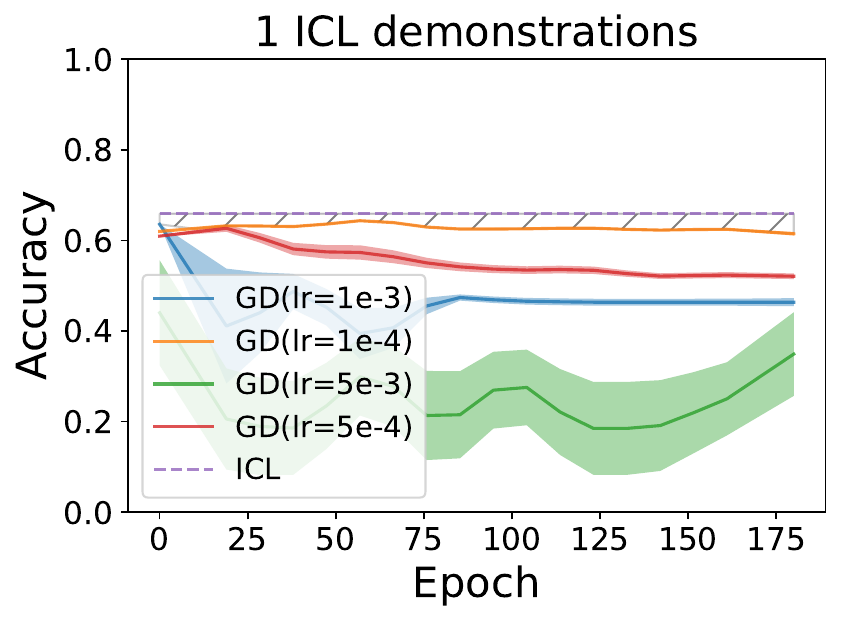}
    \includegraphics[width=0.24\textwidth]{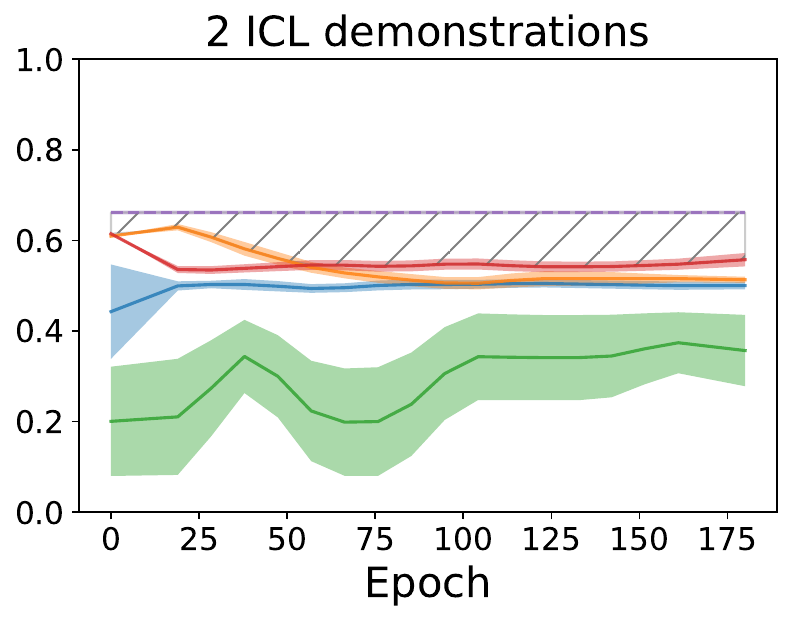}
    \includegraphics[width=0.24\textwidth]{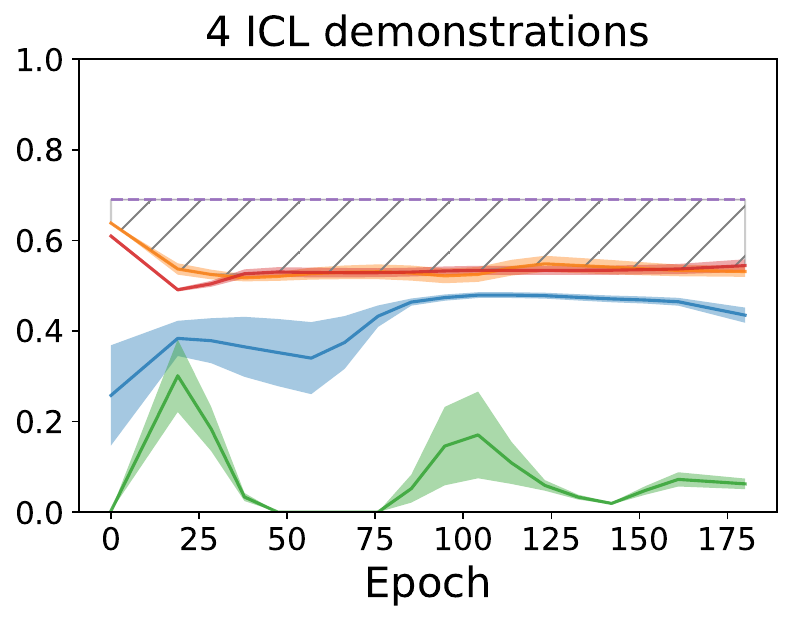}
    \includegraphics[width=0.24\textwidth]{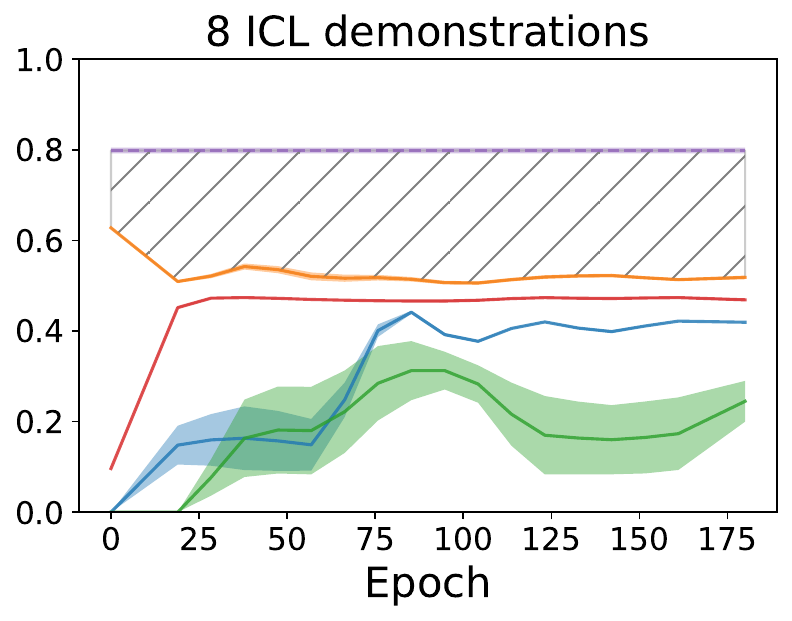}
    }
    \subfigure[\textit{Token overlap} comparison]{
    \includegraphics[width=0.25\textwidth]{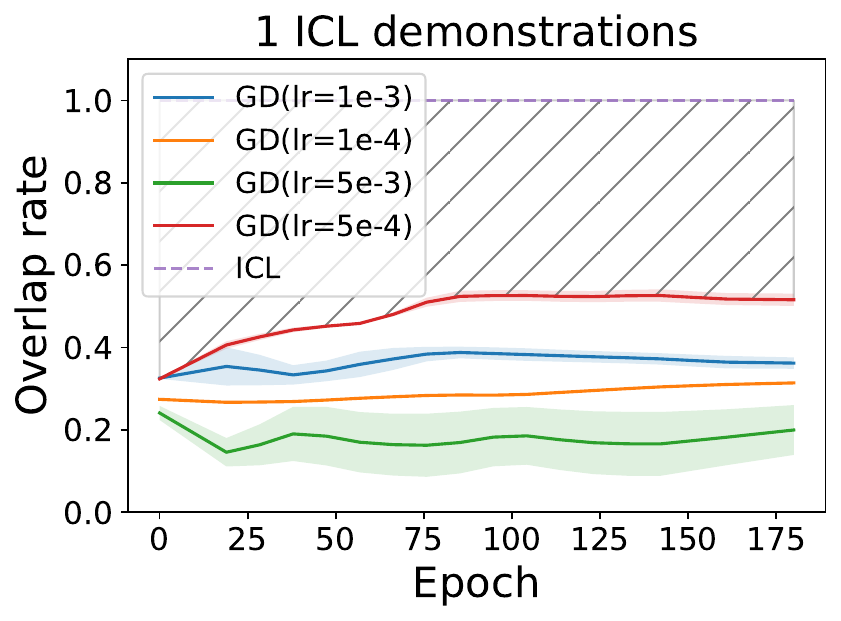}
    \includegraphics[width=0.24\textwidth]{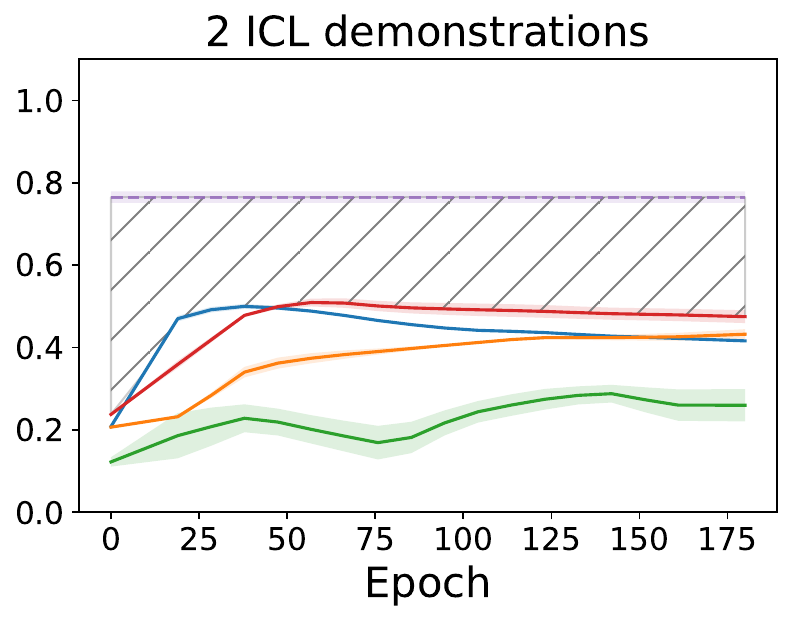}
    \includegraphics[width=0.24\textwidth]{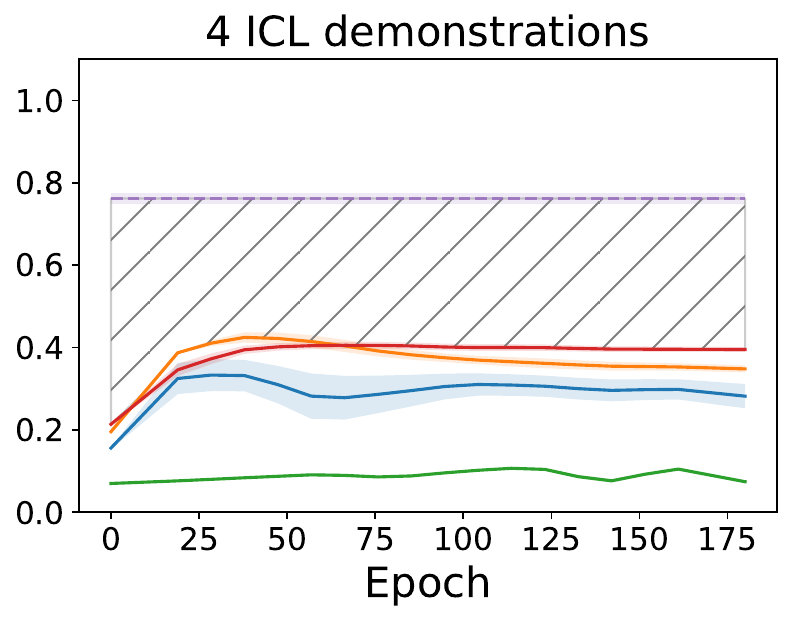}
    \includegraphics[width=0.24\textwidth]{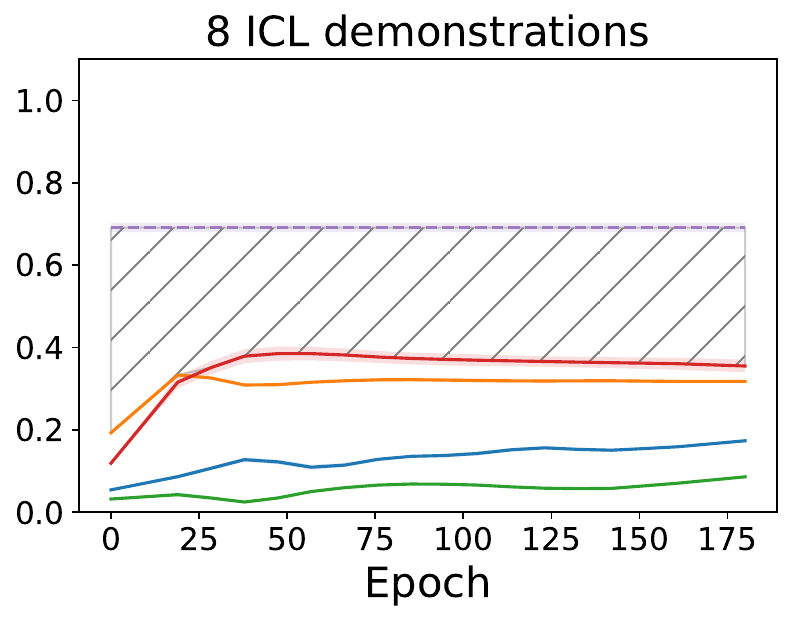}
    } 
    \subfigure[\textit{Overlap Cosine Similarity} comparison]{
    \includegraphics[width=0.25\textwidth]{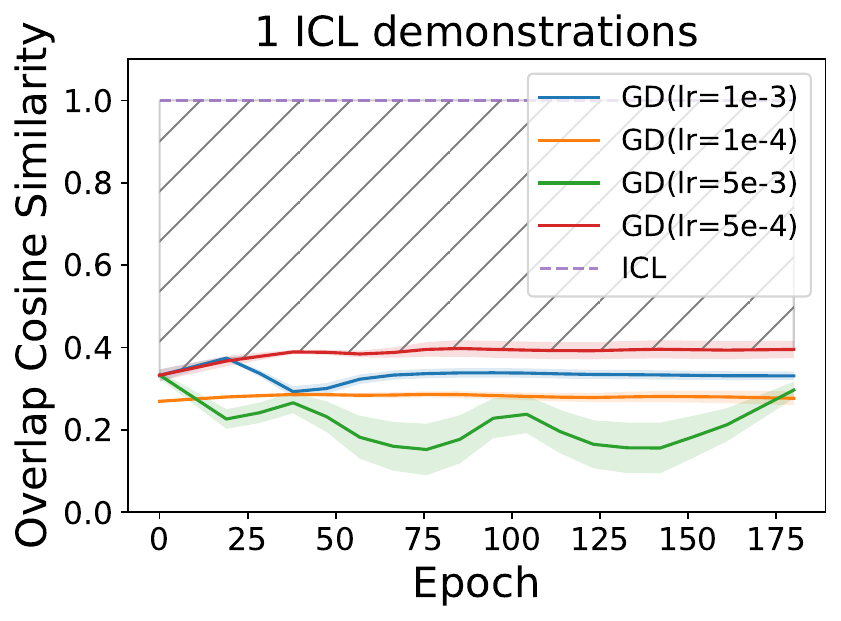}
    \includegraphics[width=0.24\textwidth]{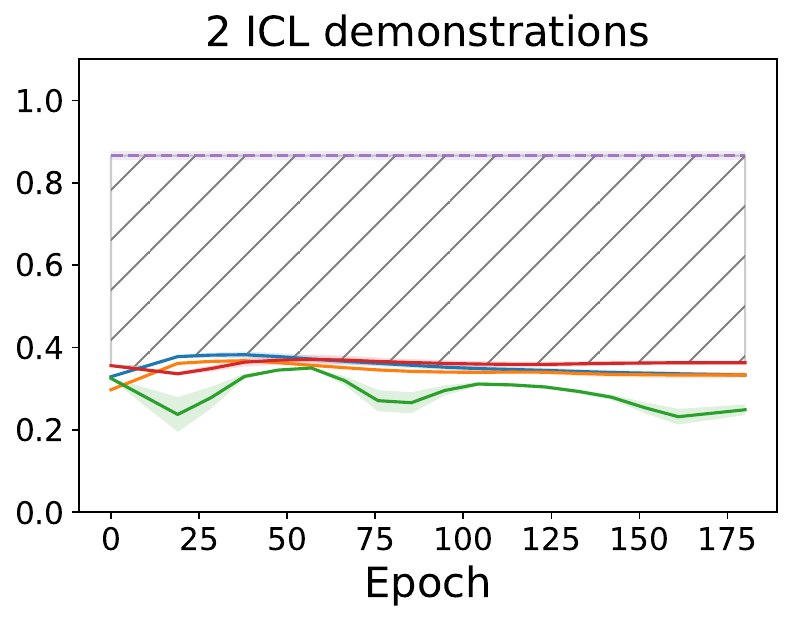}
    \includegraphics[width=0.24\textwidth]{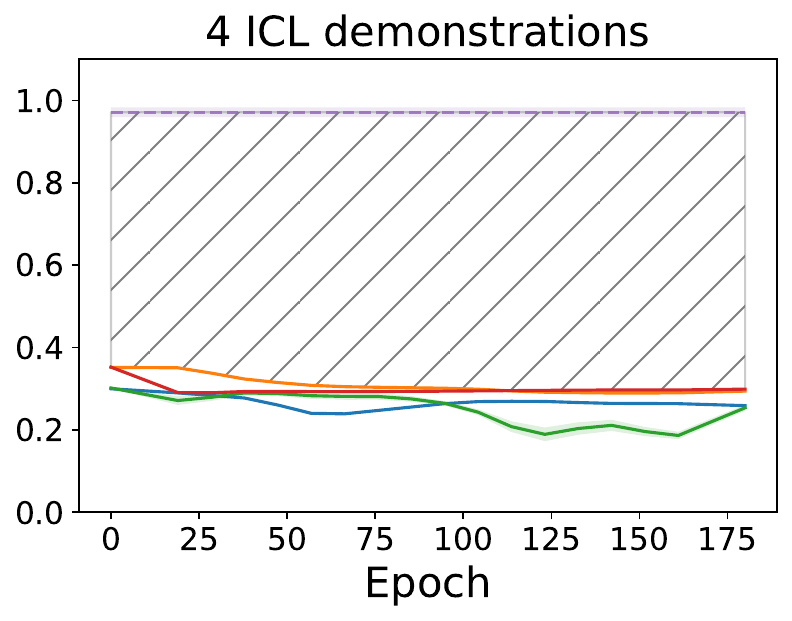}
    \includegraphics[width=0.24\textwidth]{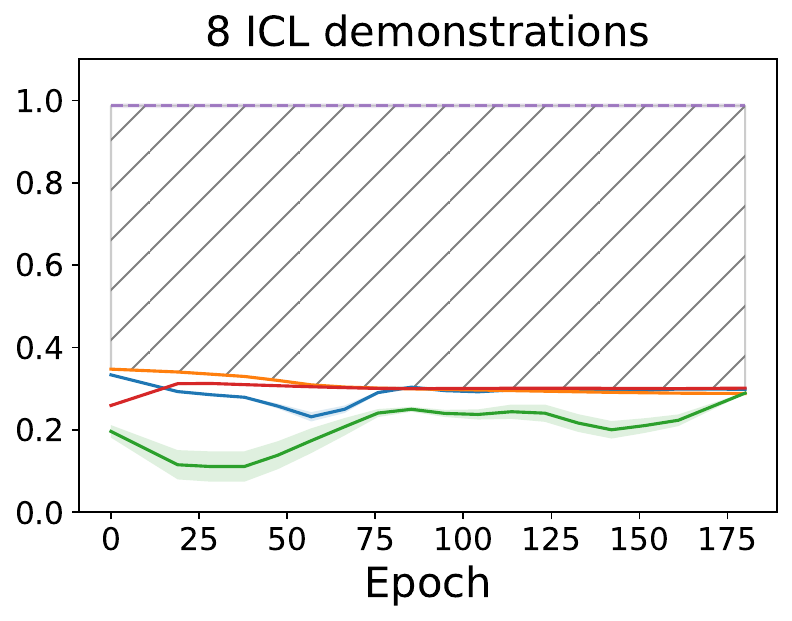}
    }\vspace{-3.5mm}
\caption{Comparison of ICL and GD for the RTE dataset, with increasing number of demonstrations.}  
\label{rte}
\end{figure}

\newpage
\section{Empirical results on ICL vs $\widehat{\text{GD}}$}\label{gd_sub}
Here, we present corresponding results on ICL vs $\widehat{\text{GD}}$.

\paragraph{How are sub-models selected for optimization?}
Since $\widehat{\text{GD}}$ conducts updates only on the subset of the model and enumerating all the possible subsets of model parameters is infeasible, we select intuitive subsets of parameters to simulate $\widehat{\text{GD}}$.

We use the hypotheses in \citep{akyurek2022learning, von2023Transformers}, to experiment with intuitive subsets of models. In particular, according to \citet{von2023Transformers} the \textit{implicit} model lies in $W_V$ of the Transformer while the probing experiments in \cite{akyurek2022learning} suggest that this iterative optimization happens in top layers of the Transformers. Therefore, we provide experiments with two intuitive subsets to simulate $\widehat{\text{GD}}$: finetuning (1) $W_V$ of a single deep layer, and (2) $W_V$ of a single middle layer.

\paragraph{Results of ICL vs. $\widehat{\text{GD}}$ (Deep layer)}
Following a similar experimental setup in \cref{experiments}, we compare the differences between ICL and $\widehat{\text{GD}}$. We randomly select one layer from the last four layers from LLaMa (29-32), repeat the experiments four times and plot the mean and std. The results are shown in \autoref{ag-deep} - \autoref{rte-deep}, and we can observe similar gaps between ICL and $\widehat{\text{GD}}$.

\begin{figure}[H]
\centering
    \subfigure[\textit{Accuracy} comparison]{
    \includegraphics[width=0.25\textwidth]{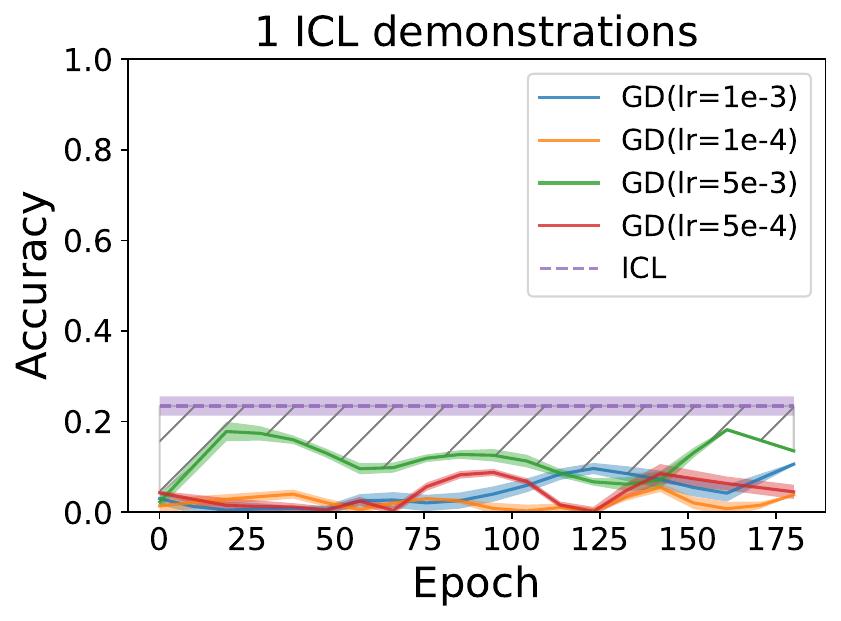}
    \includegraphics[width=0.24\textwidth]{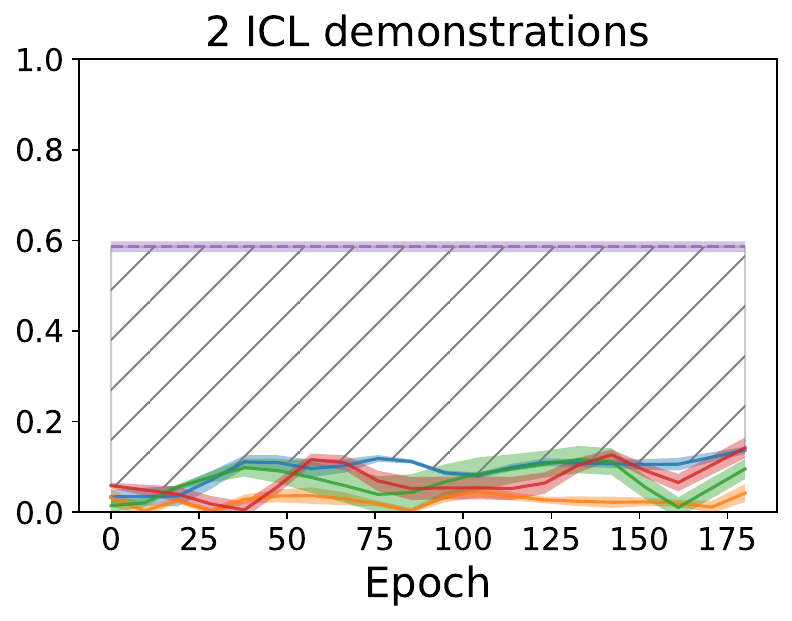}
    \includegraphics[width=0.24\textwidth]{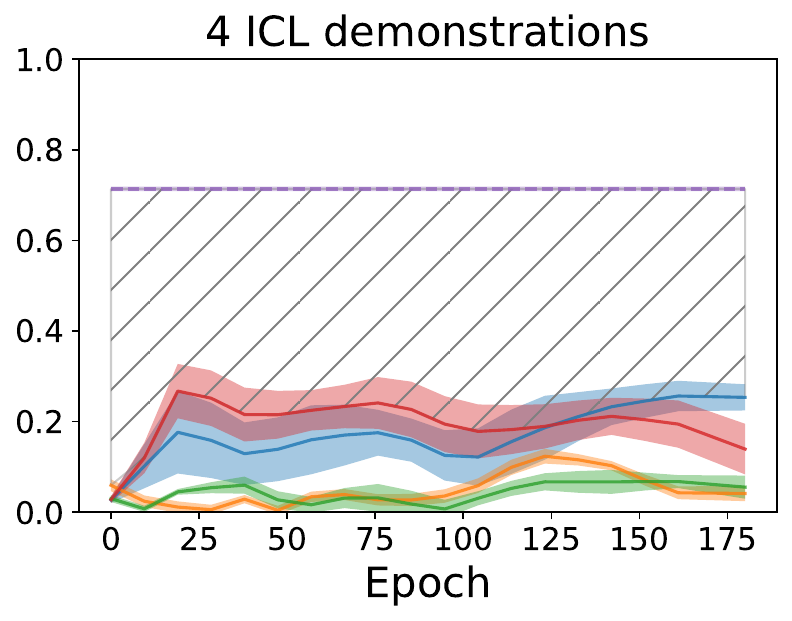}
    \includegraphics[width=0.24\textwidth]{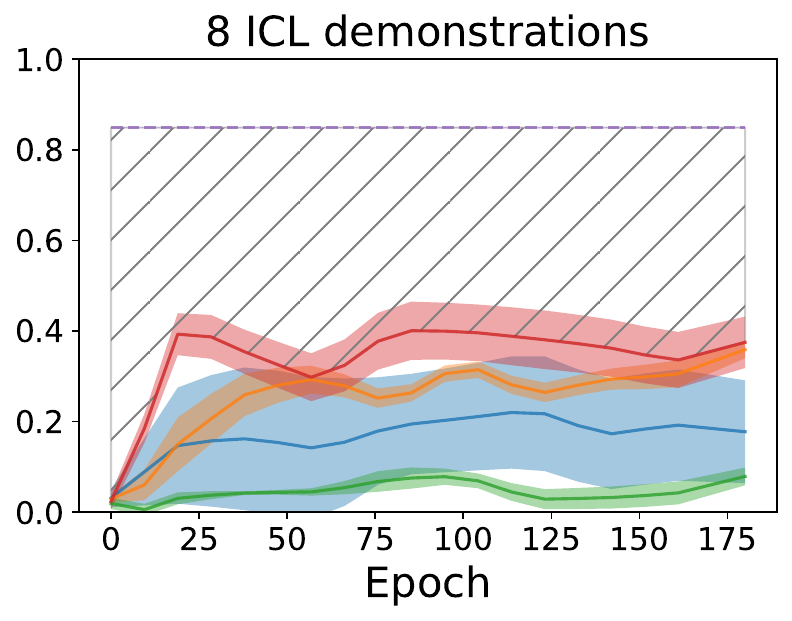}
    }
    \subfigure[\textit{Token overlap} comparison]{
    \includegraphics[width=0.25\textwidth]{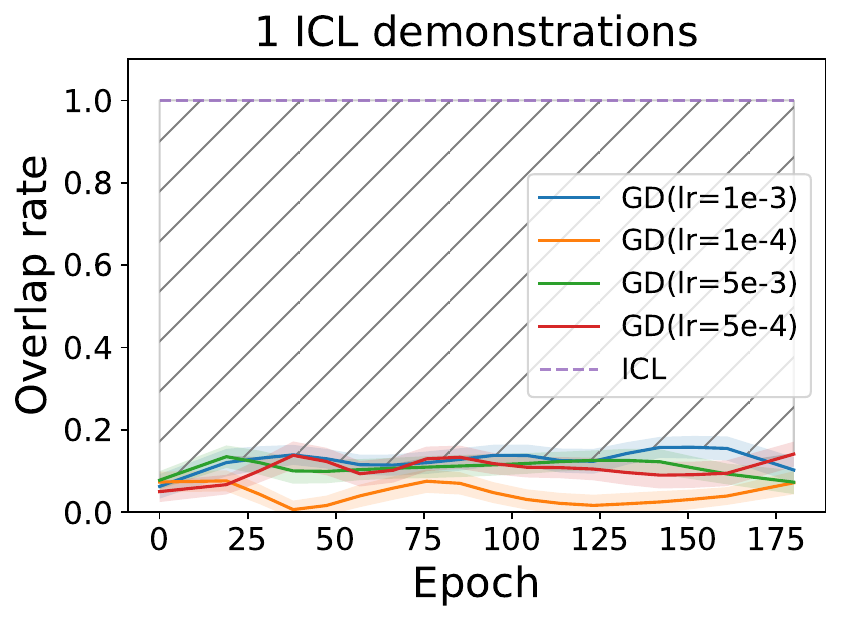}
    \includegraphics[width=0.24\textwidth]{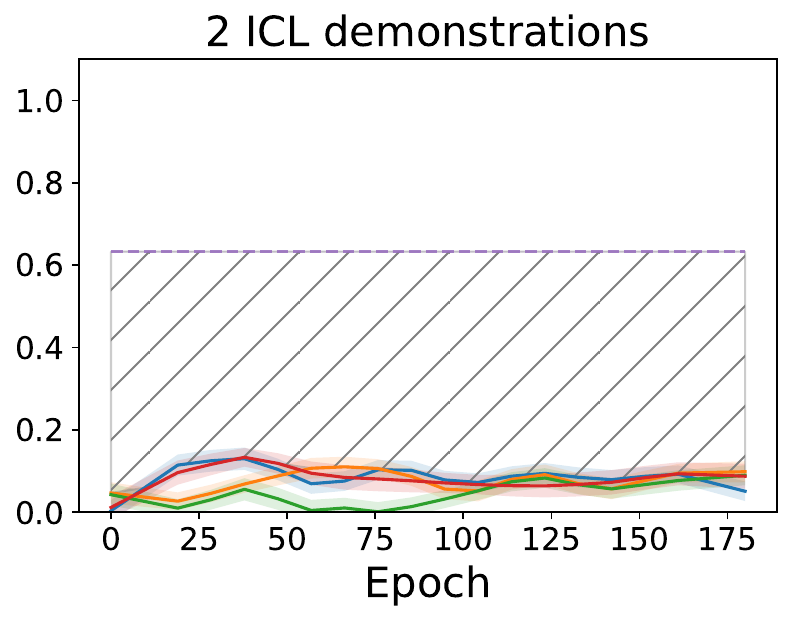}
    \includegraphics[width=0.24\textwidth]{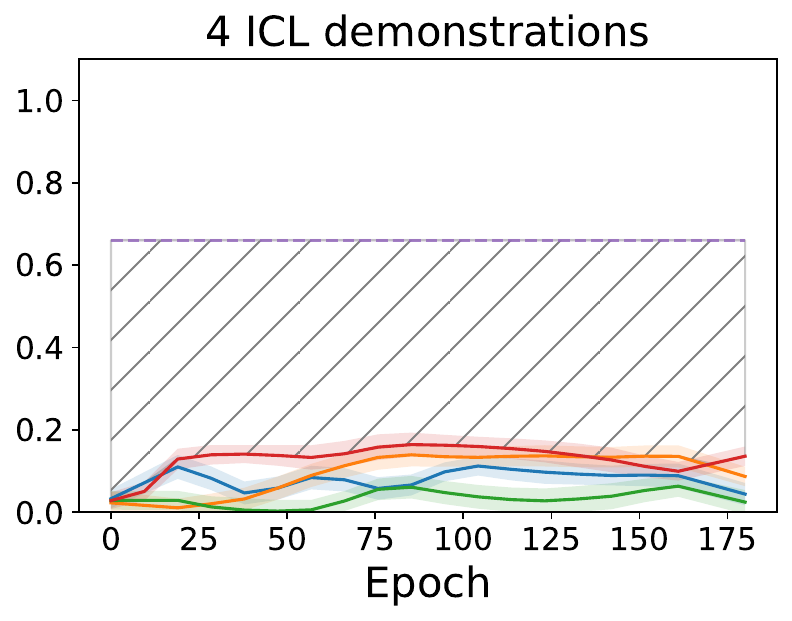}
    \includegraphics[width=0.24\textwidth]{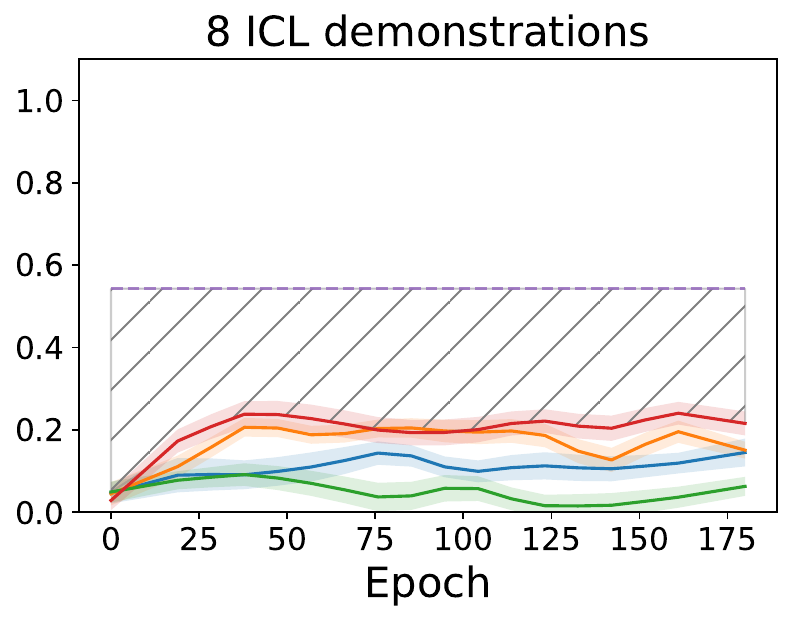}
    } 
    \subfigure[\textit{Overlap Cosine Similarity} comparison]{
    \includegraphics[width=0.25\textwidth]{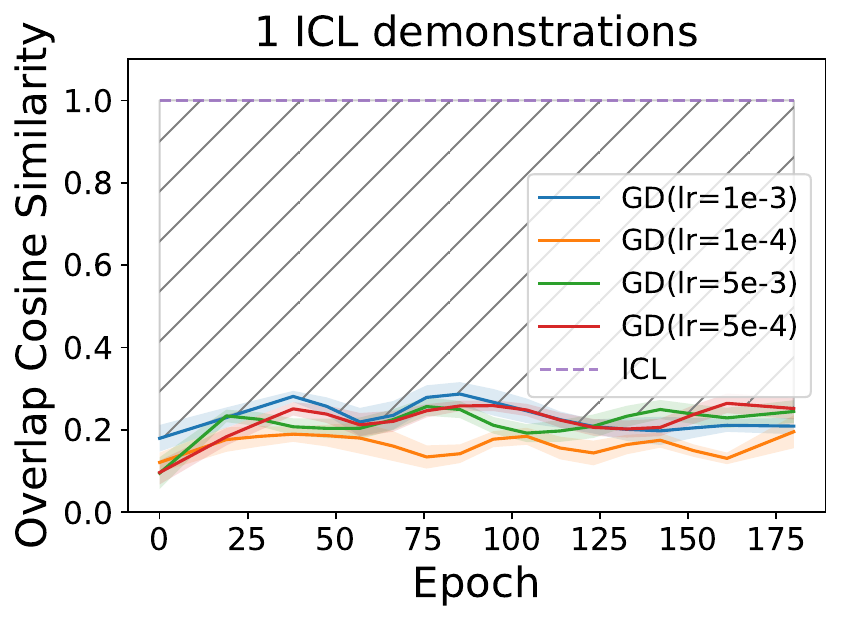}
    \includegraphics[width=0.24\textwidth]{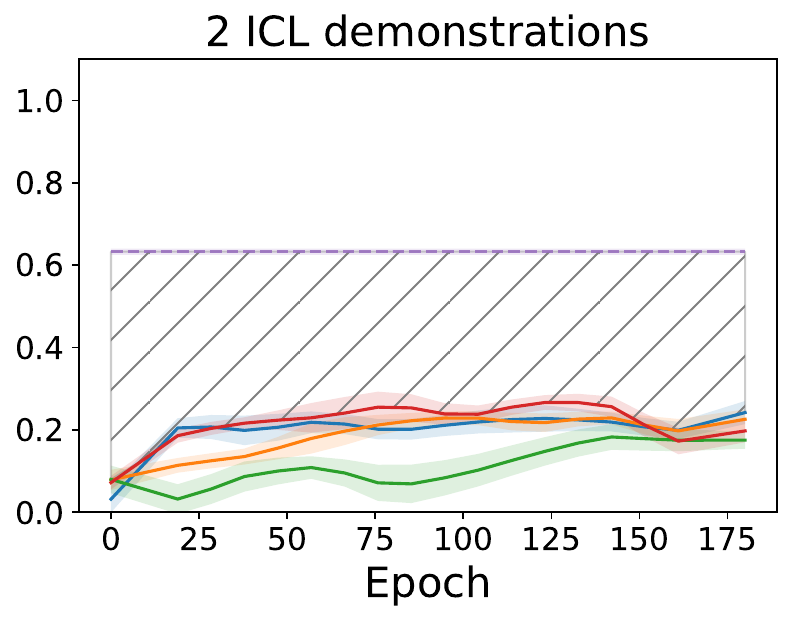}
    \includegraphics[width=0.24\textwidth]{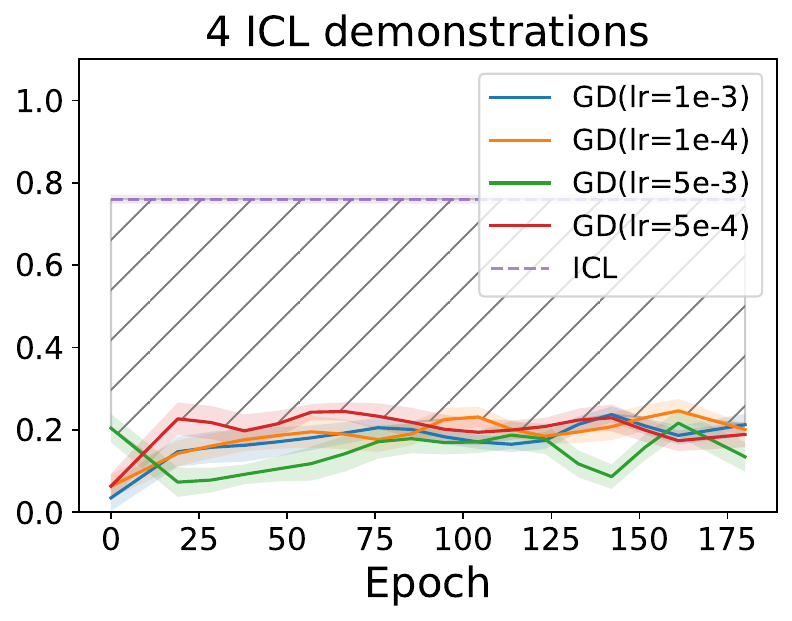}
    \includegraphics[width=0.24\textwidth]{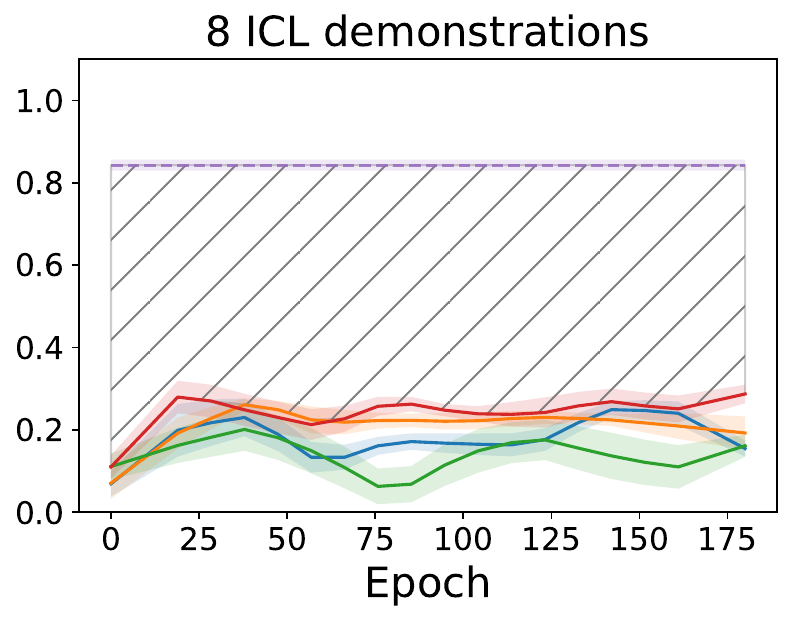}
    }\vspace{-3.5mm}
\caption{
Comparison of ICL and $\widehat{\text{GD}}$ for the AGNews dataset, with increasing number of demonstrations. $\widehat{\text{GD}}$ is simulated by optimizing on one random deep layer of LLaMa.}  
\label{ag-deep}
\end{figure}
\vspace{-2mm}

\begin{figure}[H]
\centering
    \subfigure[\textit{Accuracy} comparison]{
    \includegraphics[width=0.25\textwidth]{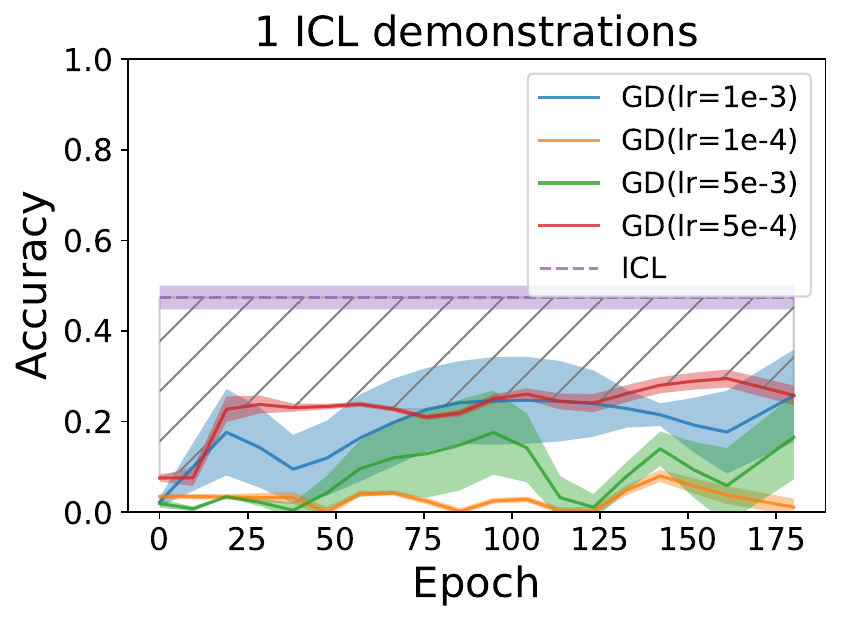}
    \includegraphics[width=0.24\textwidth]{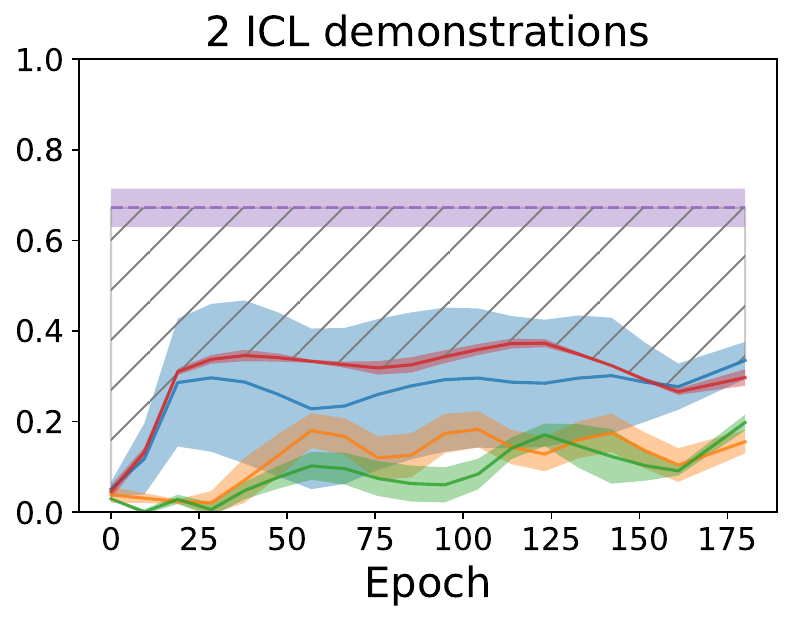}
    \includegraphics[width=0.24\textwidth]{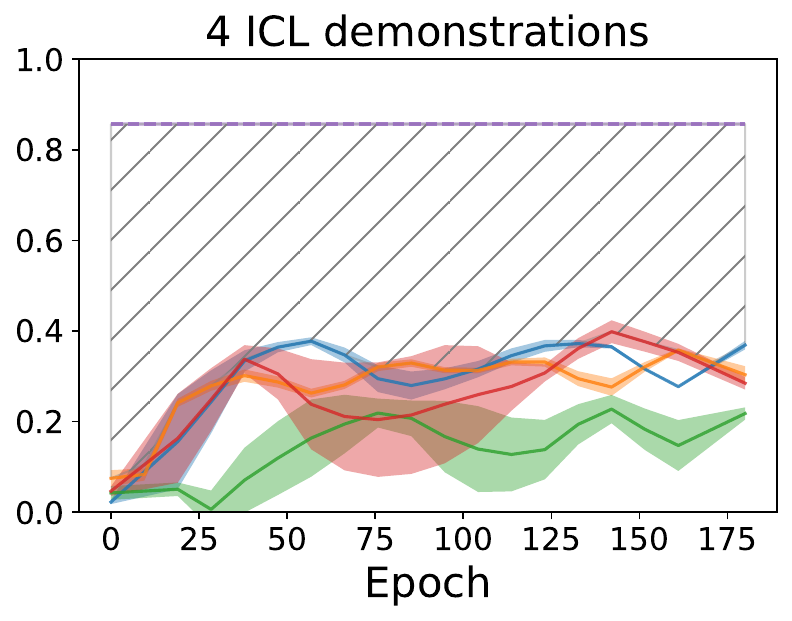}
    \includegraphics[width=0.24\textwidth]{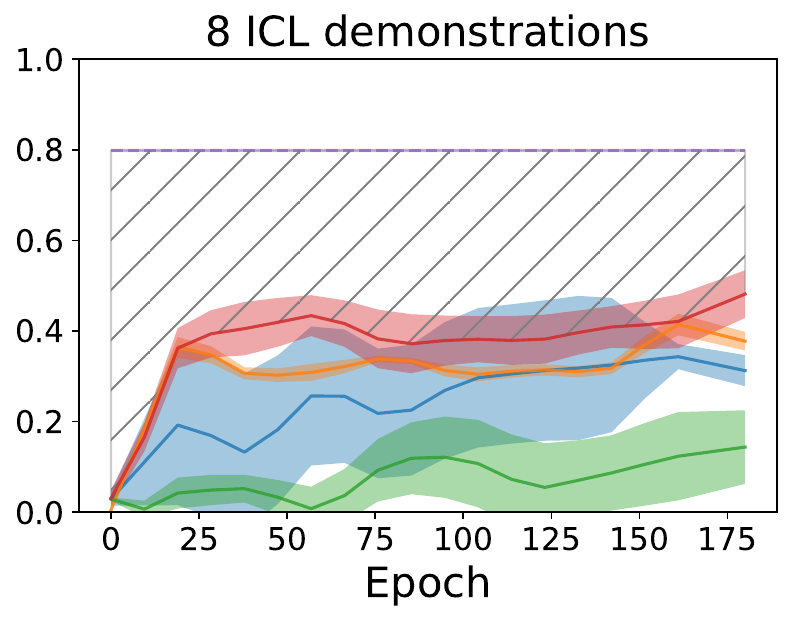}
    }
    \subfigure[\textit{Token overlap} comparison]{
    \includegraphics[width=0.25\textwidth]{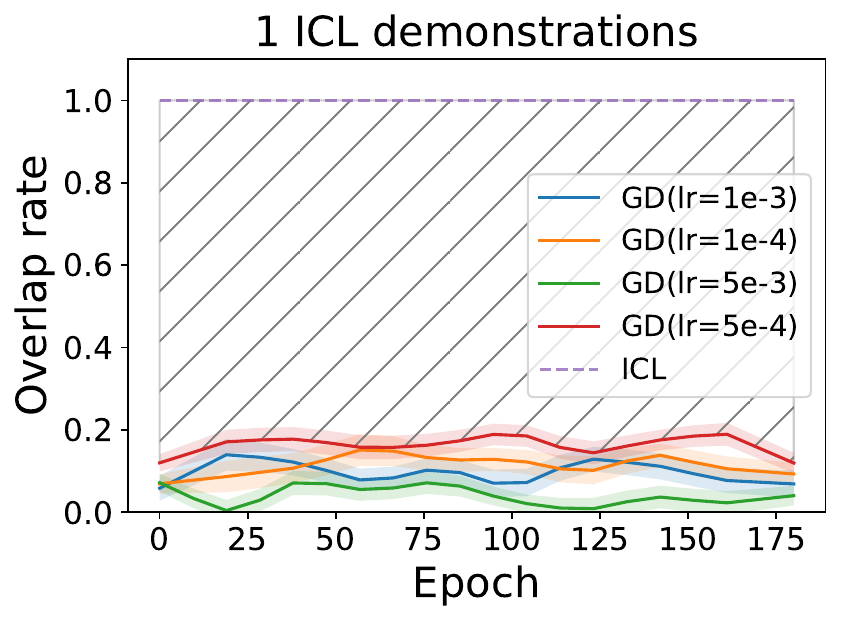}
    \includegraphics[width=0.24\textwidth]{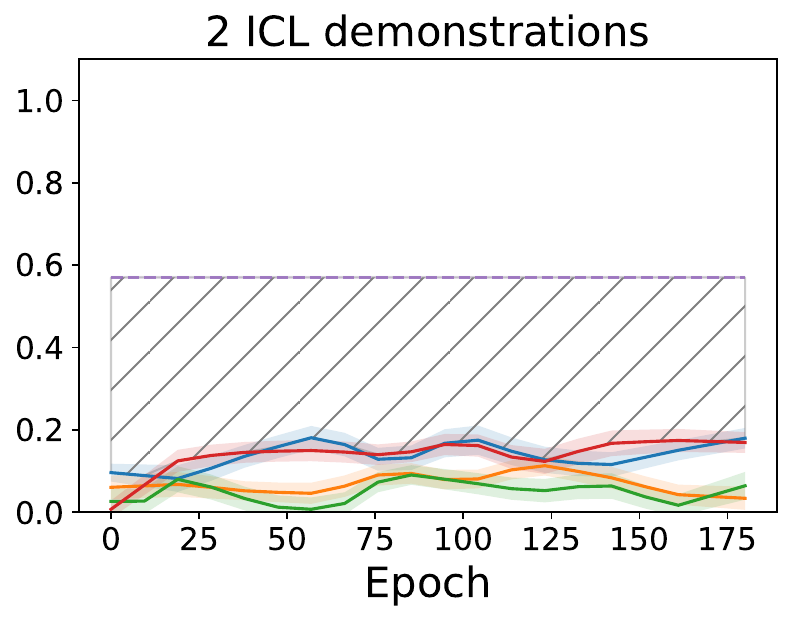}
    \includegraphics[width=0.24\textwidth]{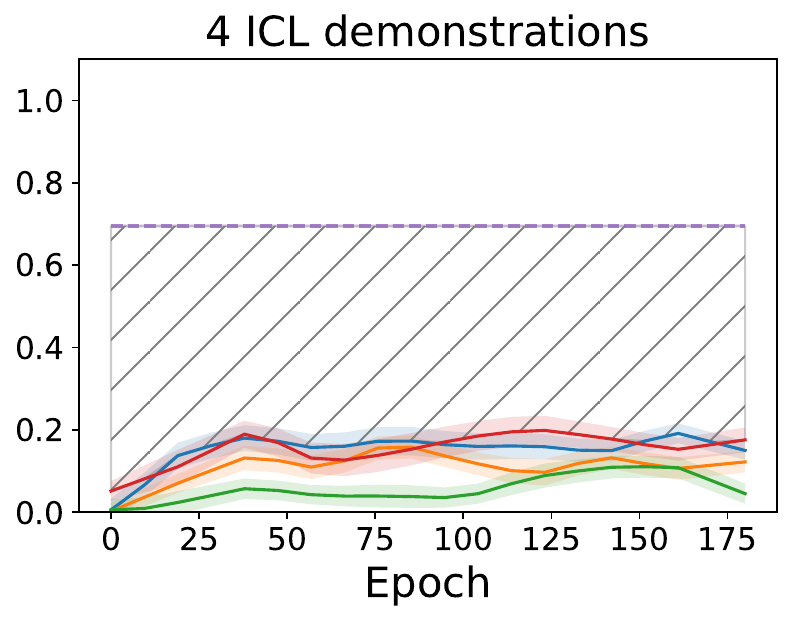}
    \includegraphics[width=0.24\textwidth]{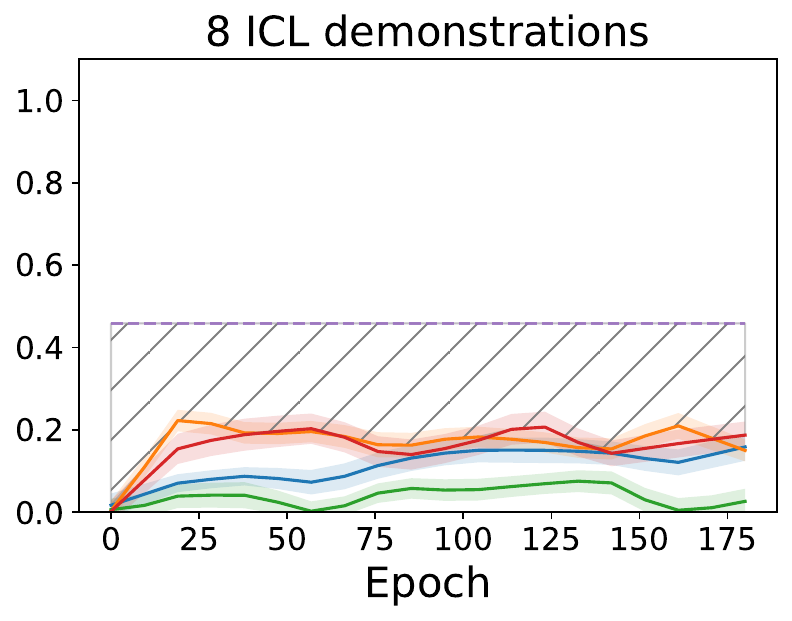}
    } 
    \subfigure[\textit{Overlap Cosine Similarity} comparison]{
    \includegraphics[width=0.25\textwidth]{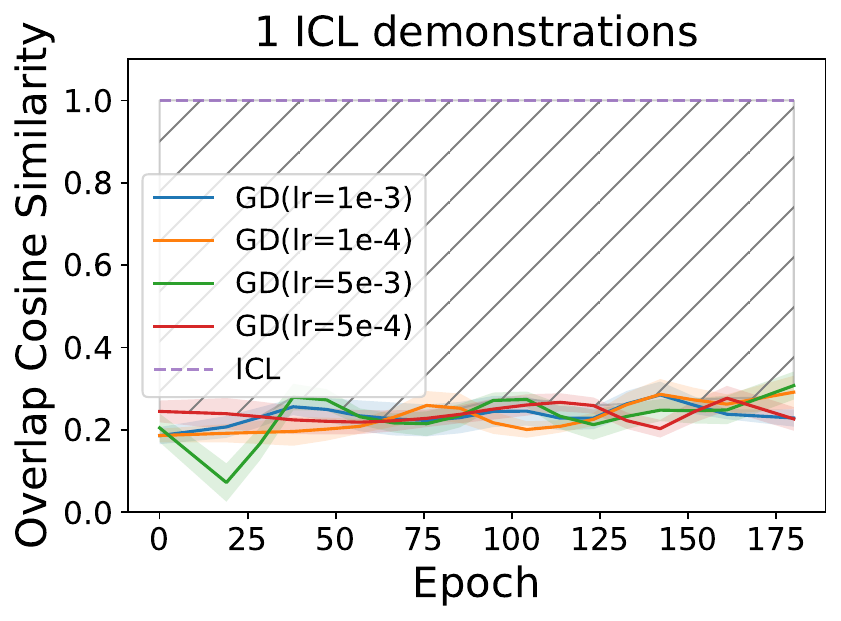}
    \includegraphics[width=0.24\textwidth]{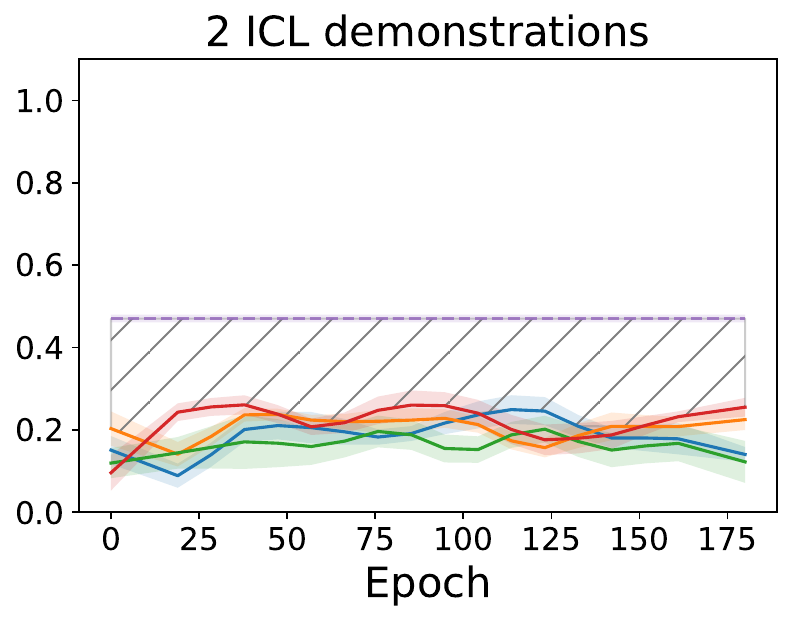}
    \includegraphics[width=0.24\textwidth]{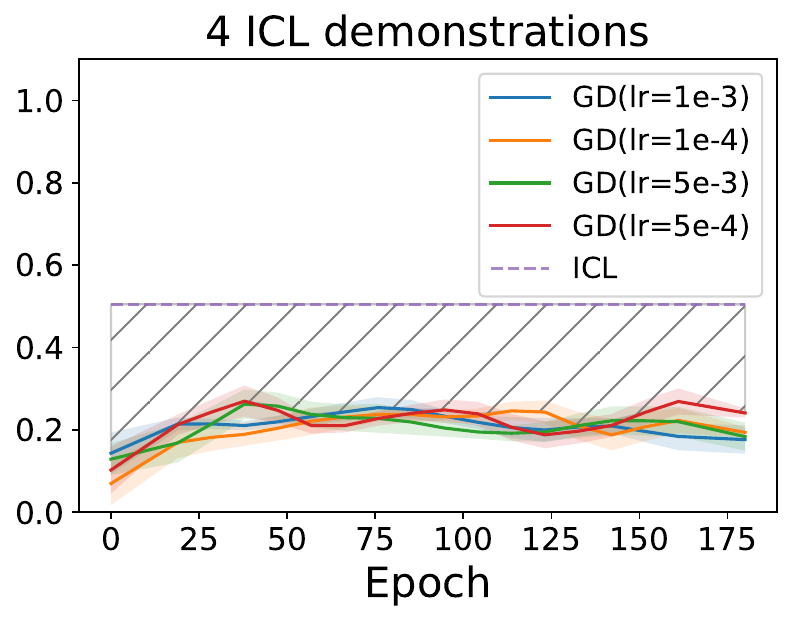}
    \includegraphics[width=0.24\textwidth]{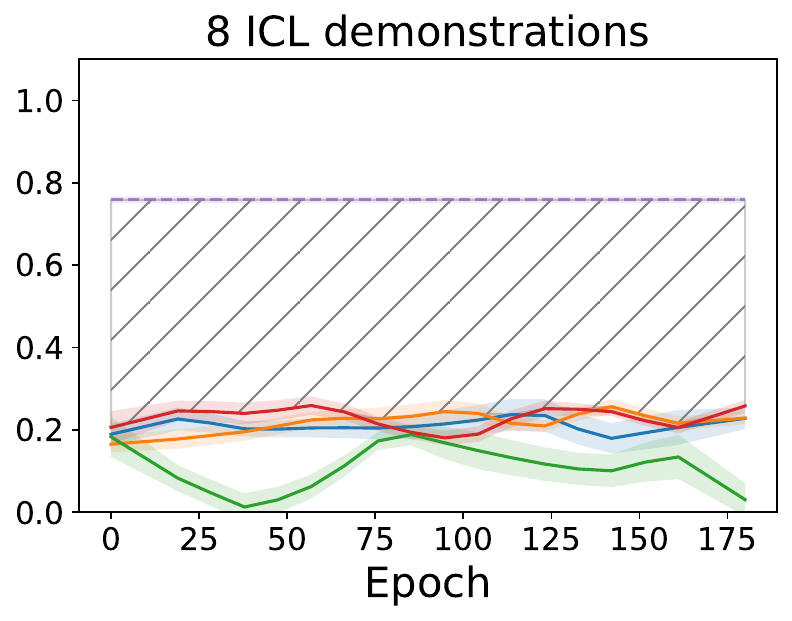}
    }\vspace{-3.5mm}
\caption{Comparison of ICL and $\widehat{\text{GD}}$ for the SST dataset, with increasing number of demonstrations. $\widehat{\text{GD}}$ is simulated by optimizing on one random deep layer of LLaMa.}  
\label{sst2-deep}
\end{figure}
\vspace{-2mm}

\begin{figure}[H]
\centering
    \subfigure[\textit{Accuracy} comparison]{
    \includegraphics[width=0.25\textwidth]{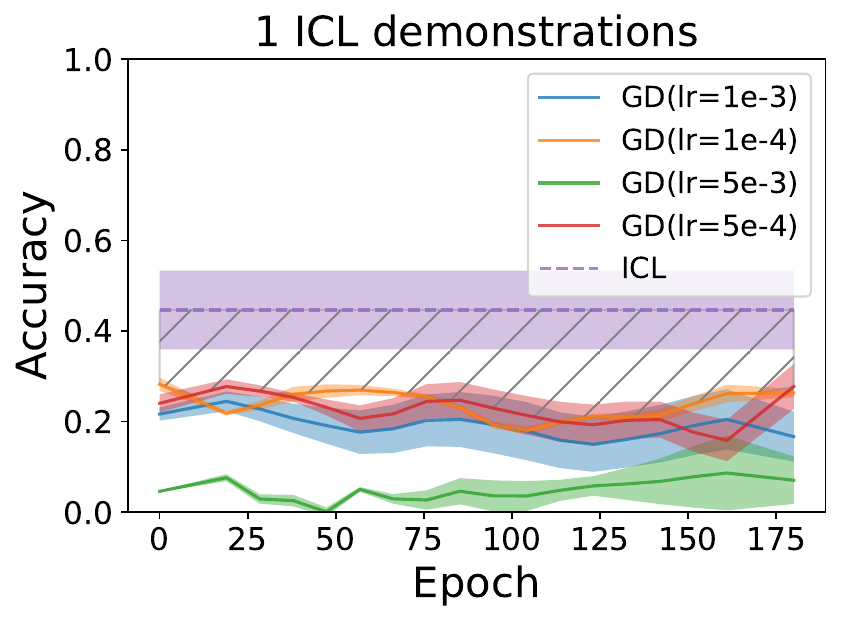}
    \includegraphics[width=0.24\textwidth]{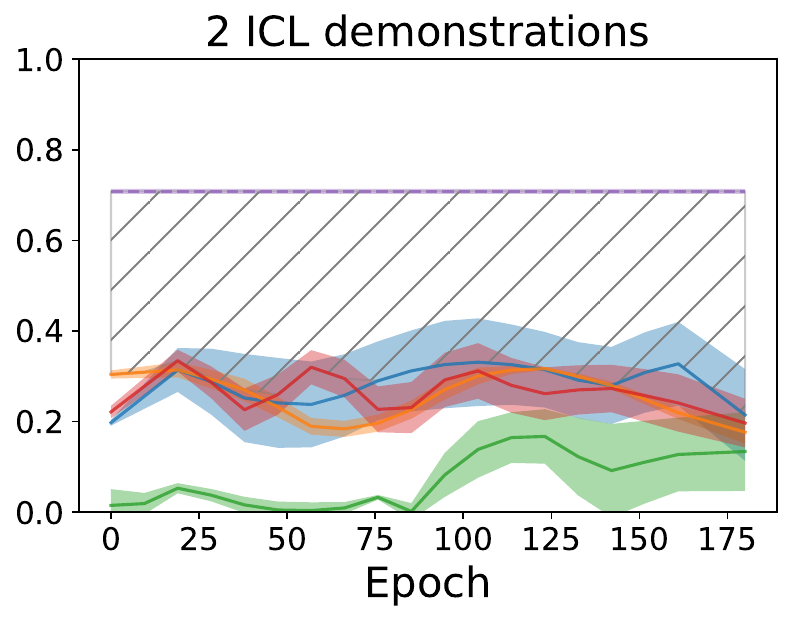}
    \includegraphics[width=0.24\textwidth]{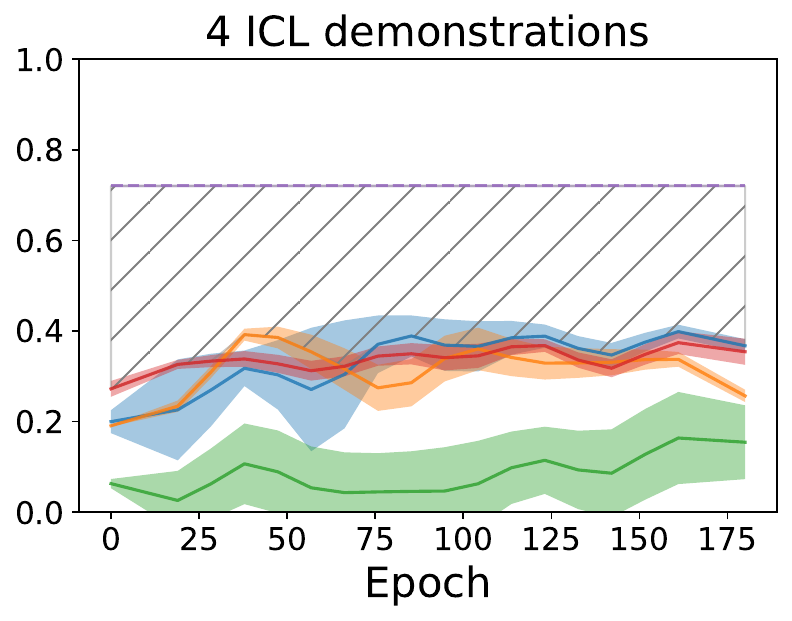}
    \includegraphics[width=0.24\textwidth]{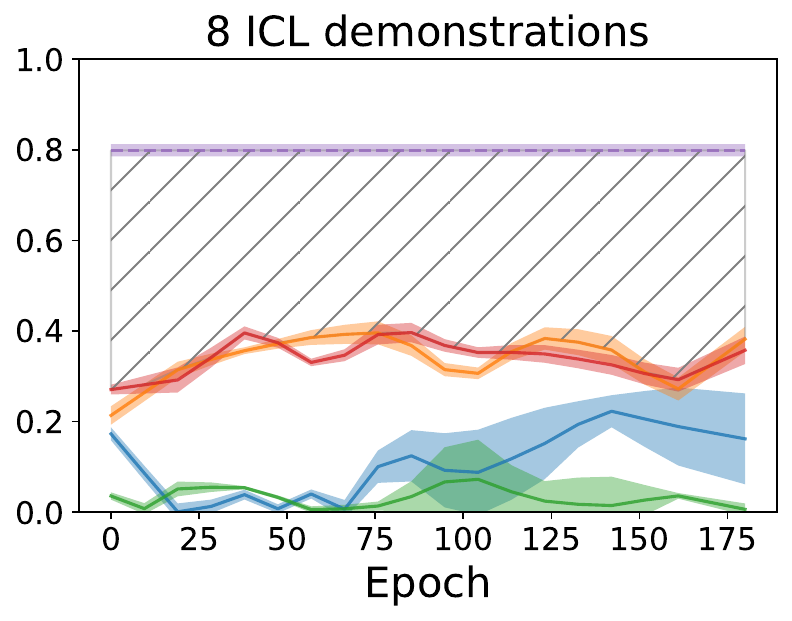}
    }
    \subfigure[\textit{Token overlap} comparison]{
    \includegraphics[width=0.25\textwidth]{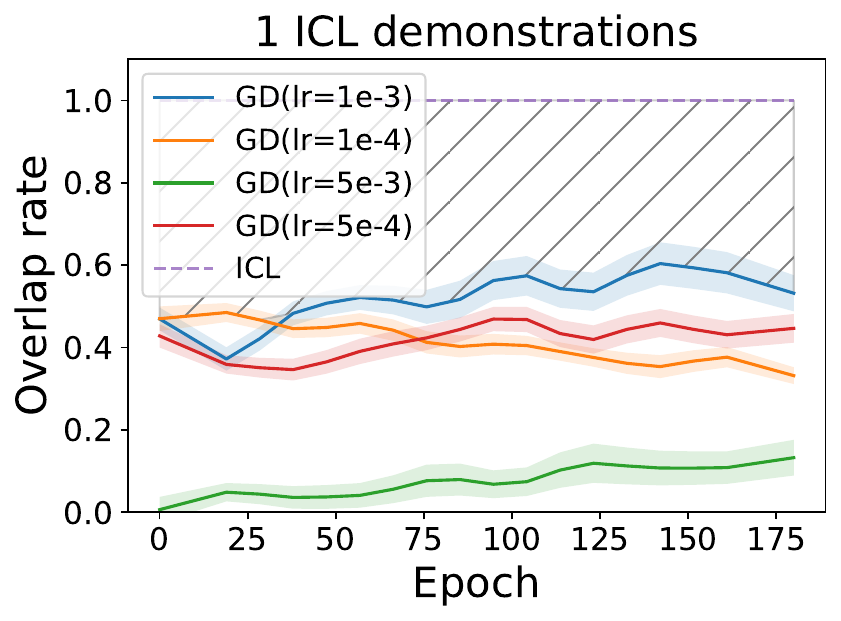}
    \includegraphics[width=0.24\textwidth]{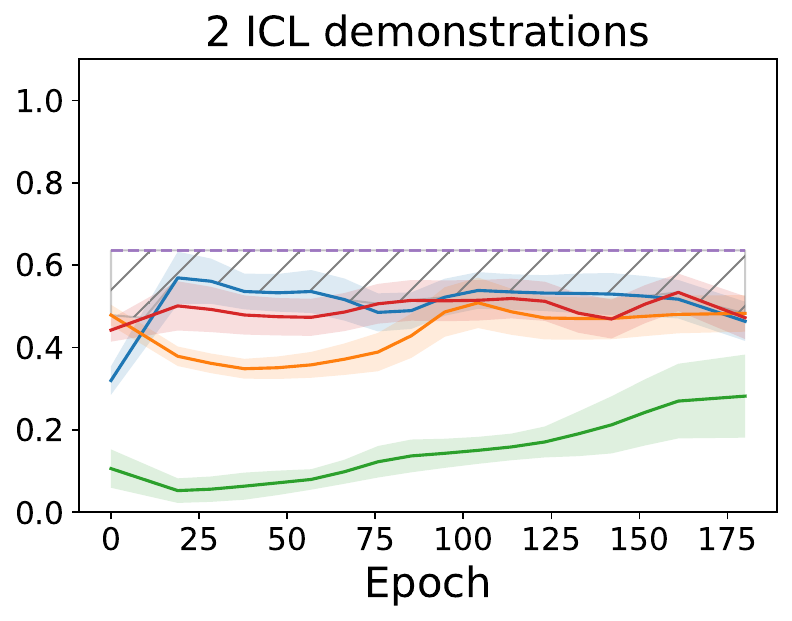}
    \includegraphics[width=0.24\textwidth]{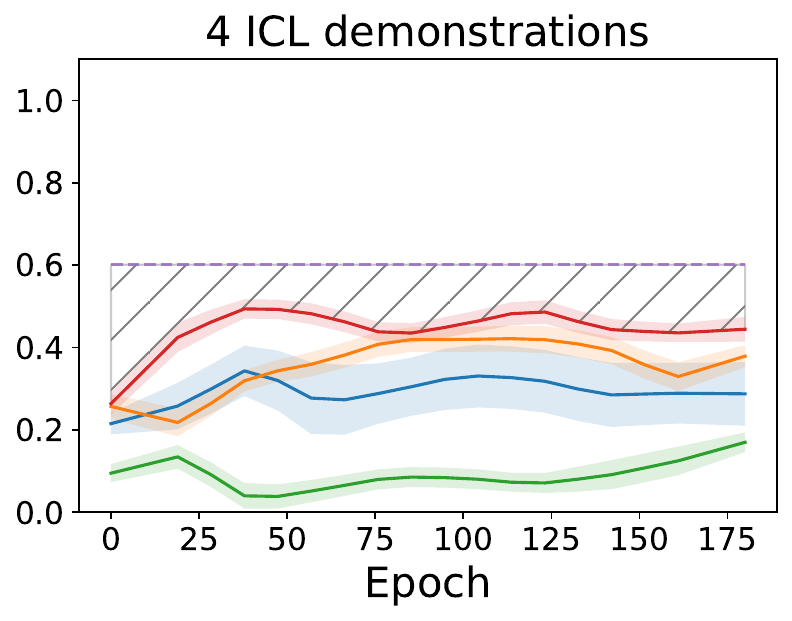}
    \includegraphics[width=0.24\textwidth]{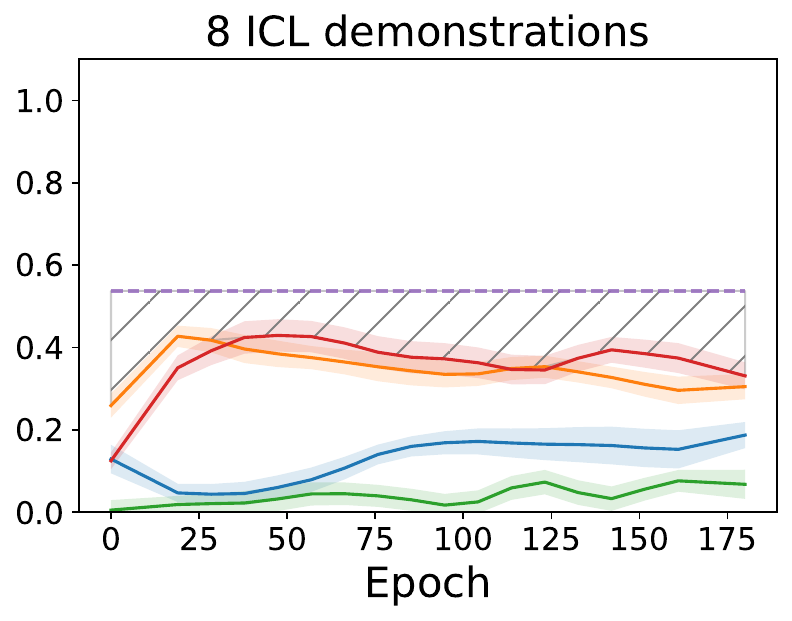}
    } 
    \subfigure[\textit{Overlap Cosine Similarity} comparison]{
    \includegraphics[width=0.25\textwidth]{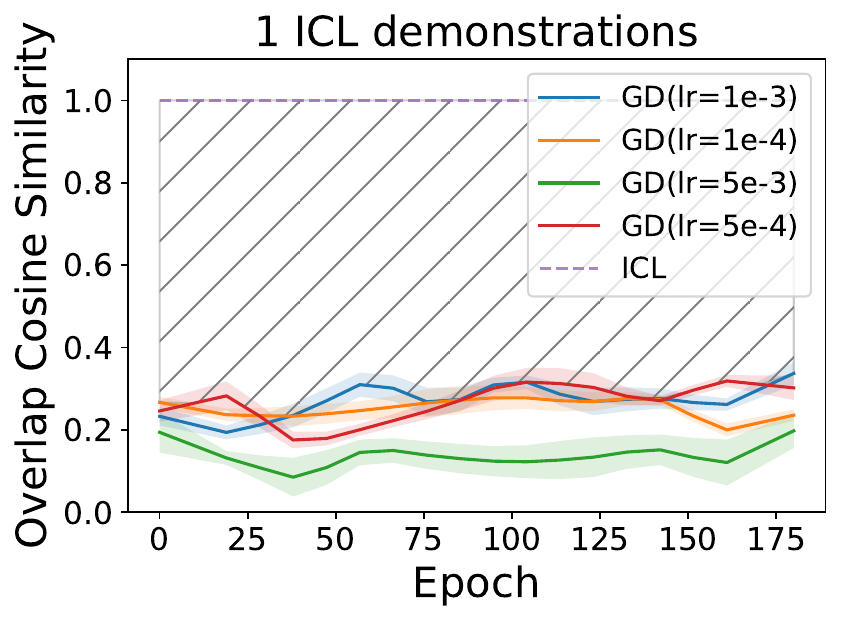}
    \includegraphics[width=0.24\textwidth]{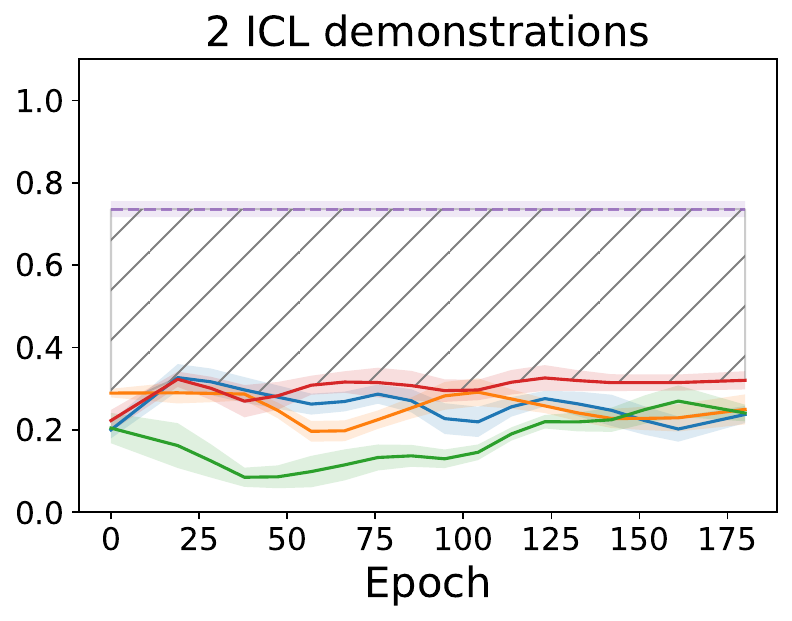}
    \includegraphics[width=0.24\textwidth]{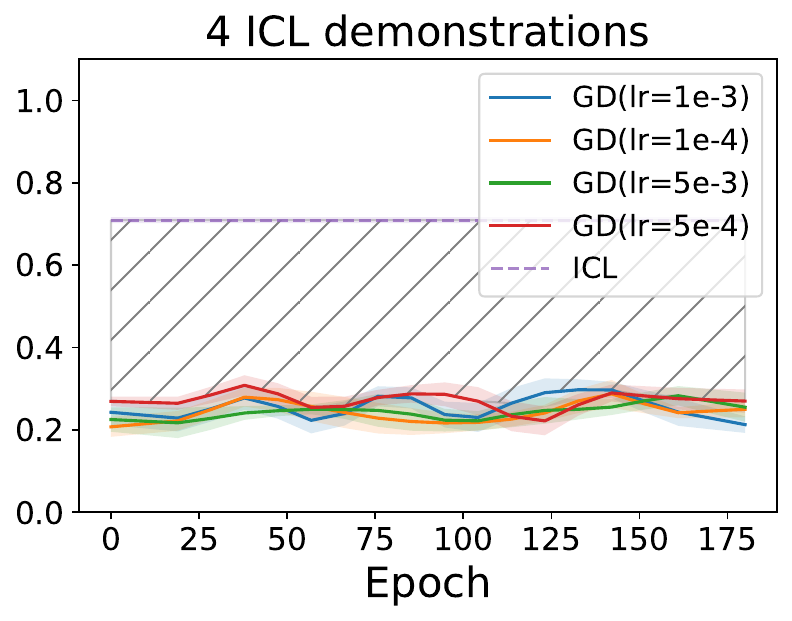}
    \includegraphics[width=0.24\textwidth]{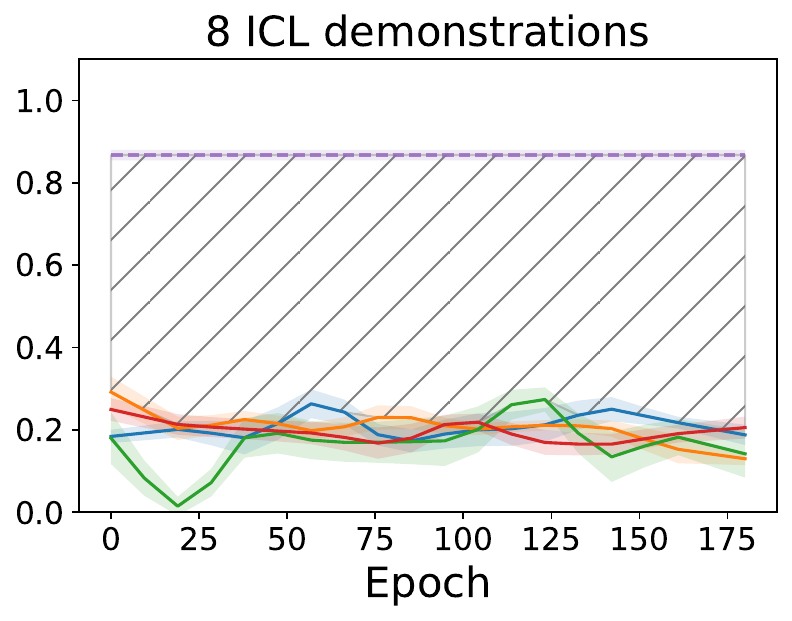}
    }\vspace{-3.5mm}
\caption{Comparison of ICL and $\widehat{\text{GD}}$ for the CB dataset, with increasing number of demonstrations. $\widehat{\text{GD}}$ is simulated by optimizing on one random deep layer of LLaMa.}  
\label{cb-deep}
\end{figure}
\vspace{-2mm}

\begin{figure}[H]
\centering
    \subfigure[\textit{Accuracy} comparison]{
    \includegraphics[width=0.25\textwidth]{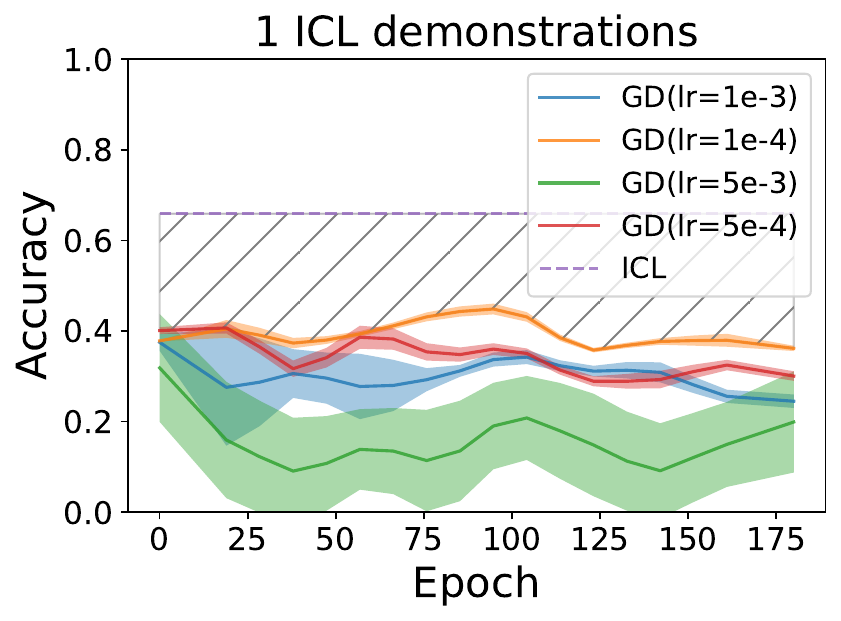}
    \includegraphics[width=0.24\textwidth]{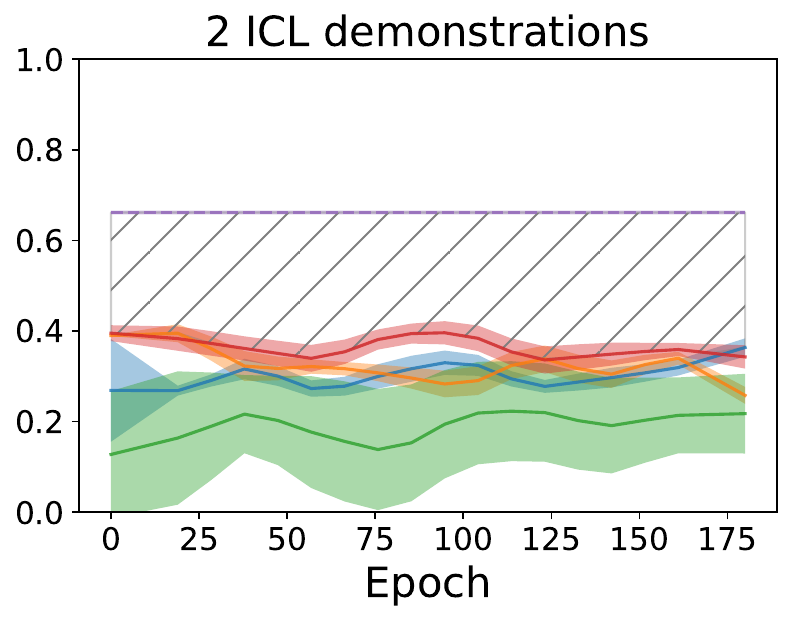}
    \includegraphics[width=0.24\textwidth]{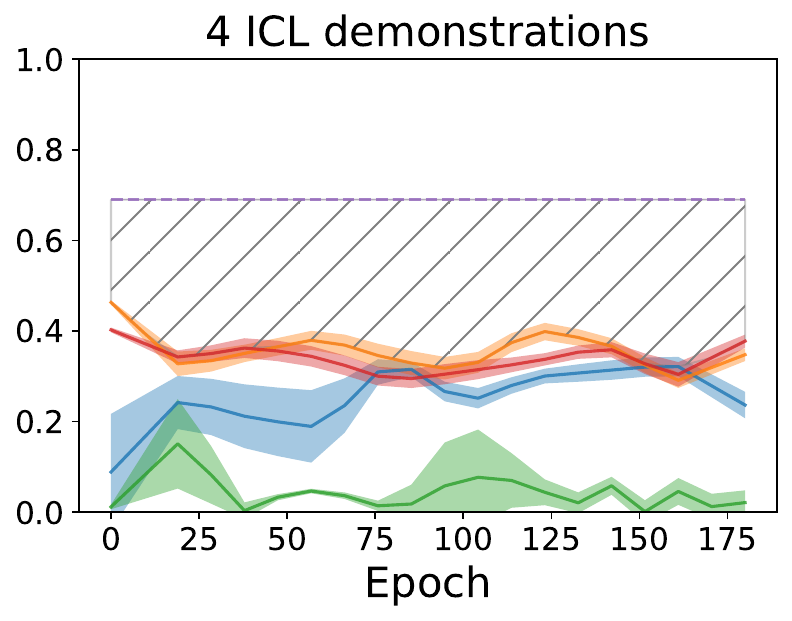}
    \includegraphics[width=0.24\textwidth]{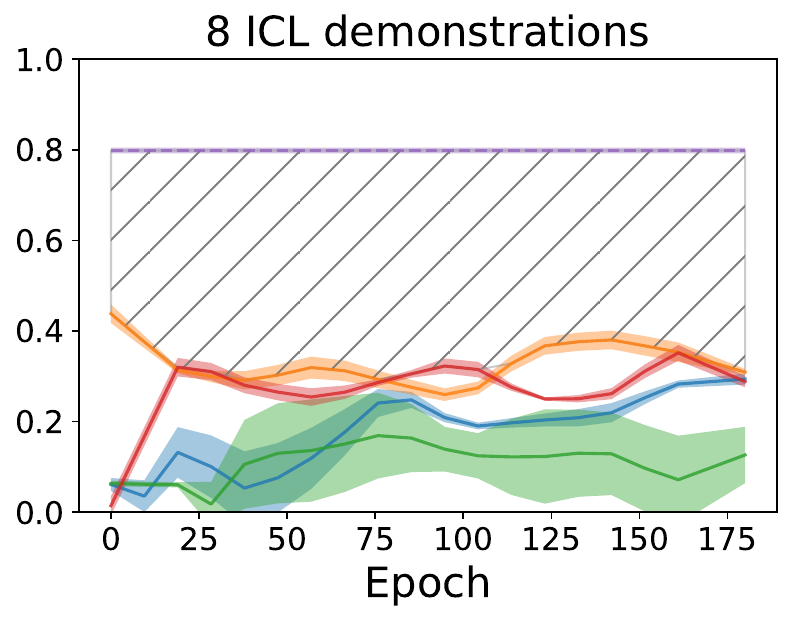}
    }
    \subfigure[\textit{Token overlap} comparison]{
    \includegraphics[width=0.25\textwidth]{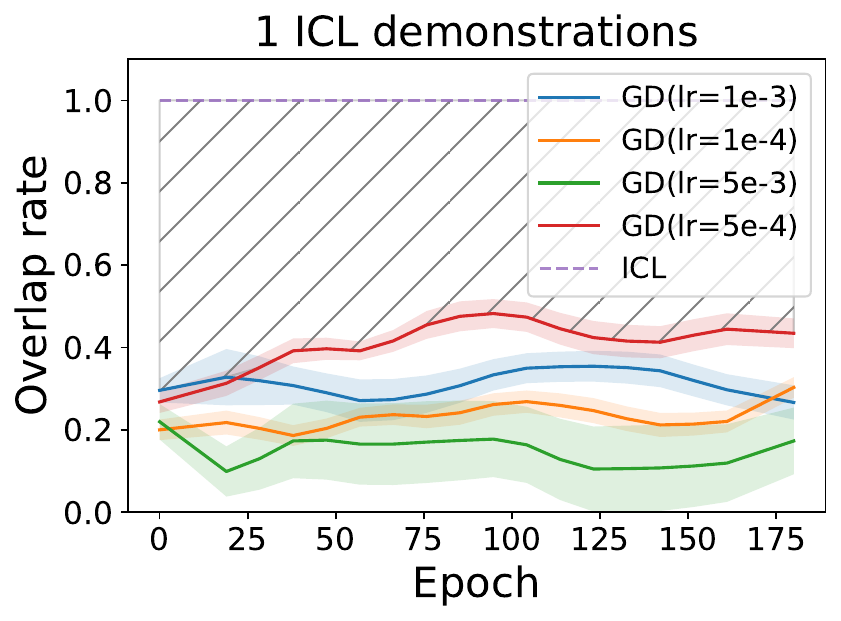}
    \includegraphics[width=0.24\textwidth]{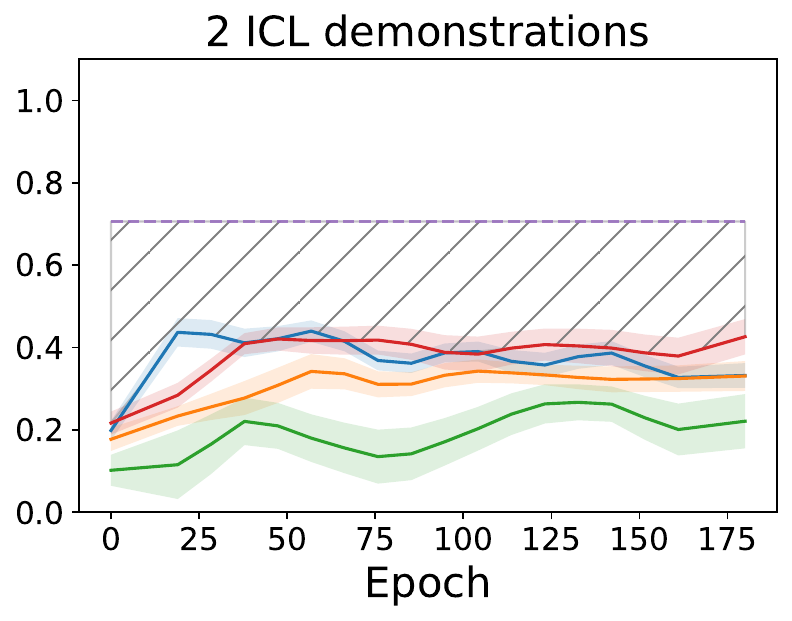}
    \includegraphics[width=0.24\textwidth]{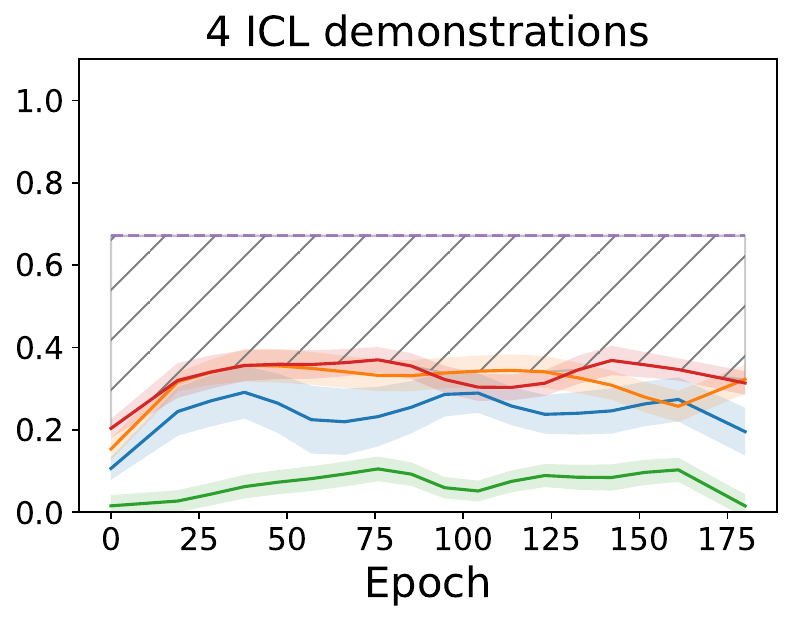}
    \includegraphics[width=0.24\textwidth]{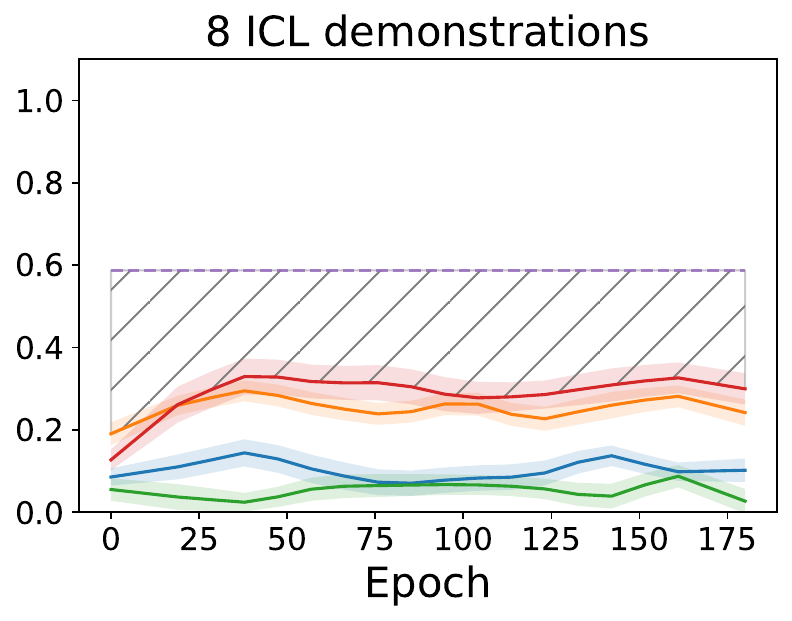}
    } 
    \subfigure[\textit{Overlap Cosine Similarity} comparison]{
    \includegraphics[width=0.25\textwidth]{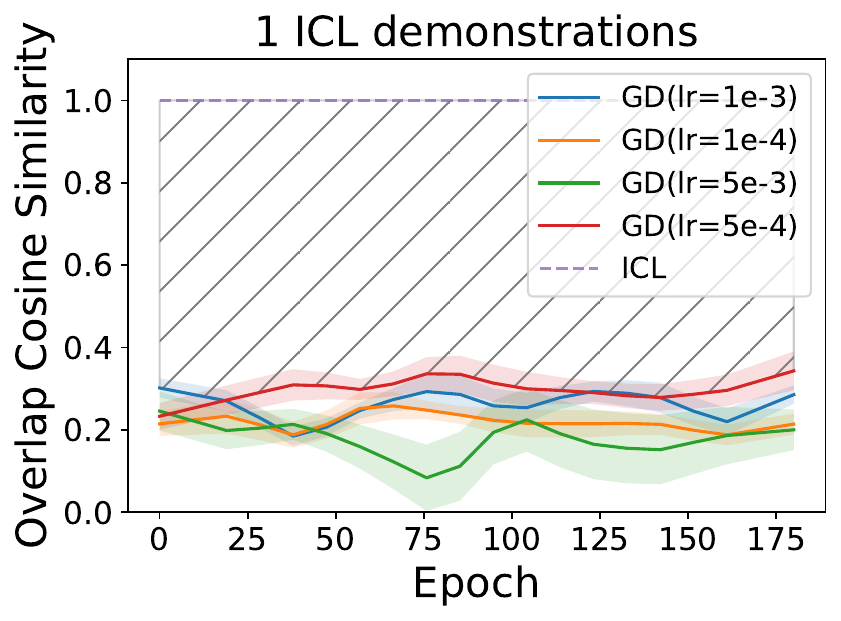}
    \includegraphics[width=0.24\textwidth]{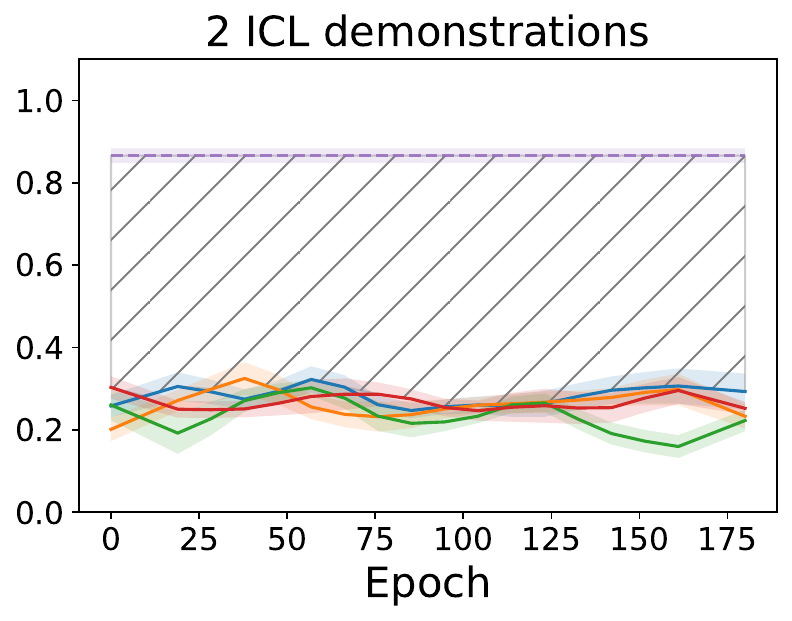}
    \includegraphics[width=0.24\textwidth]{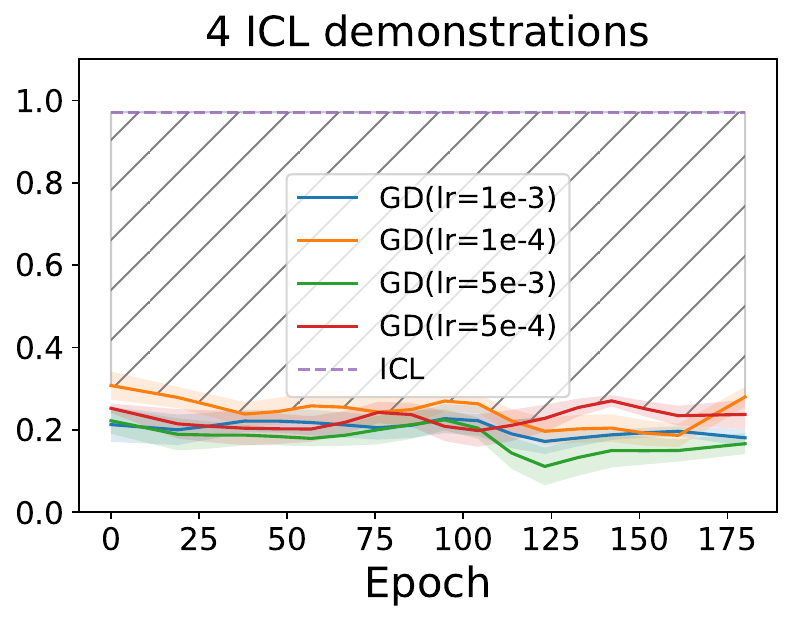}
    \includegraphics[width=0.24\textwidth]{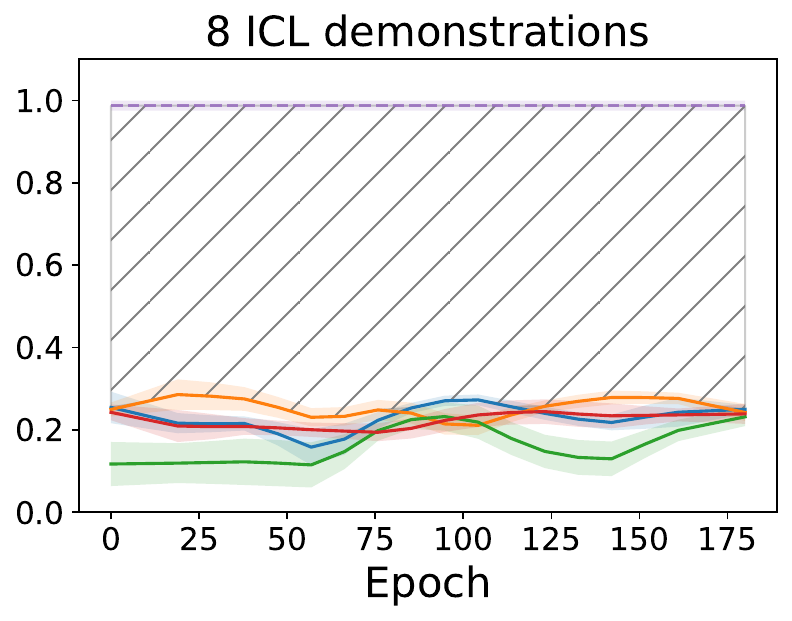}
    }\vspace{-3.5mm}
\caption{Comparison of ICL and $\widehat{\text{GD}}$ for the RTE dataset, with increasing number of demonstrations. $\widehat{\text{GD}}$ is simulated by optimizing on one random deep layer of LLaMa.}  
\label{rte-deep}
\end{figure}

\paragraph{Results of ICL vs. $\widehat{\text{GD}}$ (Middle layers)}
This time, we randomly select one layer from the middle layers of LLaMa (16-20).
The results are shown in \autoref{ag-new2} - \autoref{rte-new2}, we can observe similar gaps between ICL and $\widehat{\text{GD}}$.

\begin{figure}[H]
\centering
    \subfigure[\textit{Accuracy} comparison]{
    \includegraphics[width=0.25\textwidth]{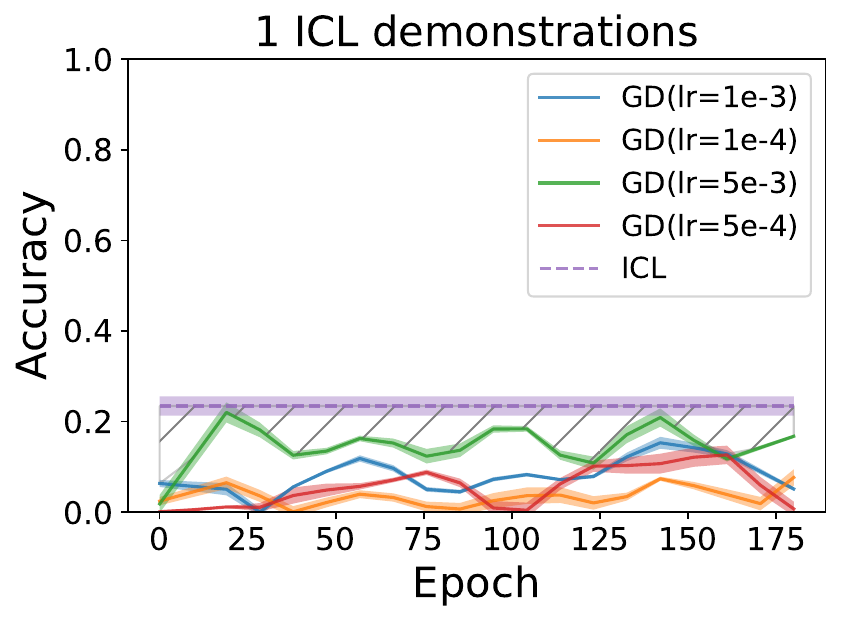}
    \includegraphics[width=0.24\textwidth]{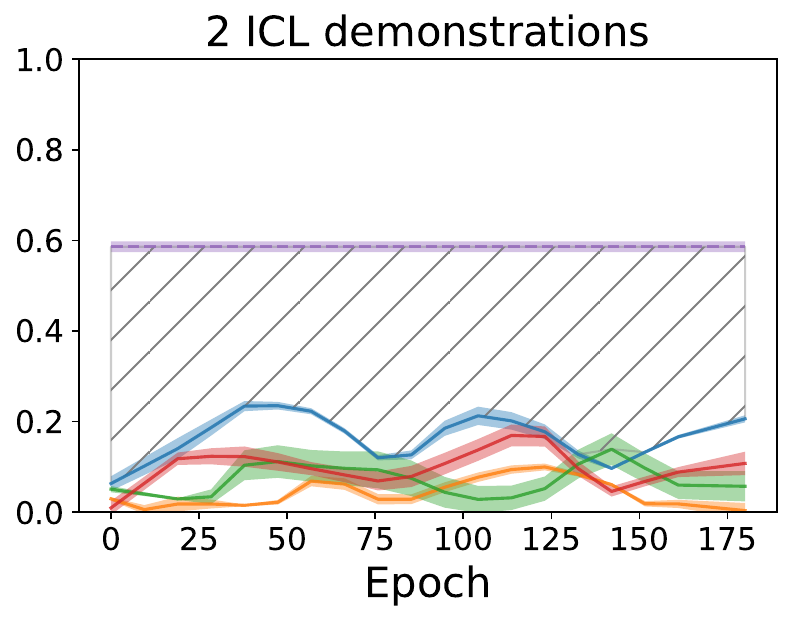}
    \includegraphics[width=0.24\textwidth]{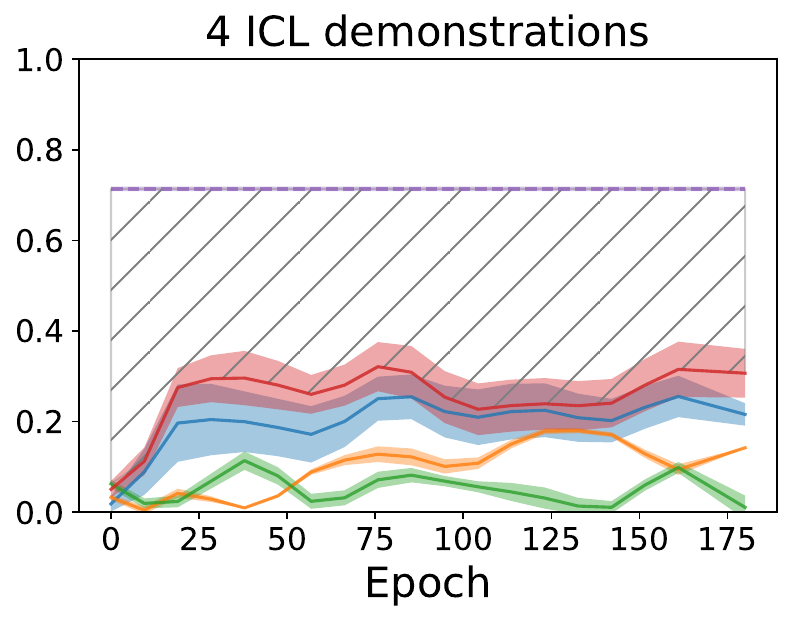}
    \includegraphics[width=0.24\textwidth]{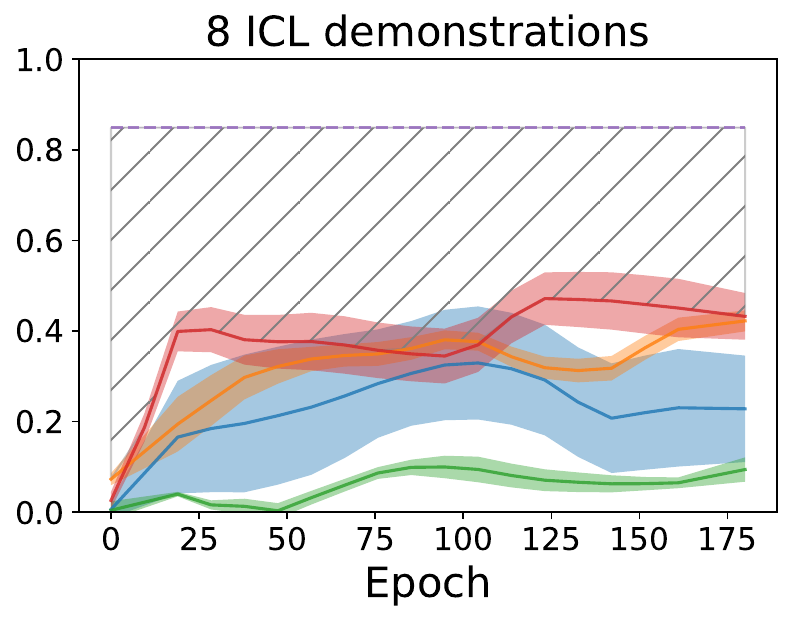}
    }
    \subfigure[\textit{Token overlap} comparison]{
    \includegraphics[width=0.25\textwidth]{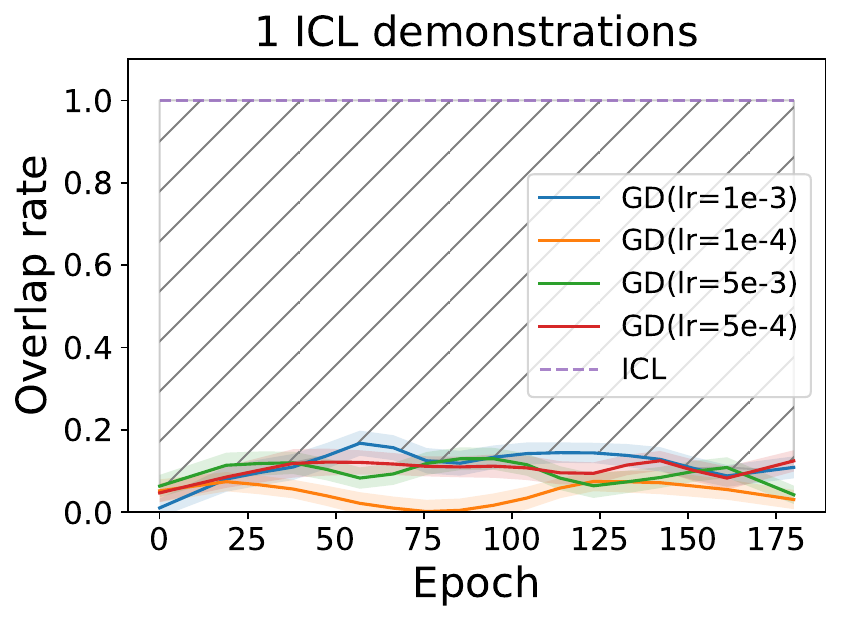}
    \includegraphics[width=0.24\textwidth]{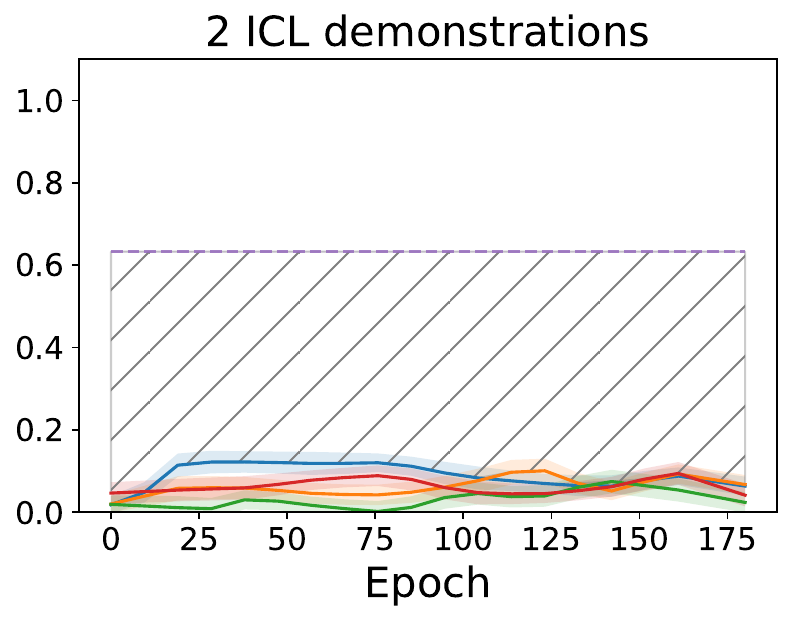}
    \includegraphics[width=0.24\textwidth]{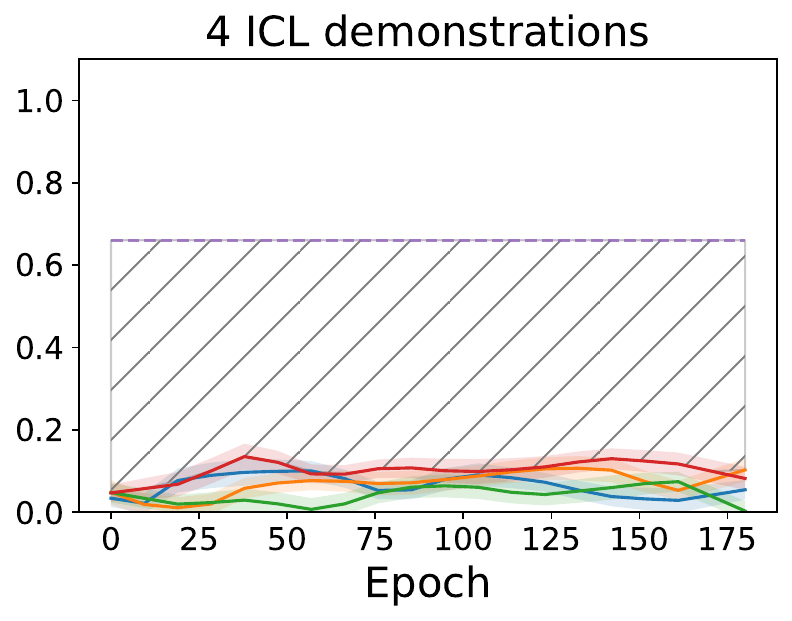}
    \includegraphics[width=0.24\textwidth]{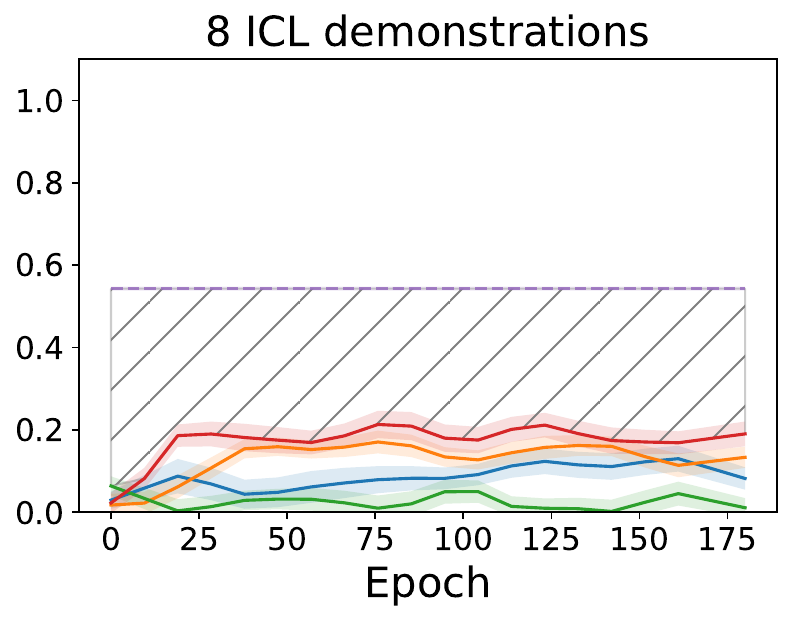}
    } 
    \subfigure[\textit{Overlap Cosine Similarity} comparison]{
    \includegraphics[width=0.25\textwidth]{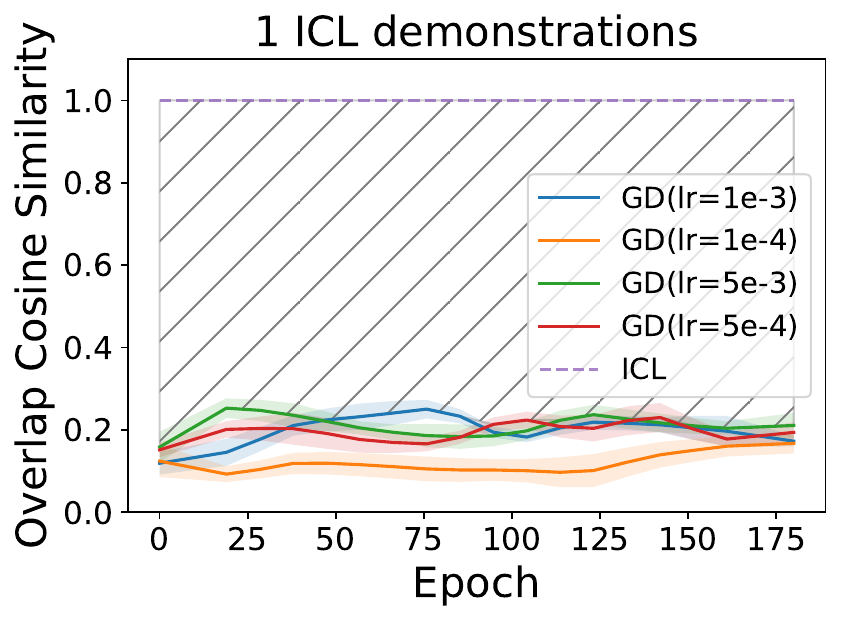}
    \includegraphics[width=0.24\textwidth]{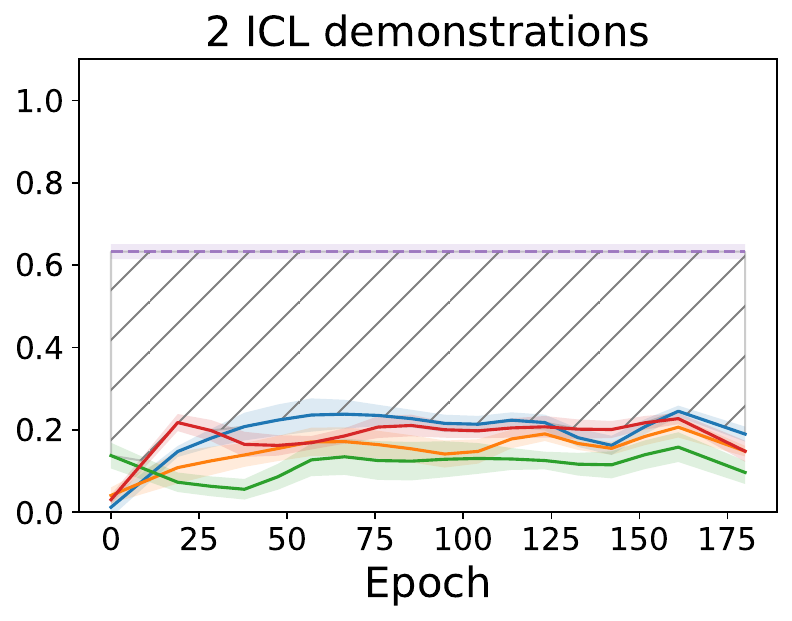}
    \includegraphics[width=0.24\textwidth]{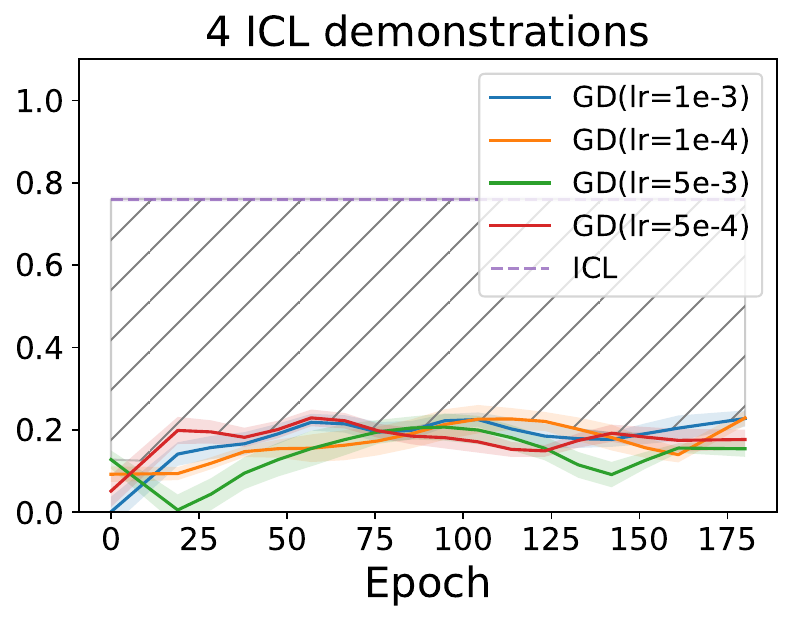}
    \includegraphics[width=0.24\textwidth]{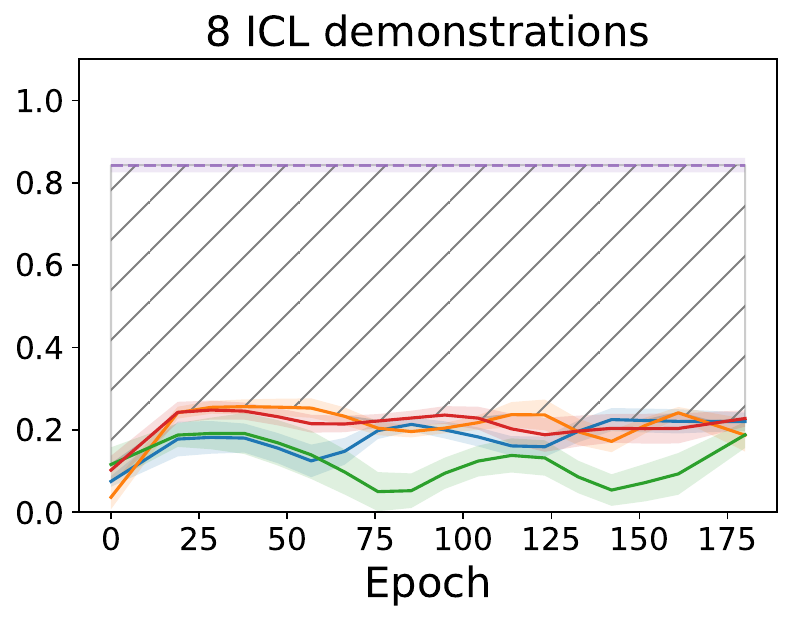}
    }\vspace{-3.5mm}
\caption{Comparison of ICL and $\widehat{\text{GD}}$ for the AGNews dataset, with increasing number of demonstrations. $\widehat{\text{GD}}$ is simulated by optimizing on one random middle layer of LLaMa.}  
\label{ag-new2}
\end{figure}
\vspace{-2mm}

\begin{figure}[H]
\centering
    \subfigure[\textit{Accuracy} comparison]{
    \includegraphics[width=0.25\textwidth]{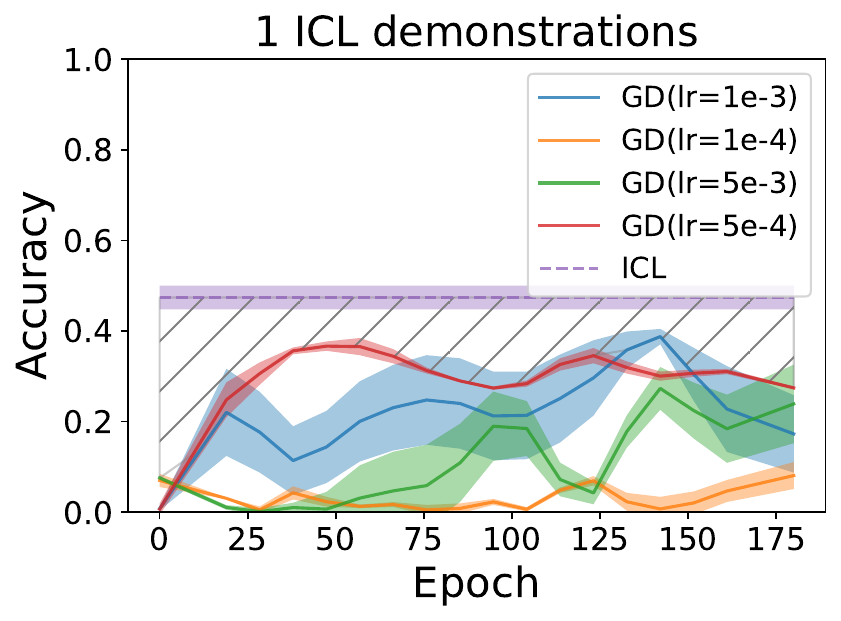}
    \includegraphics[width=0.24\textwidth]{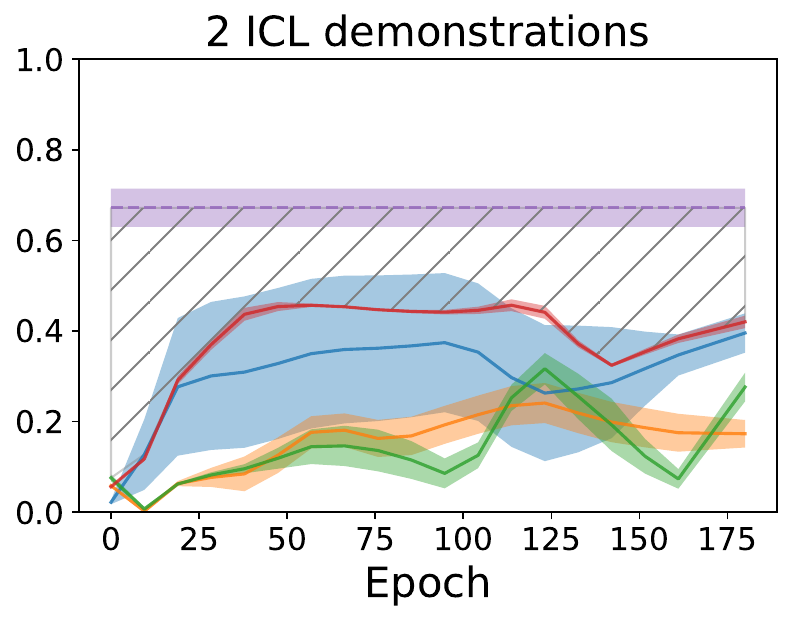}
    \includegraphics[width=0.24\textwidth]{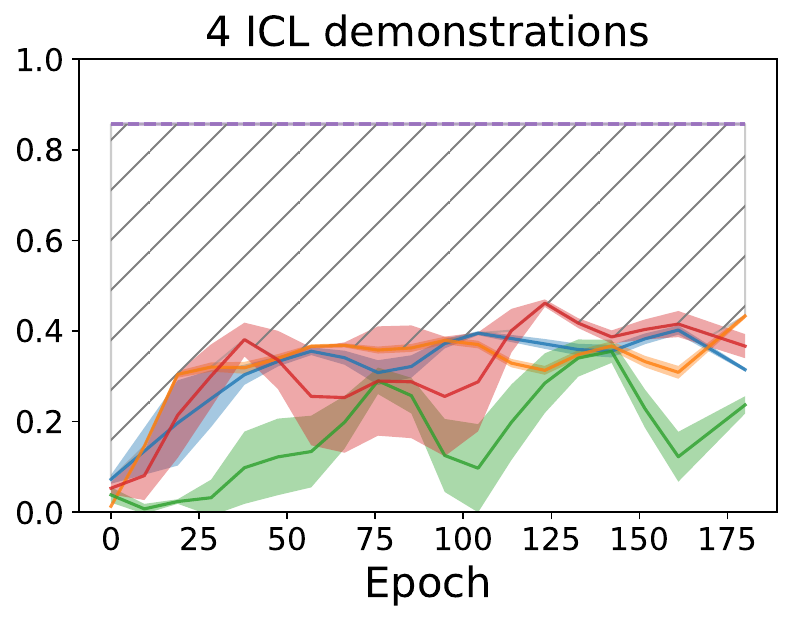}
    \includegraphics[width=0.24\textwidth]{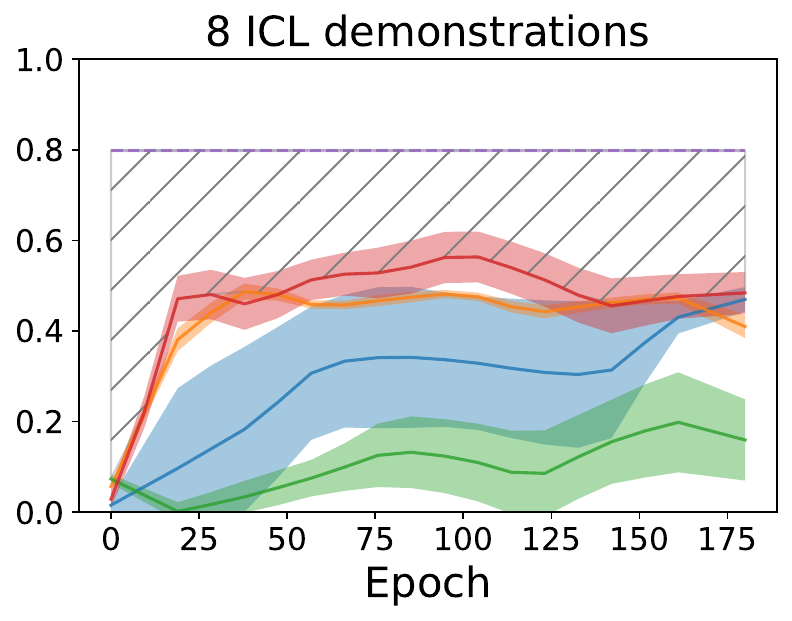}
    }
    \subfigure[\textit{Token overlap} comparison]{
    \includegraphics[width=0.25\textwidth]{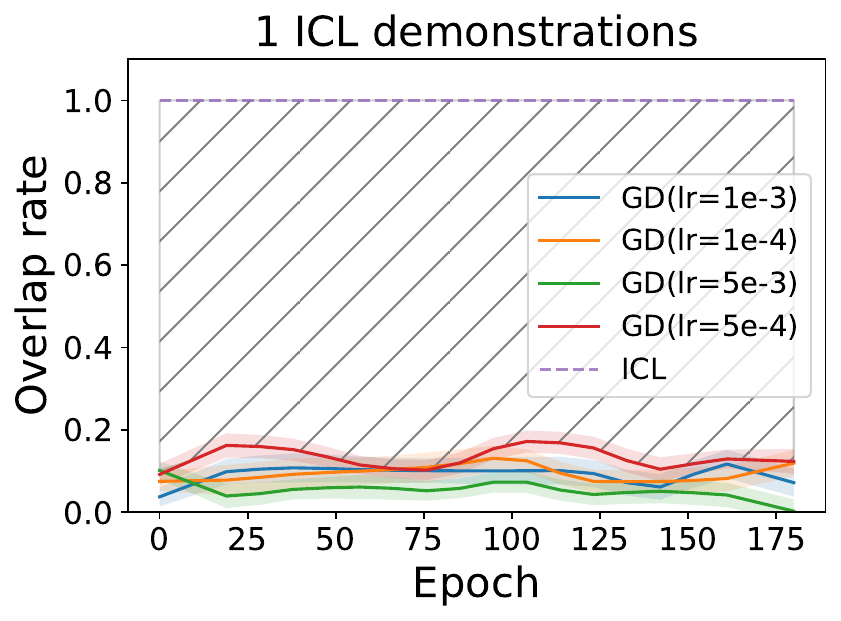}
    \includegraphics[width=0.24\textwidth]{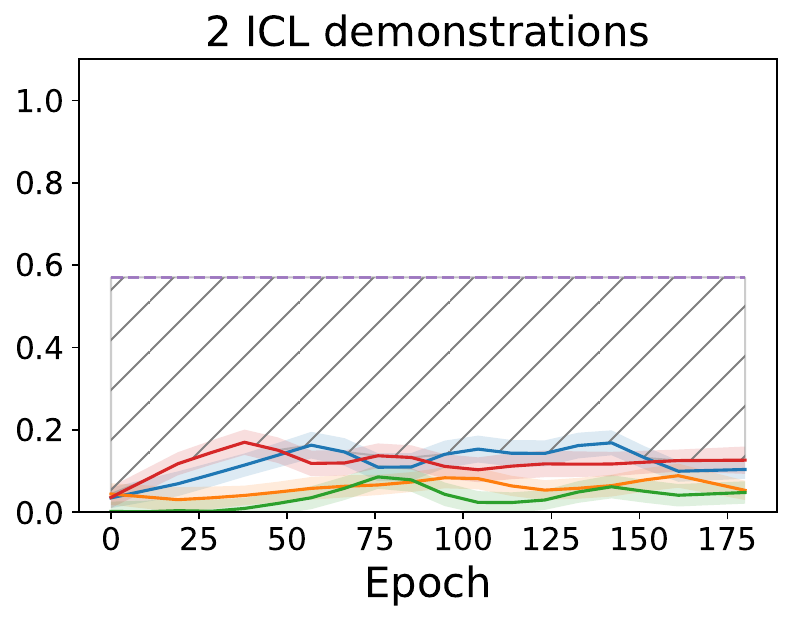}
    \includegraphics[width=0.24\textwidth]{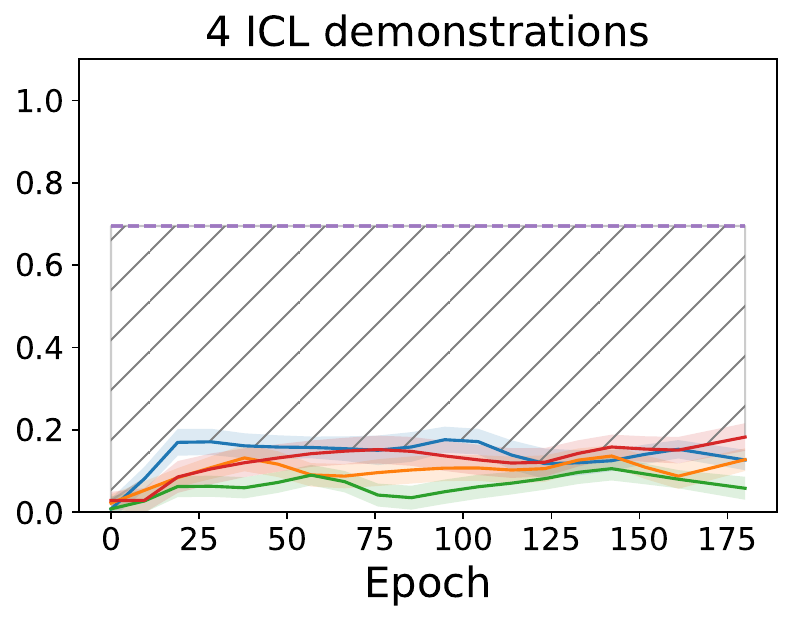}
    \includegraphics[width=0.24\textwidth]{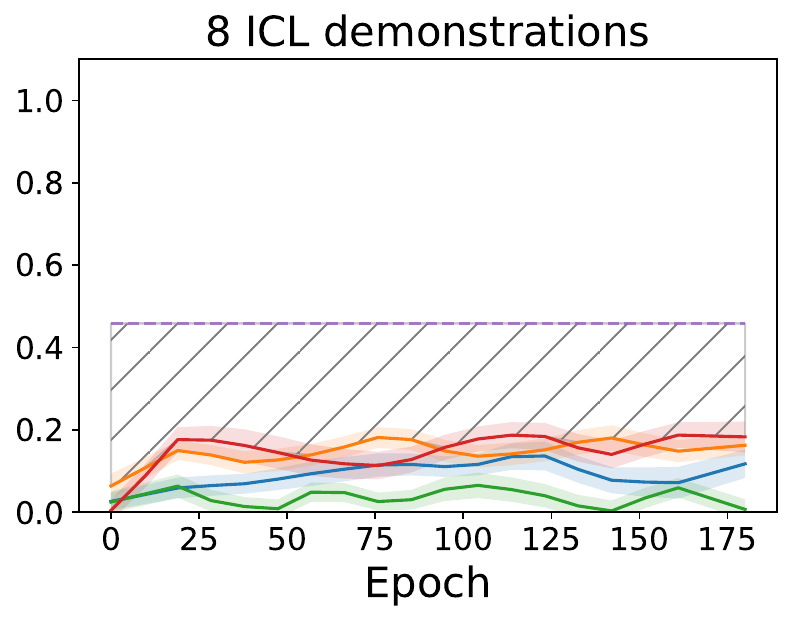}
    } 
    \subfigure[\textit{Overlap Cosine Similarity} comparison]{
    \includegraphics[width=0.25\textwidth]{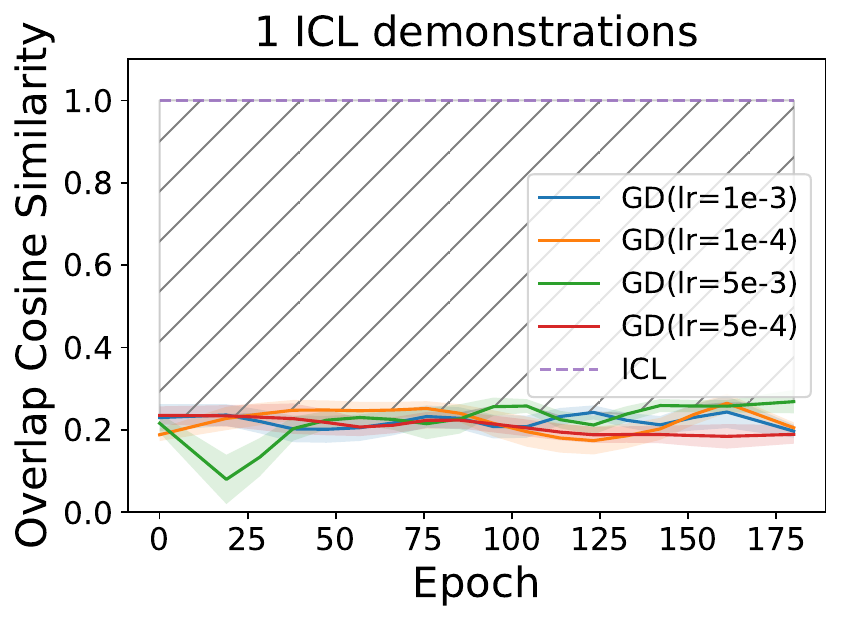}
    \includegraphics[width=0.24\textwidth]{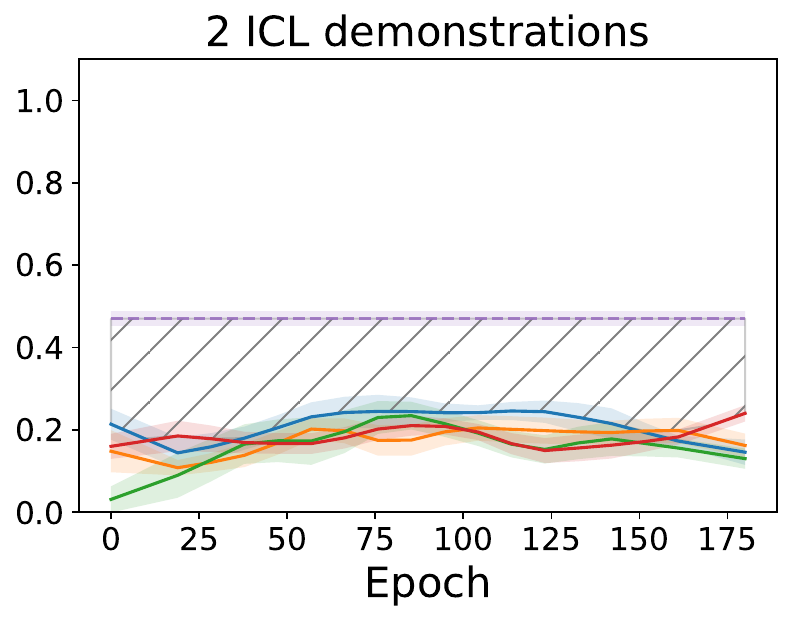}
    \includegraphics[width=0.24\textwidth]{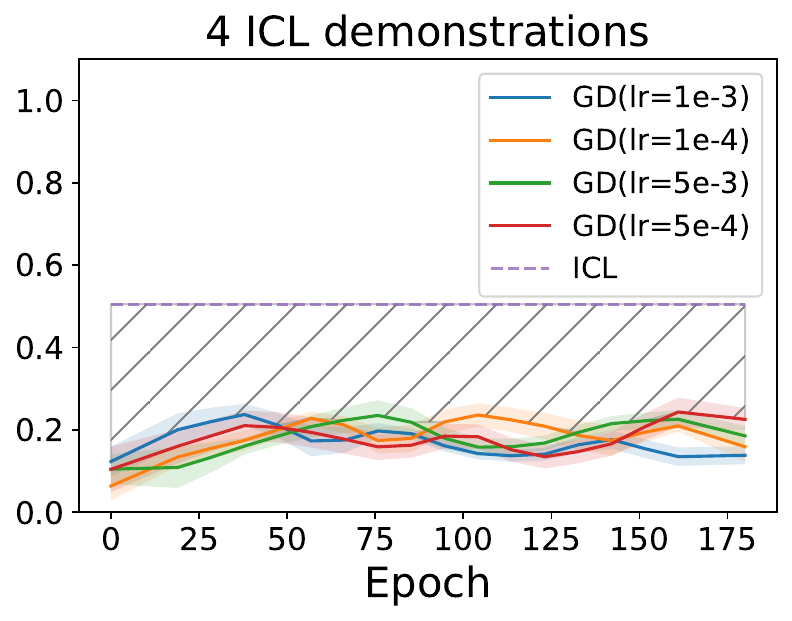}
    \includegraphics[width=0.24\textwidth]{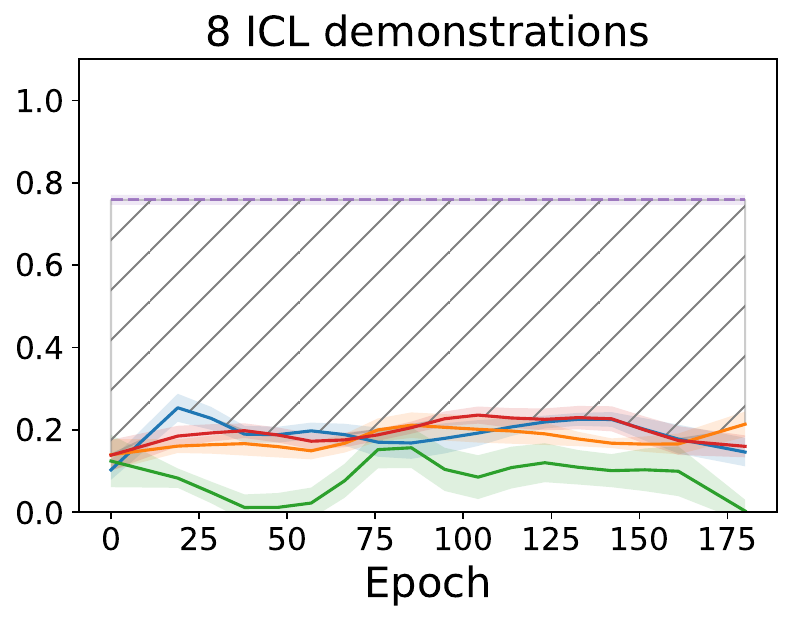}
    }\vspace{-3.5mm}
\caption{Comparison of ICL and $\widehat{\text{GD}}$ for the SST dataset, with increasing number of demonstrations. $\widehat{\text{GD}}$ is simulated by optimizing on one random middle layer of LLaMa.}  
\label{sst2-new2}
\end{figure}
\vspace{-2mm}

\begin{figure}[H]
\centering
    \subfigure[\textit{Accuracy} comparison]{
    \includegraphics[width=0.25\textwidth]{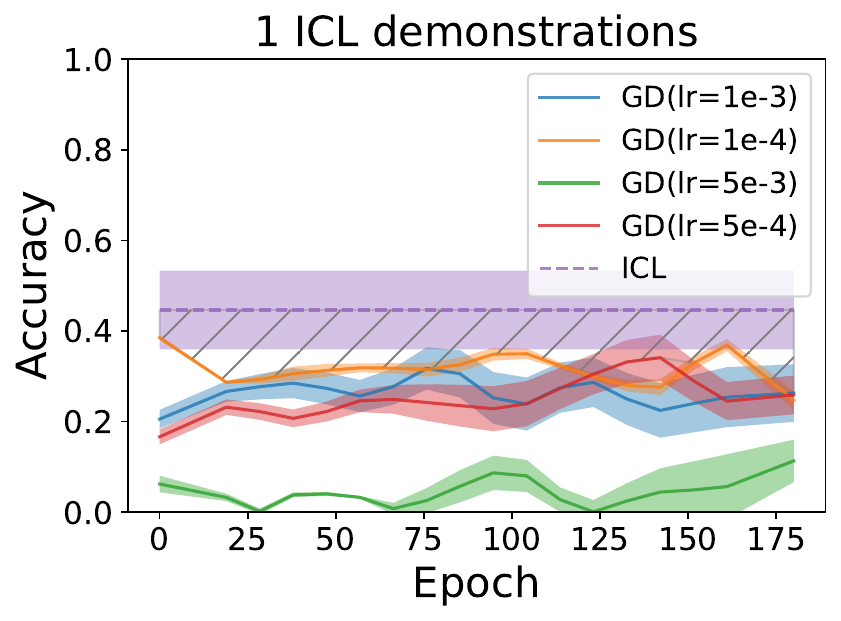}
    \includegraphics[width=0.24\textwidth]{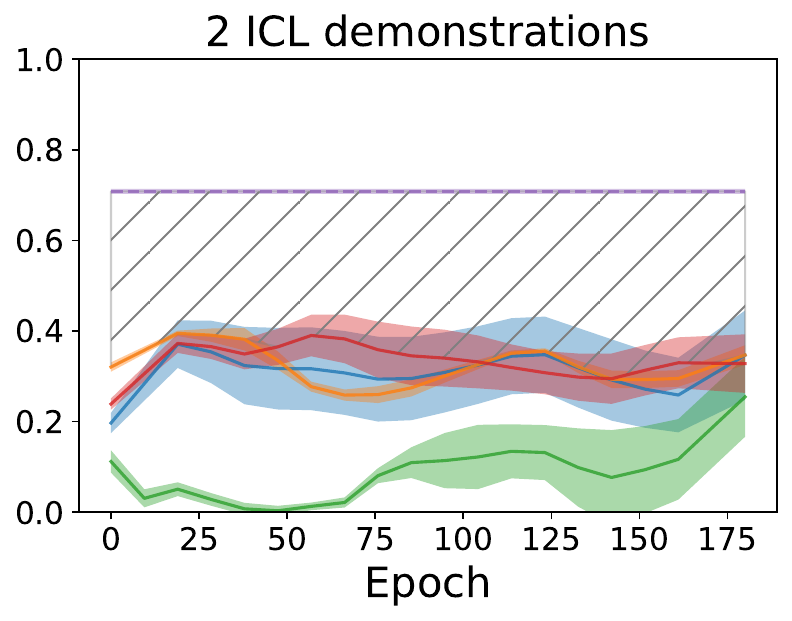}
    \includegraphics[width=0.24\textwidth]{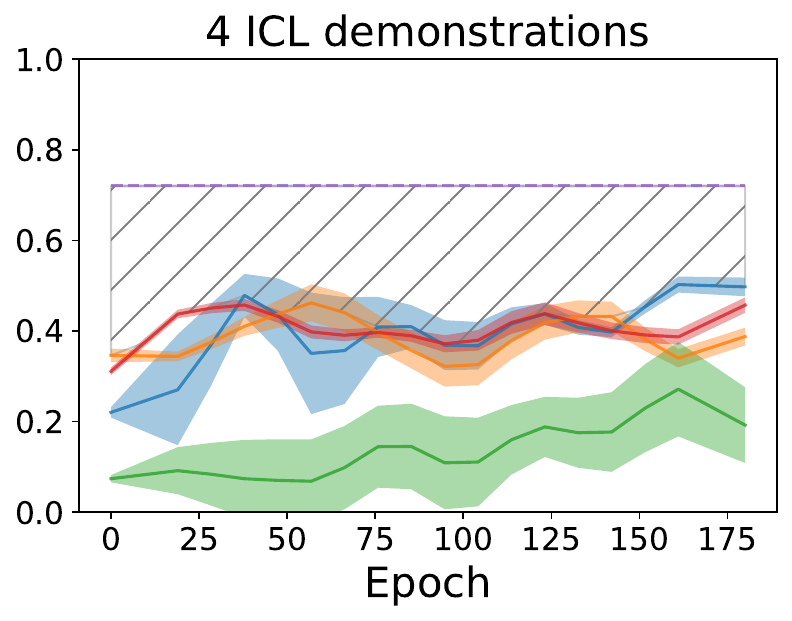}
    \includegraphics[width=0.24\textwidth]{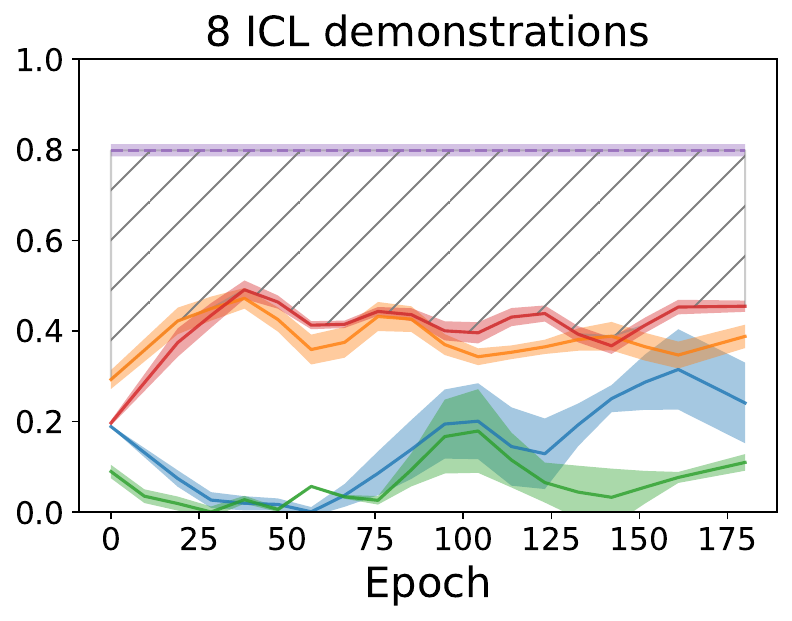}
    }
    \subfigure[\textit{Token overlap} comparison]{
    \includegraphics[width=0.25\textwidth]{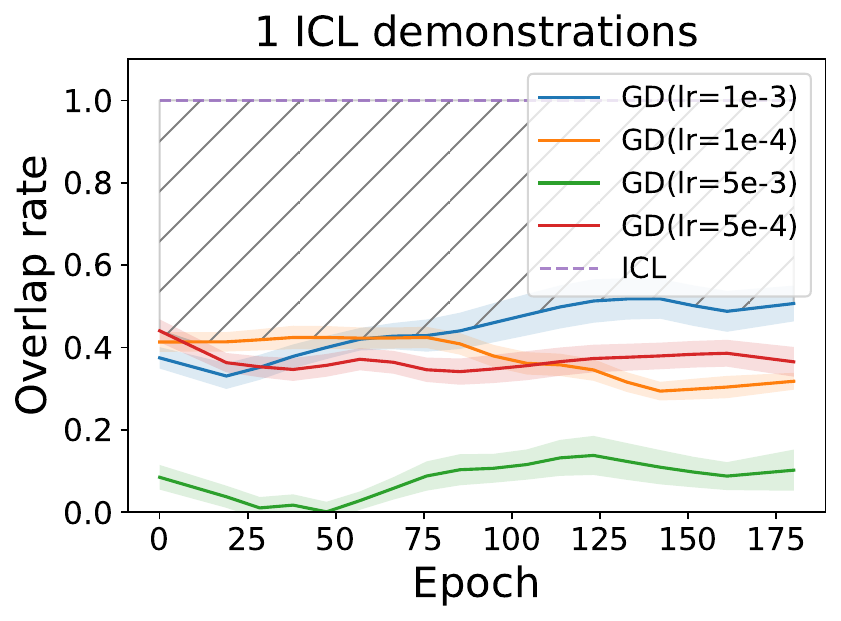}
    \includegraphics[width=0.24\textwidth]{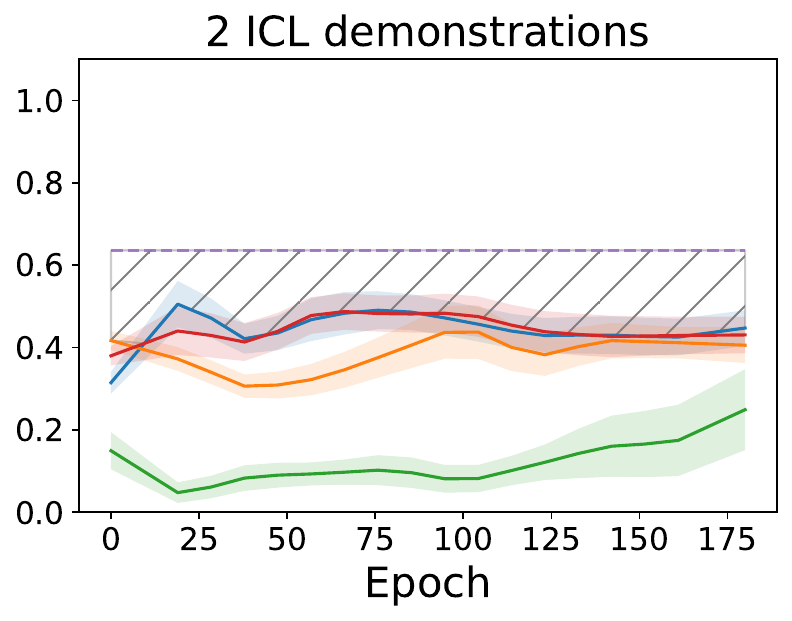}
    \includegraphics[width=0.24\textwidth]{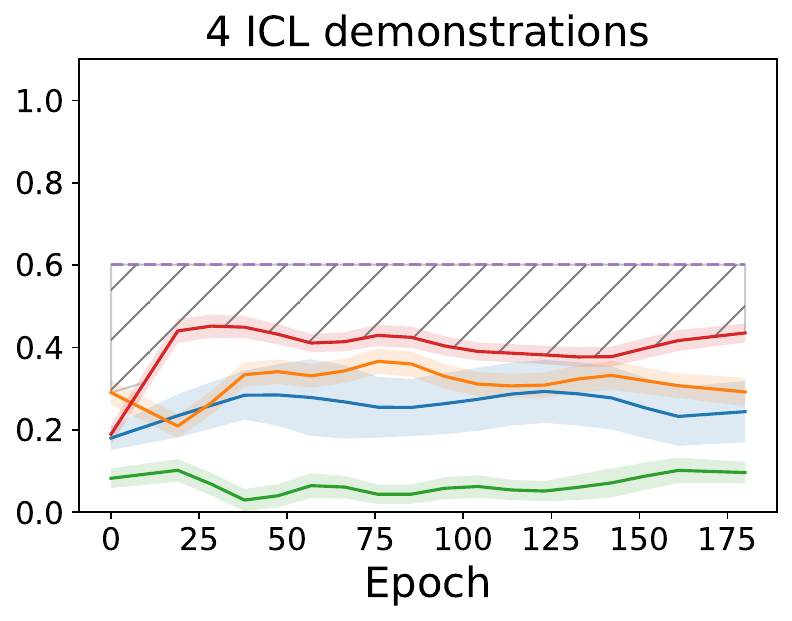}
    \includegraphics[width=0.24\textwidth]{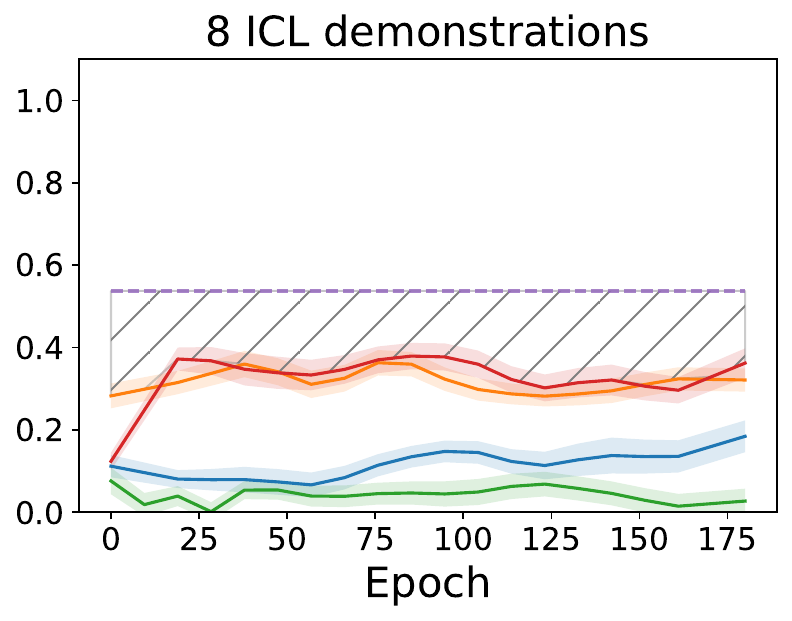}
    } 
    \subfigure[\textit{Overlap Cosine Similarity} comparison]{
    \includegraphics[width=0.25\textwidth]{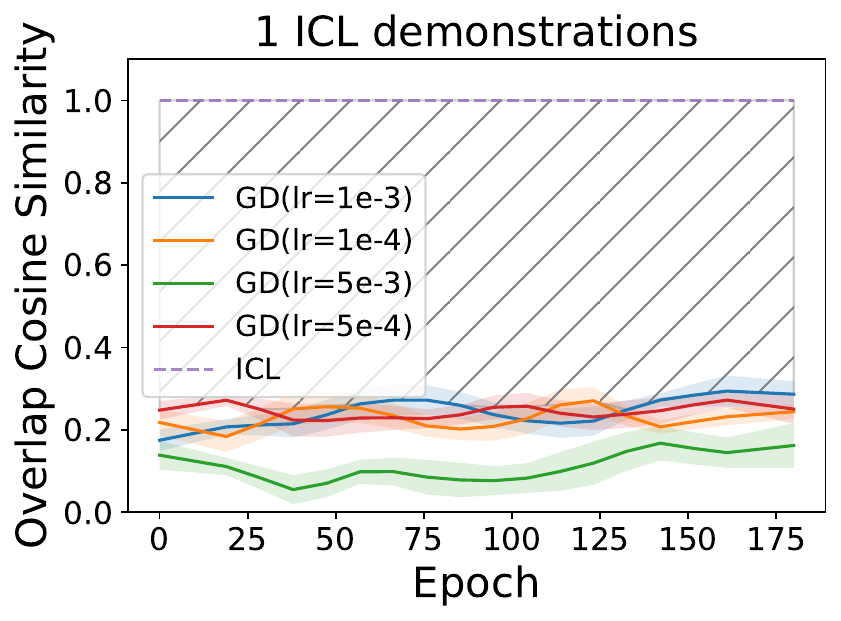}
    \includegraphics[width=0.24\textwidth]{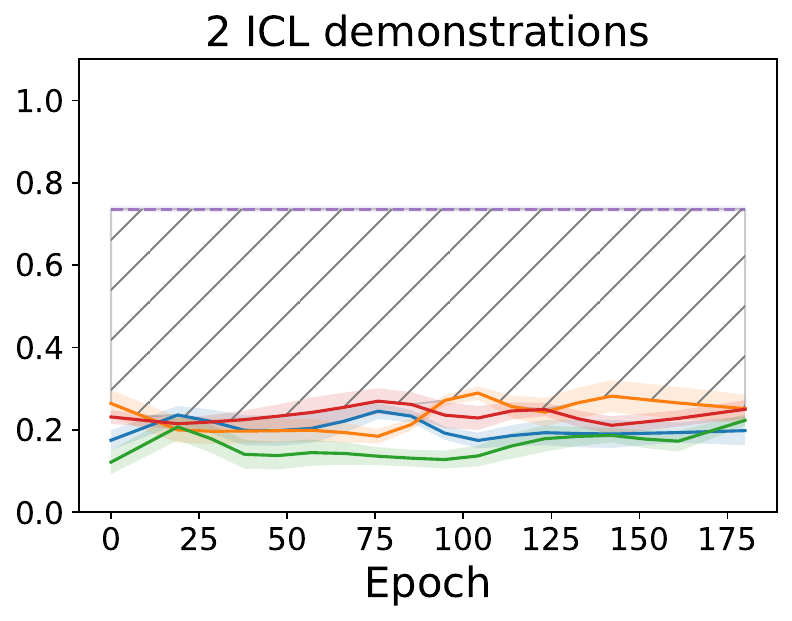}
    \includegraphics[width=0.24\textwidth]{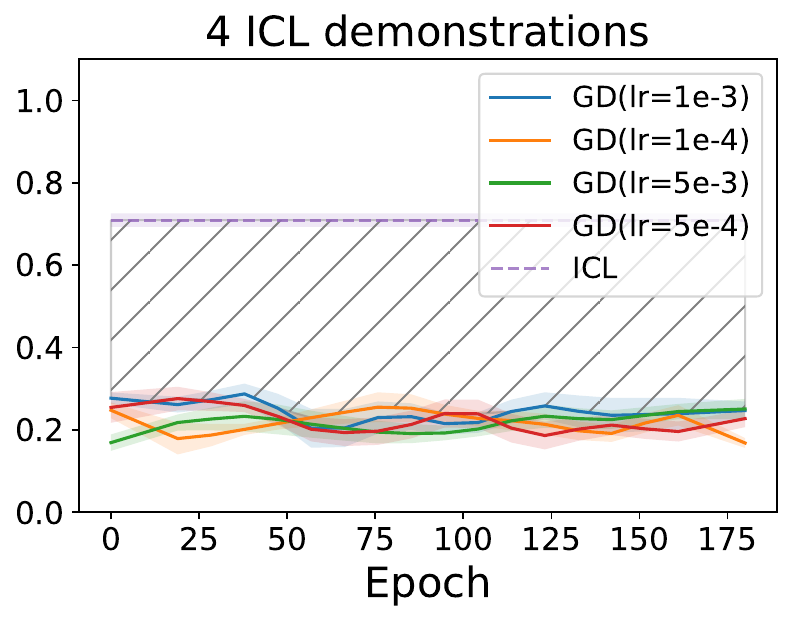}
    \includegraphics[width=0.24\textwidth]{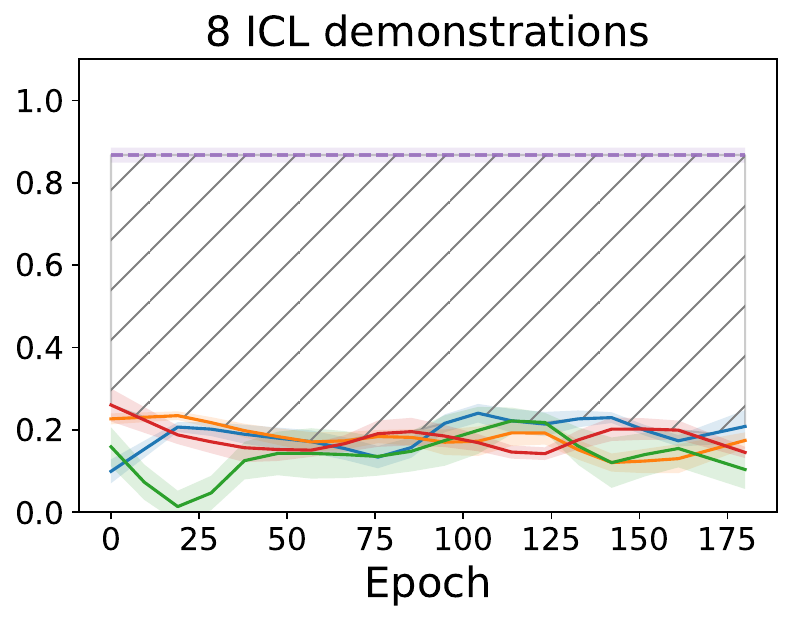}
    }\vspace{-3.5mm}
\caption{Comparison of ICL and $\widehat{\text{GD}}$ for the CB dataset, with increasing number of demonstrations. $\widehat{\text{GD}}$ is simulated by optimizing on one random middle layer of LLaMa.}  
\label{cb-new2}
\end{figure}
\vspace{-2mm}

\begin{figure}[H]
\centering
    \subfigure[\textit{Accuracy} comparison]{
    \includegraphics[width=0.25\textwidth]{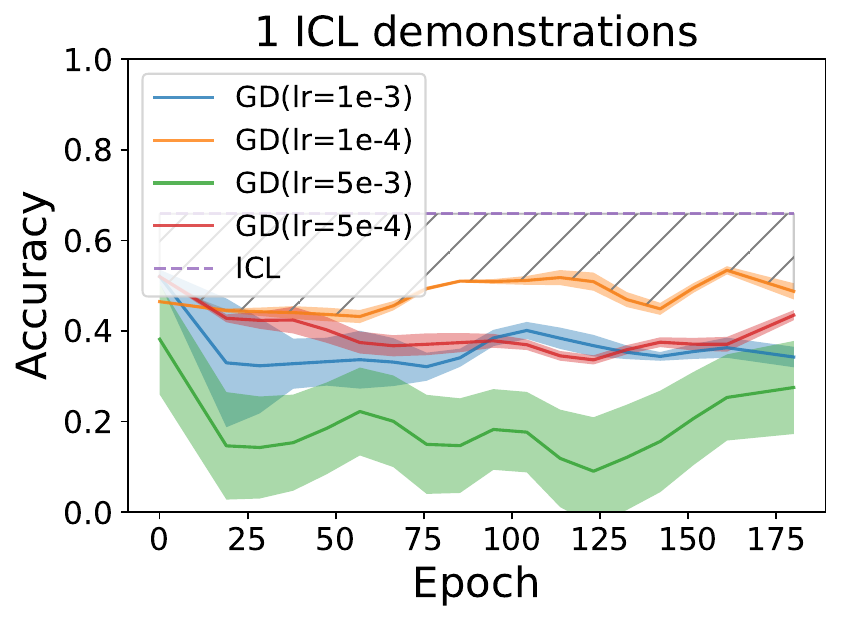}
    \includegraphics[width=0.24\textwidth]{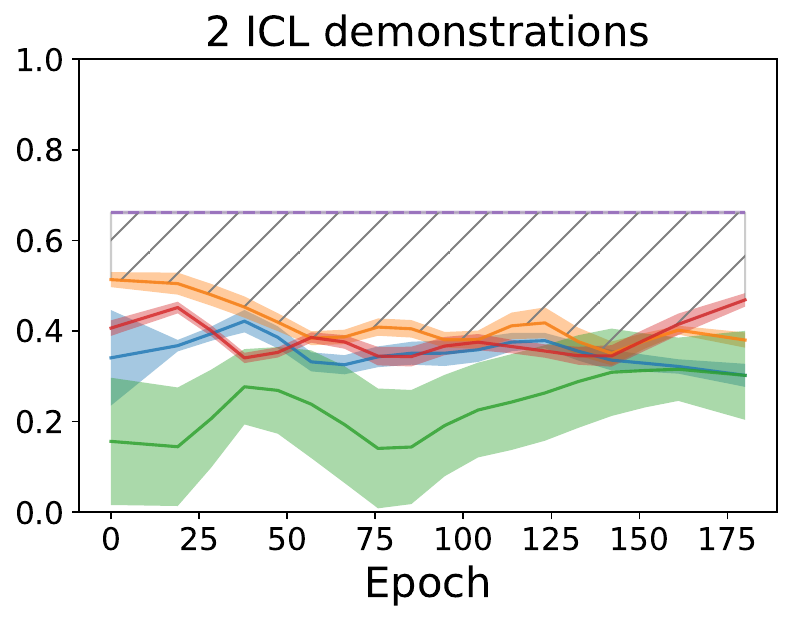}
    \includegraphics[width=0.24\textwidth]{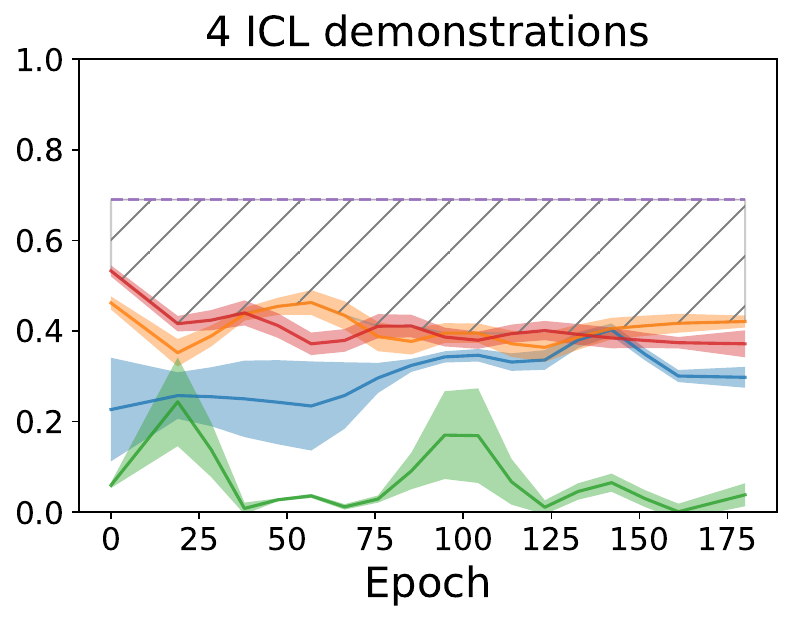}
    \includegraphics[width=0.24\textwidth]{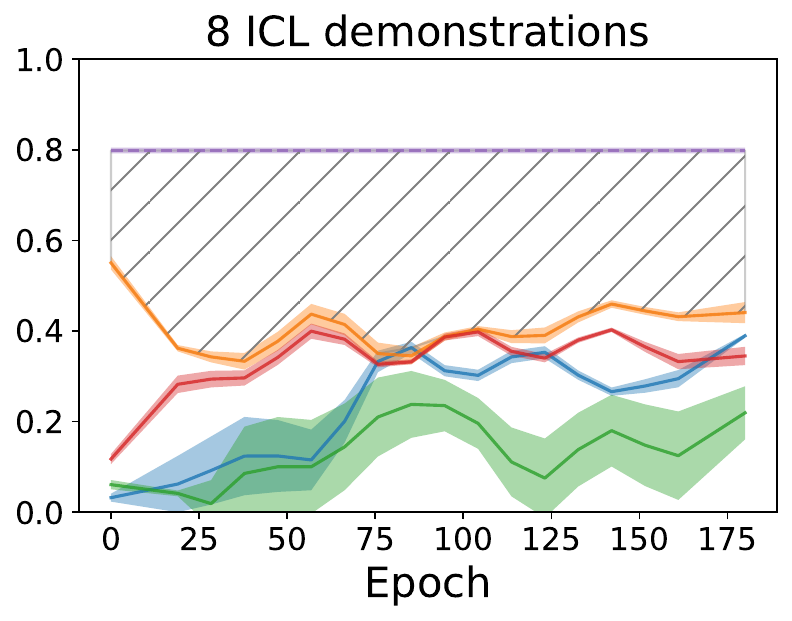}
    }
    \subfigure[\textit{Token overlap} comparison]{
    \includegraphics[width=0.25\textwidth]{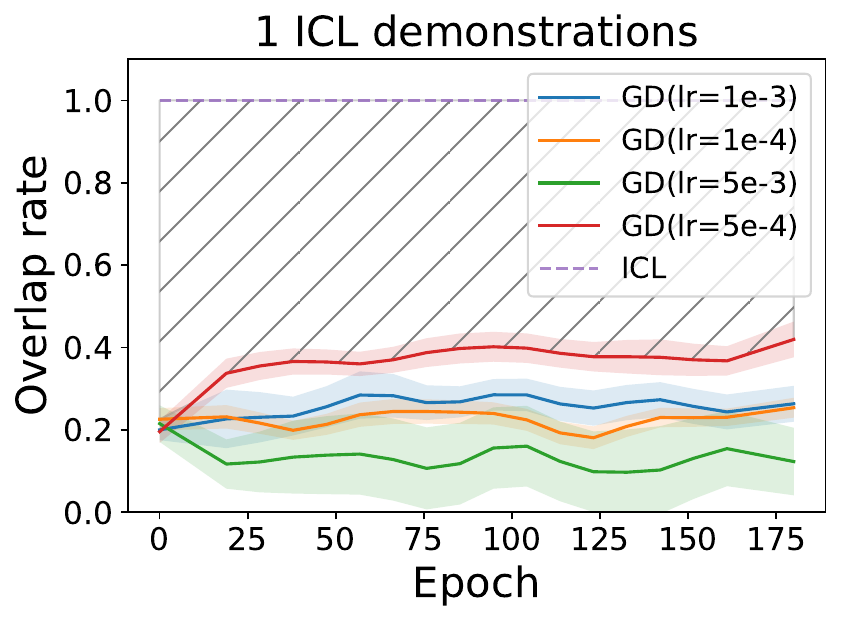}
    \includegraphics[width=0.24\textwidth]{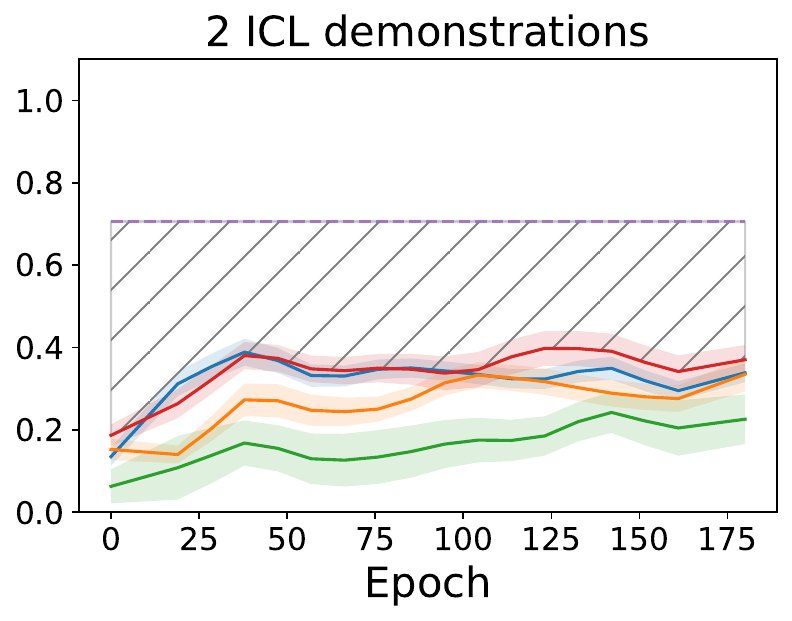}
    \includegraphics[width=0.24\textwidth]{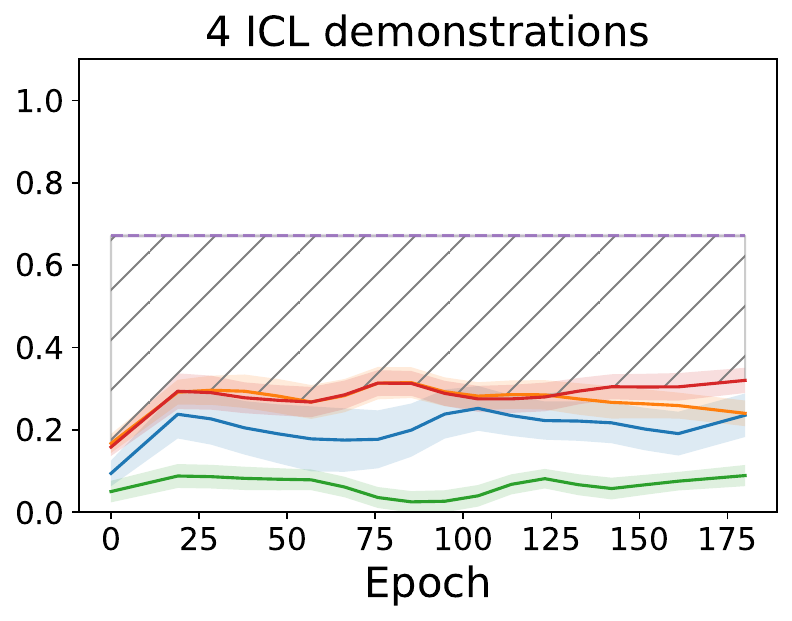}
    \includegraphics[width=0.24\textwidth]{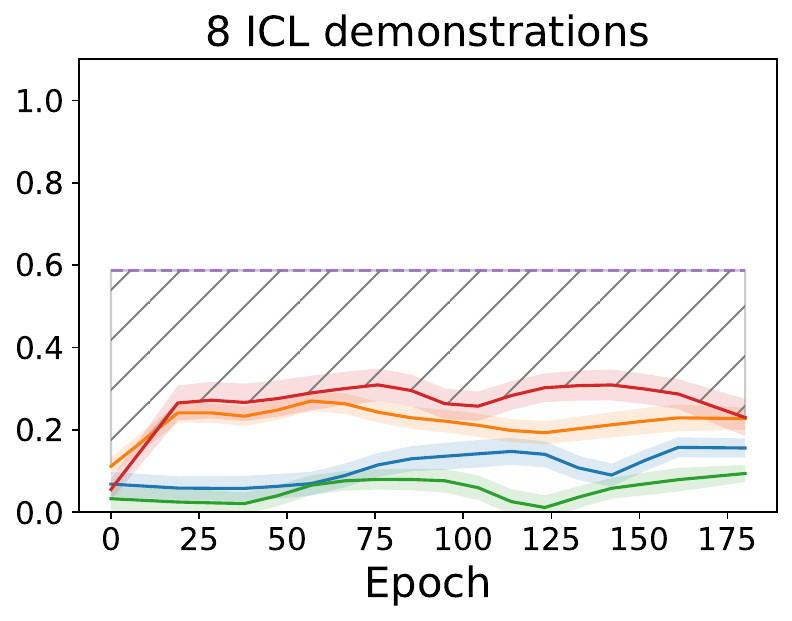}
    } 
    \subfigure[\textit{Overlap Cosine Similarity} comparison]{
    \includegraphics[width=0.25\textwidth]{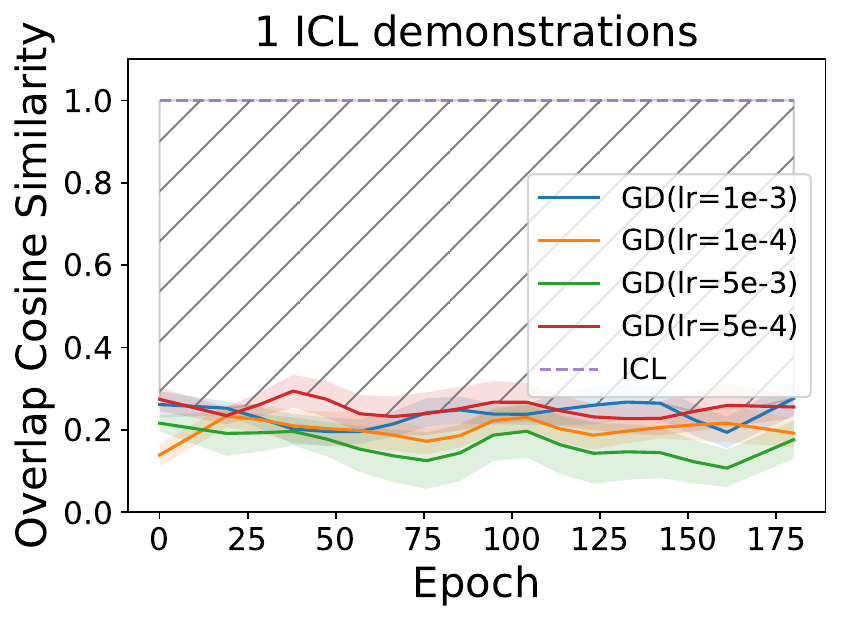}
    \includegraphics[width=0.24\textwidth]{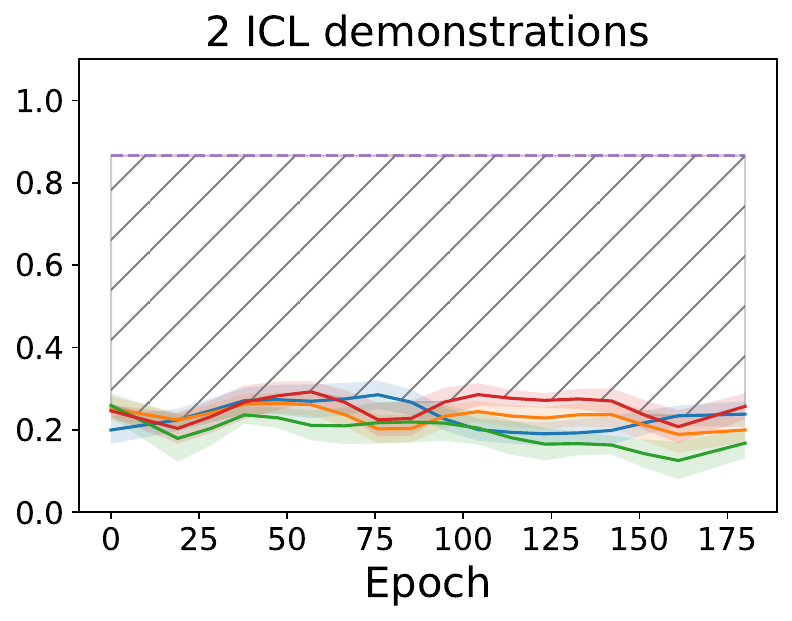}
    \includegraphics[width=0.24\textwidth]{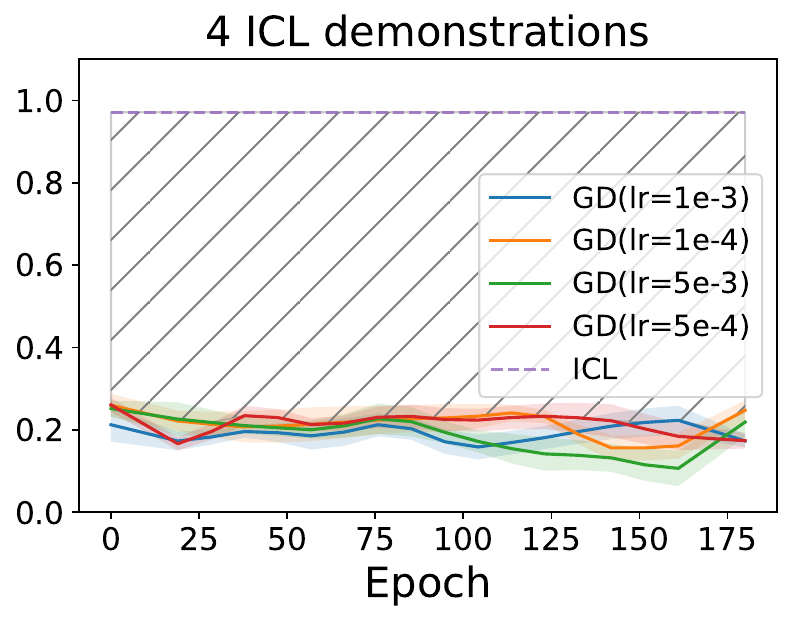}
    \includegraphics[width=0.24\textwidth]{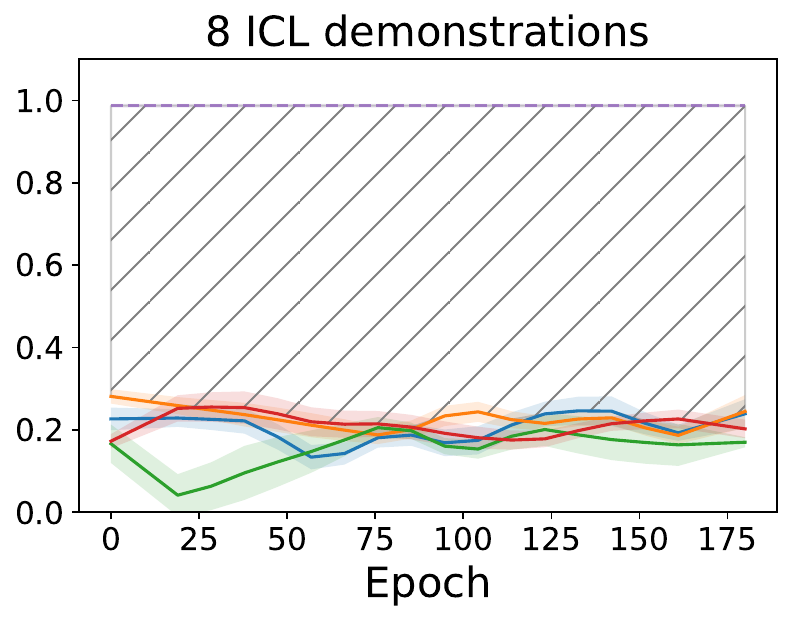}
    }\vspace{-3.5mm}
\caption{Comparison of ICL and $\widehat{\text{GD}}$ for the RTE dataset, with increasing number of demonstrations. $\widehat{\text{GD}}$ is simulated by optimizing on one random middle layer of LLaMa.}  
\label{rte-new2}
\end{figure}

\section{Impact of model capacity on the ICL vs GD.}\label{app:model_size}
We also investigated the influence of model size on the gap between ICL and GD. Specifically, we fix the dataset to AGNews, $N=8$, and select \textsc{GPT2-XL} \citep{radford2019language}, \textsc{GPT-Neo} \citep{gpt-neo}, \textsc{GPT-J} \citep{gpt-j} as models of choice to conduct ICL vs GD experiments. Note that the model capacity is ranked as follows: \textsc{LLaMa (7B)} \textgreater  \textsc{GPT-J (6B)}\textgreater \textsc{GPT-Neo (2.7B)}\textgreater \textsc{GPT2-XL (1.5B)}. The results are shown in \autoref{model_size}, from where we can see that the gap does not change significantly as the model size increases from \textsc{GPT2-XL} to \textsc{LLaMa}.

\begin{figure}[!ht]
\centering
\vspace{-2.0mm}
    \subfigure[\textit{Accuracy} comparison]{
    \includegraphics[width=0.27\textwidth]{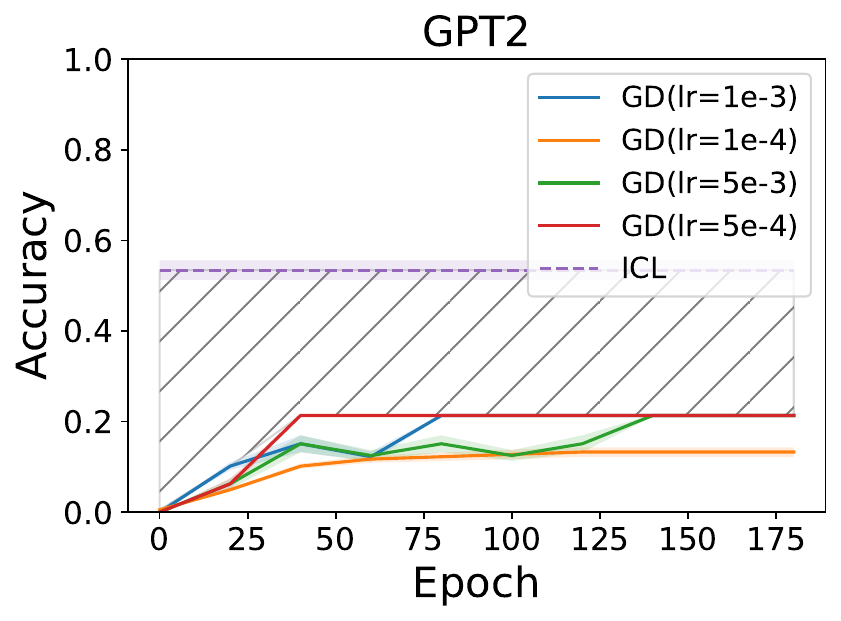}
    \includegraphics[width=0.233\textwidth,trim=1.1cm 0cm 0cm 0cm,clip=true]{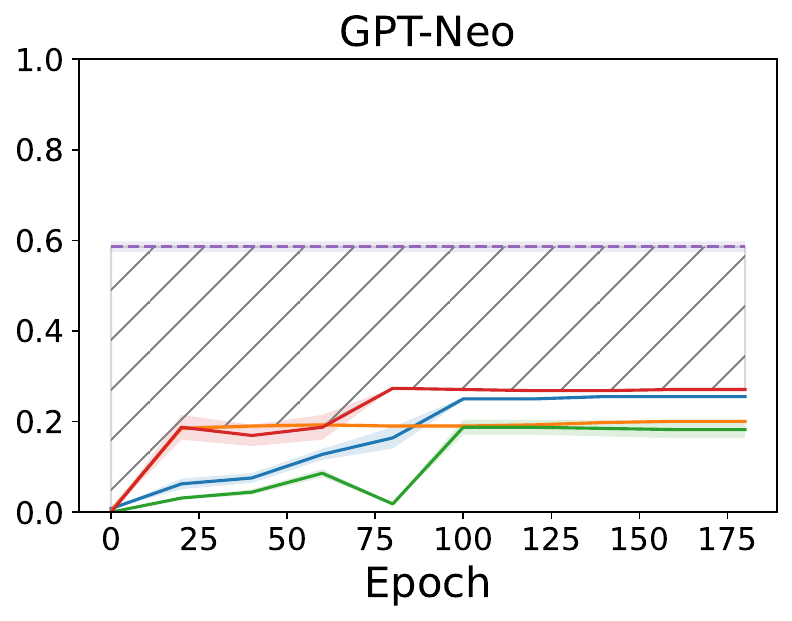}
    \includegraphics[width=0.233\textwidth,trim=1.1cm 0cm 0cm 0cm,clip=true]{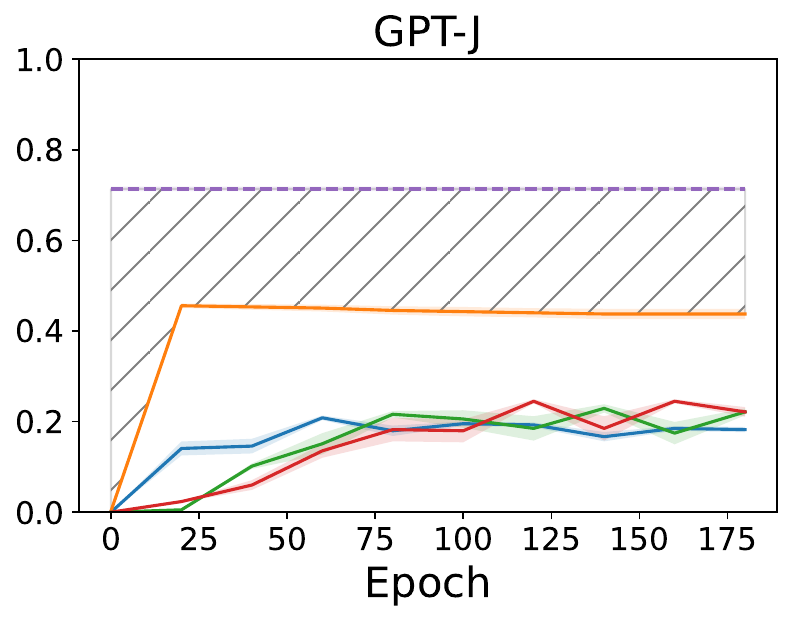}
    \includegraphics[width=0.233\textwidth,trim=1.1cm 0cm 0cm 0cm,clip=true]{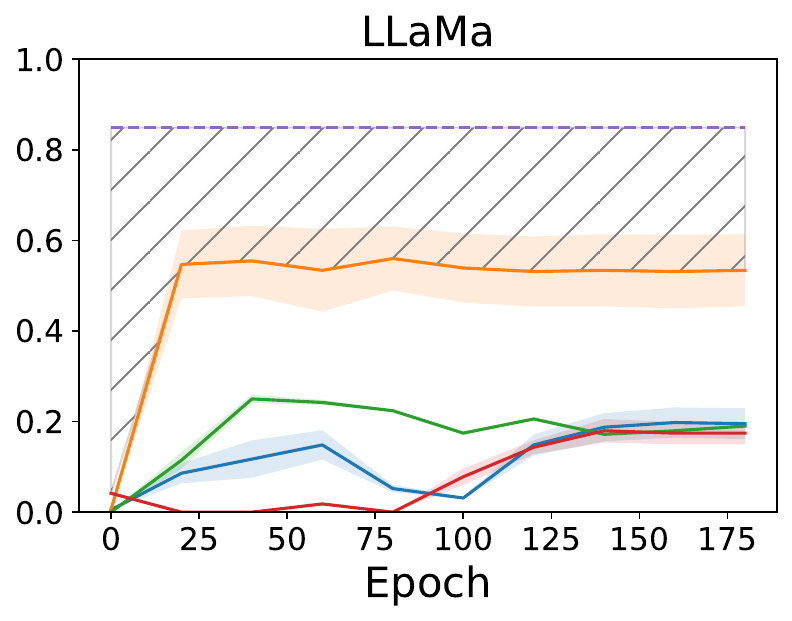}

    }
\vspace{-2.0mm}
    \subfigure[\textit{Token overlap} comparison]
    {
    \includegraphics[width=0.27\textwidth]{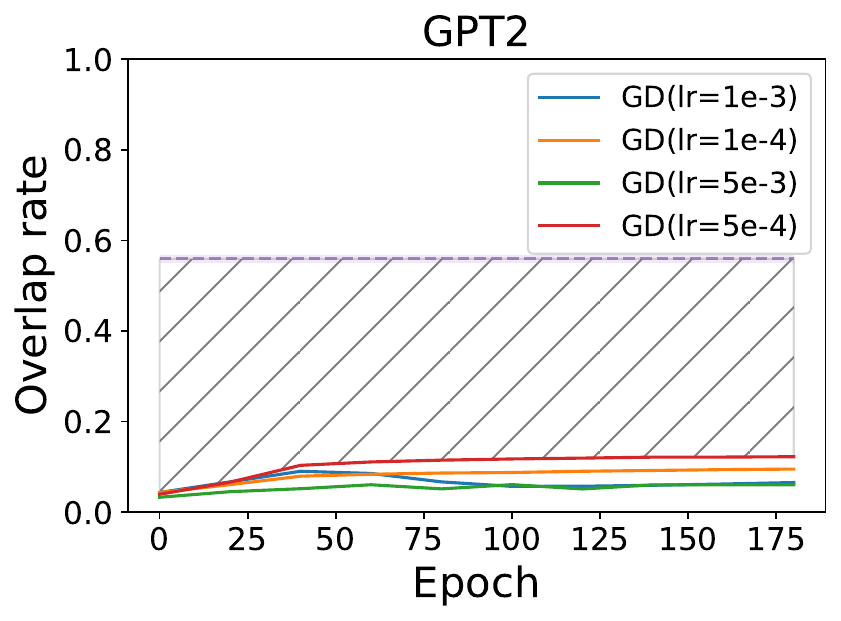}
    \includegraphics[width=0.233\textwidth,trim=1.1cm 0cm 0cm 0cm,clip=true]{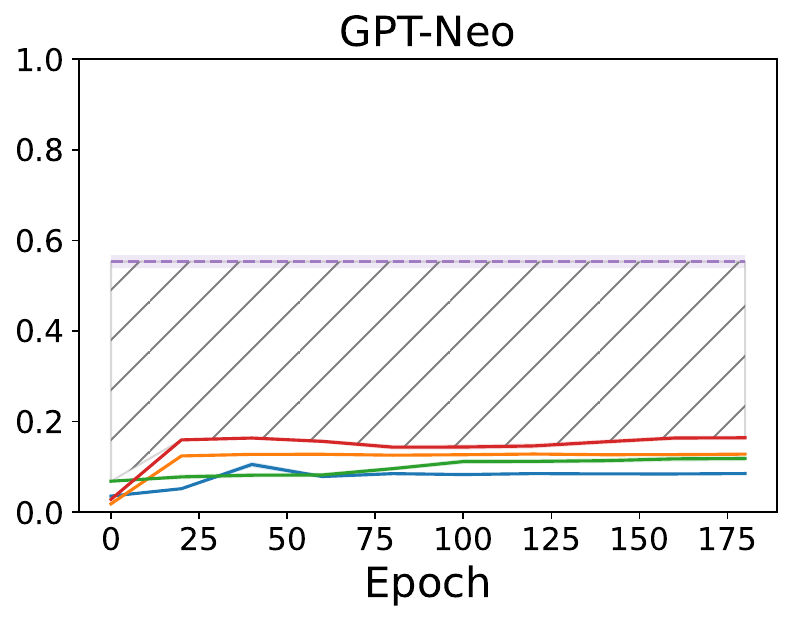}
    \includegraphics[width=0.233\textwidth,trim=1.1cm 0cm 0cm 0cm,clip=true]{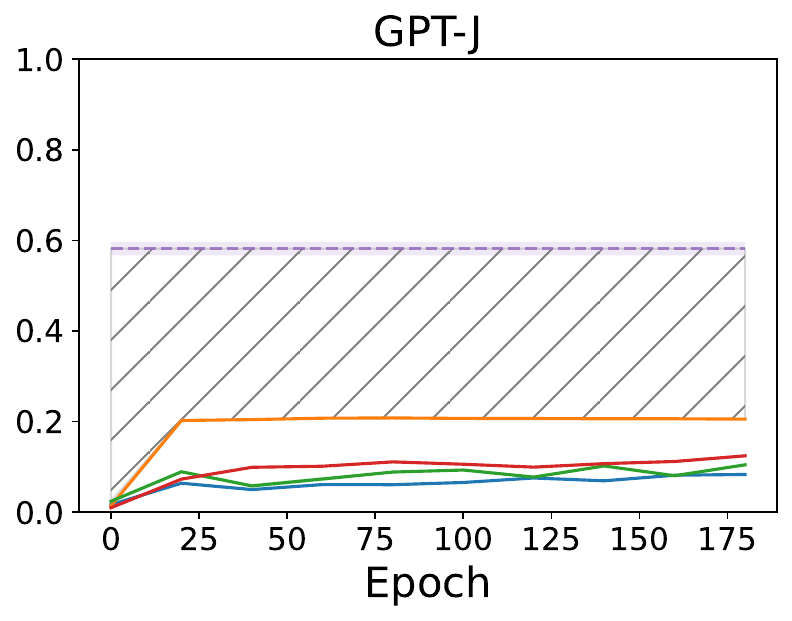}
    \includegraphics[width=0.233\textwidth,trim=1.1cm 0cm 0cm 0cm,clip=true]{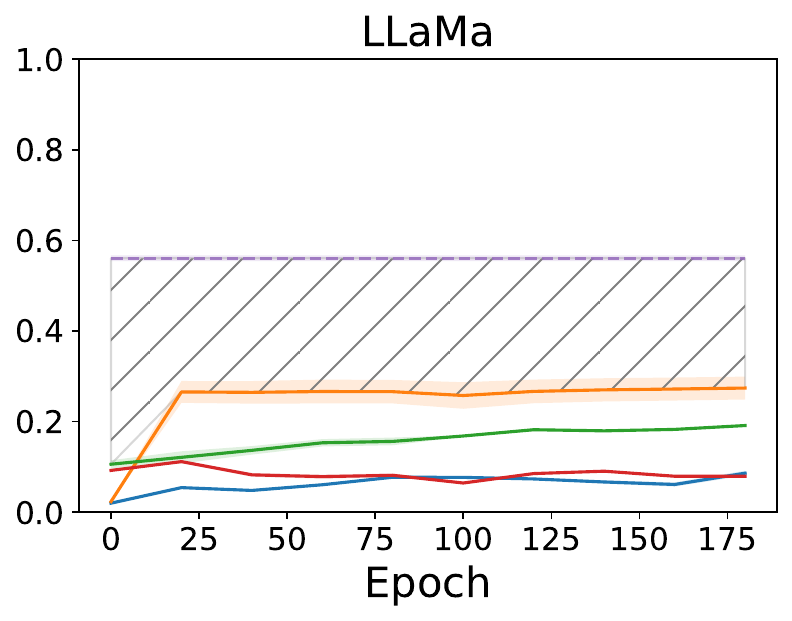}
    
    }

    \subfigure[\textit{Overlap Cosine Similarity} comparison]{
    \includegraphics[width=0.27\textwidth]{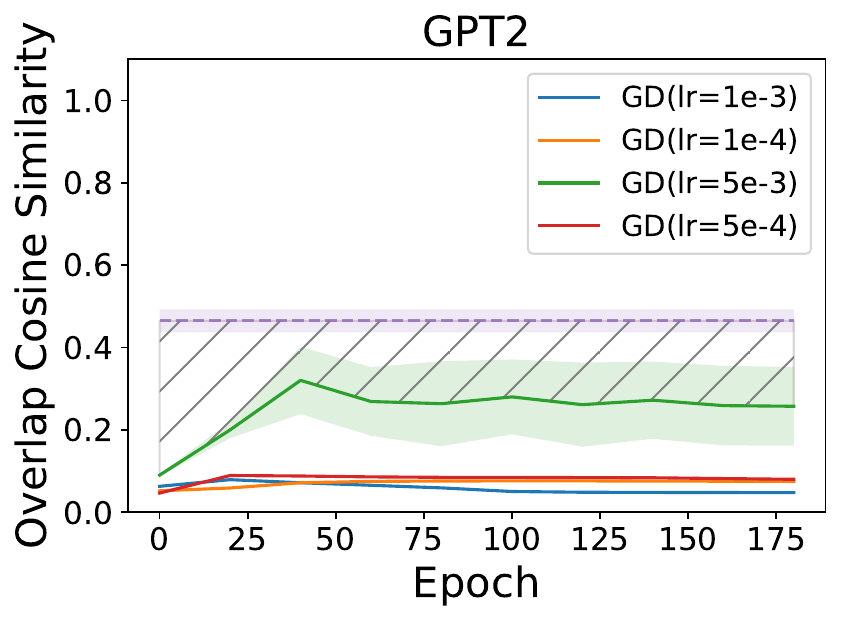}
    \includegraphics[width=0.233\textwidth,trim=1.1cm 0cm 0cm 0cm,clip=true]{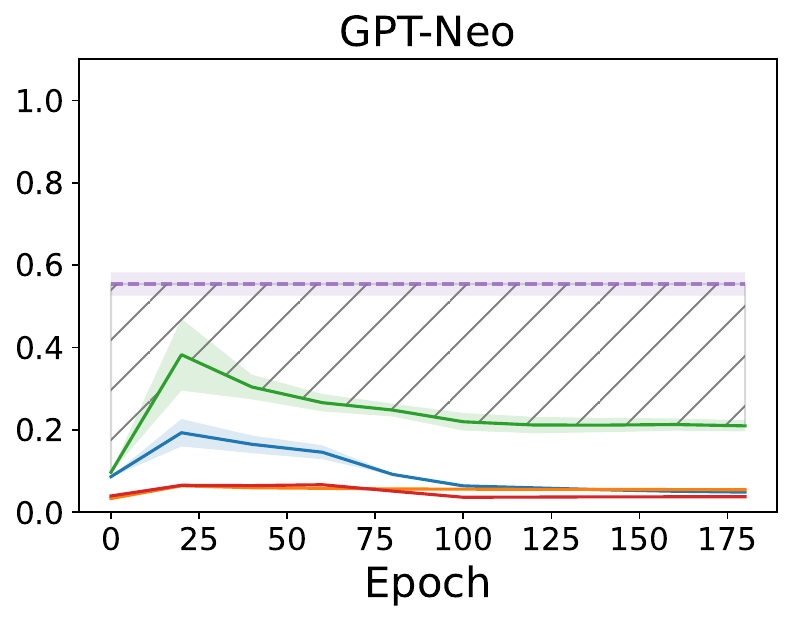}
    \includegraphics[width=0.233\textwidth,trim=1.1cm 0cm 0cm 0cm,clip=true]{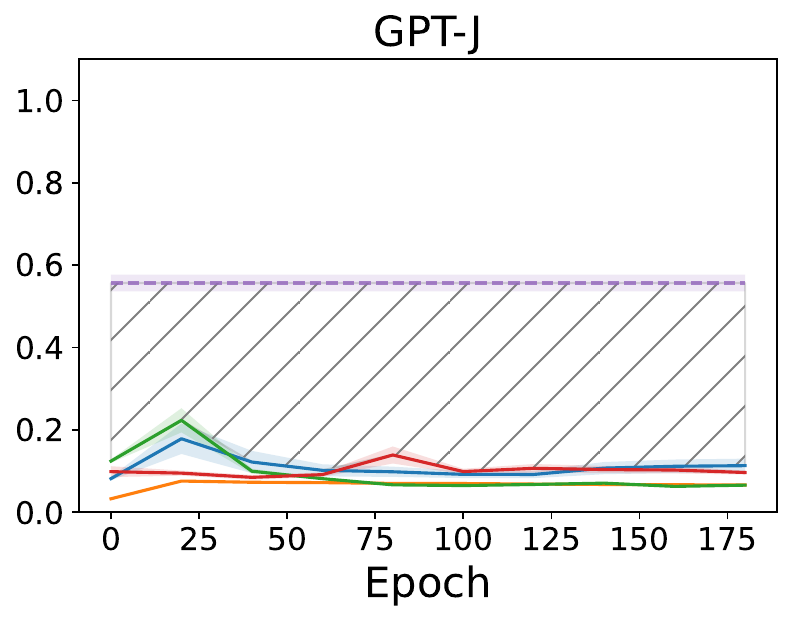}
    \includegraphics[width=0.233\textwidth,trim=1.1cm 0cm 0cm 0cm,clip=true]{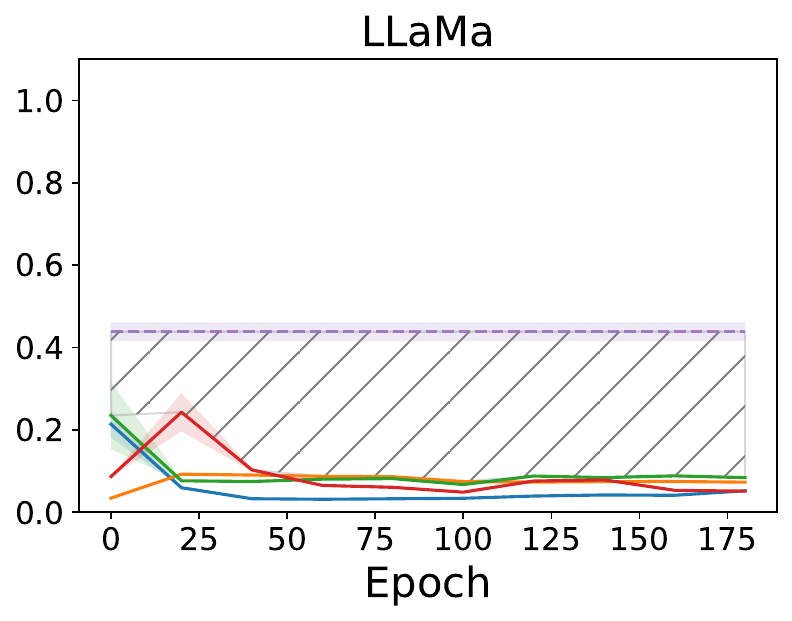}
    } 
\vspace{-4.0mm}
\caption{Comparison of ICL and GD for the AGNews dataset as model size varies.}  
\label{model_size}
\end{figure}
\vspace{-2mm}

% \section{Related Work: Distributional and Empirical Explanations of ICL}
% \label{appendix:related:work}

\end{document}

%% file: math_commands.tex
\usepackage{amsmath,amsfonts,bm}

\def\eqref#1{equation~\ref{#1}}
\def\1{\bm{1}}

\DeclareMathAlphabet{\mathsfit}{\encodingdefault}{\sfdefault}{m}{sl}
\SetMathAlphabet{\mathsfit}{bold}{\encodingdefault}{\sfdefault}{bx}{n}

\DeclareMathOperator*{\argmax}{arg\,max}
\DeclareMathOperator*{\argmin}{arg\,min}